\documentclass[11pt]{article}

\usepackage[margin=1in]{geometry}
\usepackage{times}
\usepackage{amsmath,amssymb,amsthm}
\usepackage{array}
\usepackage{booktabs}
\usepackage{longtable}
\usepackage{enumitem}
\usepackage{float}
\usepackage{placeins}
\usepackage{graphicx}
\usepackage{microtype}
\usepackage[numbers,sort&compress]{natbib}
\usepackage[hidelinks]{hyperref}
\usepackage{url}
\hypersetup{
  pdftitle={Contact Geometry and Covariance Deficits in Volume-Sampled Least Squares},
  pdfauthor={Kihun Rhee, Hanjoon Byun, Junpyo Seo},
  pdfsubject={Fixed-design attainability for volume-sampled least squares},
  pdfkeywords={volume sampling, least squares, randomized numerical linear algebra, covariance}
}

\newtheorem{theorem}{Theorem}[section]
\newtheorem{proposition}[theorem]{Proposition}

\newtheorem{corollary}[theorem]{Corollary}
\newtheorem{lemma}[theorem]{Lemma}
\theoremstyle{remark}

\DeclareMathOperator{\tr}{tr}
\newcommand{\R}{\mathbb{R}}
\newcommand{\E}{\mathbb{E}}
\newcommand{\PP}{\mathbb{P}}
\newcommand{\1}{\mathbf{1}}

\newcommand{\lossstar}{L^*}

\newcommand{\Buniv}{B_{\mathrm{univ}}}
\newcommand{\Blin}{B_{\mathrm{lin}}}
\newcommand{\Bres}{B_{\mathrm{res}}}

\newcolumntype{L}[1]{>{\raggedright\arraybackslash}p{#1}}

\title{Contact Geometry and Covariance Deficits\\
in Volume-Sampled Least Squares}
\author{Kihun Rhee\\
\texttt{rheekh00@snu.ac.kr}
\and Hanjoon Byun\\
\texttt{byunhanjoon@snu.ac.kr}
\and Junpyo Seo\\
\texttt{jpseo99@snu.ac.kr}}
\date{}

\begin{document}
\raggedbottom
\maketitle

\begin{abstract}

We classify when ordinary fixed-size volume sampling followed by unweighted
least squares attains its sharp coefficient-covariance ceiling on a fixed
design. For a real whitened design without coloops and a fixed positive-loss
residual, the contact space is unchanged at every strict-interior sample
size. Its possible nonzero values form a finite orthogonal family: each
maximal parallel class of normalized Naimark-complement rows determines a
deletion nullspace of dimension one less than the class size. A single
residual attains an entire query precisely when the query range lies in one
class space. The proof starts from two-sided Loewner comparison of every
normalized covariance deficit with an explicit leave-one-out operator,
using supported omission moments and reverse deletion. Residual
augmentation provides resolvent and second-moment upper bounds, while
complement geometry yields query-specific margins, angular concentration,
local alignment, and a multi-output energy obstruction. Exact families give
closed-form margins and covariances, exhibit support-boundary jumps, and
approach the ceiling despite a uniformly positive geometric margin.
Finally, the same moment identities give upper and lower bounds on expected
fixed-query squared-loss excess. The subset draw is the only randomness;
all support and endpoint restrictions are explicit.
\end{abstract}

\section{Introduction}
\label{sec:introduction}

A sharp covariance bound describes the worst design and response in a class.
It does not identify the directions that attain the bound on one fixed
design. For ordinary volume sampling followed by unweighted least squares,
we ask a more precise question: which residuals and coefficient subspaces
contact the covariance ceiling, and how does their geometry control the
deficit below that ceiling?

The answer is a finite geometric classification. After whitening a
full-column-rank design to $A^\top A=I_d$, normalize the rows of an
orthonormal complement $N$. Partition these rows into maximal parallel
classes, allowing opposite signs. A class $C$ determines a coefficient
space $V_C=\ker A_{-C}$ of dimension $|C|-1$. The nonzero spaces $V_C$
are mutually orthogonal. On the positive-loss, no-coloop domain, they are
exactly the possible nonzero contact spaces at every strict-interior
sample size. A query contacts its entire trace ceiling precisely when its
range lies in one such space; one residual then works for every size.

\paragraph{From sampling to geometry.}
The bridge is an explicit leave-one-out (LOO) deficit. Write
$c_i=1-\|a_i\|^2>0$ and let $z$ be a compatible unit residual. For
$d<s<m$ we prove
\[
 b_sD_{\rm LOO}(z)\preceq
 I_d-\frac{\overline M_s^{A,z}}{\alpha_sL^*}
 \preceq D_{\rm LOO}(z),\qquad
 D_{\rm LOO}(z)=\sum_i\left(1-\frac{z_i^2}{c_i}\right)a_i a_i^\top,
\]
where $\alpha_s=(m-s)/(m-d)$ and
$b_s=(s-d)/[(m-s)(m-d-1)]$. Both factors are positive, so every interior
covariance deficit has the same kernel. Complement geometry identifies
that kernel and classifies it across residuals. The two inequalities have
different proofs: an ordinary-law omission-moment remainder, and a
reverse-deletion coupling using the universal covariance envelope. Neither
requires a contact class or row general position.

\paragraph{Main results.}
The paper develops three connected levels of this question.
First, the contact-space classification resolves both directional and
whole-query attainment. Its feature-only phase test is the vanishing of a
two-row complement minor, equivalently rank loss after deleting two design
rows. Second, the LOO comparison, query-specific residual margins, and a
residual-augmented resolvent quantify departure from contact. They yield
global angular-cap bounds, local joint residual--query alignment, and a
multi-output residual-energy obstruction. Third, exact families separate
the quantities that this geometry controls: a matched rotation family has
closed-form margins and covariances, whereas a high-dimensional family
approaches the ceiling despite a uniformly positive geometric margin.
The latter shows why a positive phase cannot imply a gap independent of
dimension and sampling budget.

The covariance statements also give query-risk bounds. Unbiasedness turns
coefficient covariance into expected squared-loss excess over the full-pool
fit. Residual-augmented moments give upper bounds; optimizing the omission
predictor gives lower bounds. These are consequences of the covariance
theory and require no distributional model for the responses.

\paragraph{Scope and organization.}
All designs, responses, and query features are fixed before sampling; subset
selection is the only randomness. We retain ordinary indexed volume
sampling and selected unweighted OLS throughout. Section~\ref{sec:setup}
proves the sharp envelope on all full-rank designs.
Section~\ref{sec:loo-comparison} establishes the comparison across sizes,
and Section~\ref{sec:query-geometry} gives the contact classification and
fixed-design phase. Section~\ref{sec:mechanism} develops augmented moment
bounds; Section~\ref{sec:quantitative-geometry} quantifies the geometry.
Exact families and risk consequences follow in
Sections~\ref{sec:exact-families} and~\ref{sec:risk}.
The appendices contain determinant details, moment computation, and
numerical certification.

\subsection{Relation to prior work}
\label{sec:related}
Ordinary-volume normalizers, selected-OLS unbiasedness, basis padding, and
rank-size regression moments are established ingredients
\citep{derezinski2017unbiased,derezinski2018reverse,derezinski2019bias}.
The all-size sampling identities and rank-size regression results in
\citet[Theorems~5--8]{derezinski2018reverse} concern the same sampling
primitive. The ordinary-law lower construction in
\citet[Theorem~1, Eq.~(2)]{derezinski2018leveraged} has related budget
dependence for scalar loss. The rank-size active-regression statement in
\citet[Theorem~3.2, Eq.~(3.4)]{epperly2026adaptive} also concerns the
rank-size endpoint. Our target is the all-size centered coefficient
covariance, its fixed-design equality spaces, and quantitative deficits.

Leveraged and regularized volume sampling alter the law or estimator
\citep{derezinski2018leveraged,derezinski2018regularized}; volume-rescaled
random-design regression and minimax experimental design use different
design models or variance criteria
\citep{derezinski2019bias,derezinski2022random,derezinski2019minimax}.
The augmented law used here is an analysis device, while the actual sampler
remains feature-only. Inverse-map Jensen is classical
\citep{hansen2003jensen}; the additional inputs are the residual change of
measure, its moments, and an ordered inverse remainder.

Naimark-complement identities and lower-frame bounds under erasures supply
the geometric background \citep{casazza2013naimark,fickus2012nerf}.
At critical equal leverage, the repeated-pair zero locus has the same
binary-erasure boundary as uniform Parseval frames
\citep{bodmann2005frames}. Our operator instead uses the coupled fractional
weights $1-\ell_i-z_i^2$ generated by one residual in a fixed complement.
Contact-space classification and comparison across sizes connect this
geometry to ordinary-volume covariance.

\section{Supported volume sampling and the sharp covariance envelope}
\label{sec:setup}

Let $X\in\R^{m\times d}$ have full column rank and let $y\in\R^m$ be a
deterministic response.  The full-pool OLS fit, residual, and loss are
\begin{equation}
  G=X^\top X,\qquad w^*=G^{-1}X^\top y,\qquad
  e=y-Xw^*,\qquad \lossstar=\lVert e\rVert_2^2.
  \label{eq:full-fit}
\end{equation}
For an indexed subset $S\subseteq[m]$ of size $s$, write
\begin{equation}
  G_S=X_S^\top X_S,\qquad D_S=\det(G_S).
  \label{eq:subset-gram}
\end{equation}
For $m>d$ and $d\le s\le m$, ordinary unrescaled fixed-size volume sampling
and its standard Cauchy--Binet normalizer are
\begin{equation}
  \PP_X(S)=\frac{D_S}{Z_{X,s}},\qquad
  Z_{X,s}=\sum_{|U|=s}D_U
  =\binom{m-d}{s-d}\det(G).
  \label{eq:law}
\end{equation}
Only positive-volume subsets have positive probability.  On such a subset we
use selected unweighted OLS,
\begin{equation}
  w_S=G_S^{-1}X_S^\top y_S,\qquad
  L_S=\lVert X_Sw_S-y_S\rVert_2^2.
  \label{eq:selected-fit}
\end{equation}

The subset draw is the only randomness.  Our matrix target and its
coordinate-invariant whitening are
\begin{equation}
  M_s=\E_X[(w_S-w^*)(w_S-w^*)^\top],\qquad
  \overline M_s=G^{1/2}M_sG^{1/2}.
  \label{eq:covariance-metric}
\end{equation}
For fixed query features $Z\in\R^{q\times d}$ with $q\ge1$, define
the whitened query metric
$\Xi=G^{-1/2}Z^\top ZG^{-1/2}\succeq0$.
The unbiasedness proved below gives, for every fixed query target
$y_Q\in\R^q$,
\begin{equation}
 \E_X\|Zw_S-y_Q\|_2^2-\|Zw^*-y_Q\|_2^2
 =\tr(\Xi\overline M_s).
 \label{eq:query-trace-identity}
\end{equation}
Dividing by $q$ gives expected query-MSE excess over the full-pool fit.
A normalized query trace below requires $\tr\Xi>0$.

For $m>d$, define the universal-ceiling coefficient and its complementary
weight by
\begin{equation}
  \alpha=\frac{m-s}{m-d},\qquad
  \beta=\frac{s-d}{m-d}=1-\alpha.
  \label{eq:alpha-beta}
\end{equation}

\subsection{Endpoint conventions}
\label{sec:endpoints}
Endpoint branches are resolved before a normalized ratio or residual
direction is formed.  If $m=d$ or $s=m$, the selected fit is deterministic;
if $\lossstar=0$, every supported selected fit equals $w^*$; in each case the
relevant covariance is zero.  At $s=d<m$, supported square systems
interpolate, but the rank-$(d+1)$ residual augmentation used below is not
invoked.  We repeatedly use the fixed-pool identity
\begin{equation}
  \lVert Xw-y\rVert_2^2=\lossstar+(w-w^*)^\top G(w-w^*)
  \qquad (w\in\R^d).
  \label{eq:pythagoras-main}
\end{equation}

\subsection{Fixed-cardinality normalizer and theorem statement}
\label{subsec:ue-law}

\begin{proposition}[Fixed-cardinality normalizer]
\label{prop:ue-normalizer}
For $m>d$ and $d\le s\le m$,
\begin{equation}
  Z_{X,s}:=\sum_{|S|=s}D_S
  =\binom{m-d}{s-d}\det(G)>0,
  \qquad
  \PP_X(S)=\frac{D_S}{Z_{X,s}}.
  \label{eq:ue-normalizer}
\end{equation}
\end{proposition}

\begin{proof}
For each indexed $S$, Cauchy--Binet gives
\begin{equation}
  D_S=\sum_{\substack{T\subseteq S\\|T|=d}}\det(X_T)^2.
  \label{eq:ue-inner-cauchy-binet}
\end{equation}
Each indexed $d$-set $T$ occurs in exactly
$\binom{m-d}{s-d}$ indexed size-$s$ supersets.  Summing
\eqref{eq:ue-inner-cauchy-binet} and applying Cauchy--Binet once more yields
\[
  \sum_{|S|=s}D_S
  =\binom{m-d}{s-d}\sum_{|T|=d}\det(X_T)^2
  =\binom{m-d}{s-d}\det(X^\top X).
\]
Full column rank makes the final determinant positive.  The count is over
indexed subsets, so it is unchanged when some rows coincide; terms with zero
determinant simply contribute zero.
\end{proof}

\begin{theorem}[Universal covariance envelope for fixed-pool volume-sampled least squares]
\label{thm:ue-universal-envelope}\label{thm:main-envelope}
For every full-column-rank $X$, every fixed $y$, $m>d$, and
$d\le s\le m$,
\begin{equation}
  \E_X w_S=w^*.
  \label{eq:ue-unbiasedness}
\end{equation}
Consequently the centered moment $M_s$ is the covariance of $w_S$, and
\begin{equation}
  M_s\preceq \alpha\lossstar G^{-1},
  \qquad
  \overline M_s:=G^{1/2}M_sG^{1/2}
  \preceq \alpha\lossstar I_d.
  \label{eq:ue-envelope}
\end{equation}
Moreover,
\begin{equation}
  \E_X\lVert Xw_S-y\rVert_2^2
  \le (1+d\alpha)\lossstar
  =\left(1+\frac{d(m-s)}{m-d}\right)\lossstar.
  \label{eq:ue-loss-corollary}
\end{equation}
The coefficient in~\eqref{eq:ue-envelope} is attained by the family in
Proposition~\ref{prop:ue-attainment}.  For positive-loss row-general-position
designs with $d\ge2$ it is strictly unattained when $d<s<m$, but it remains
their operator-norm supremum as described in
Proposition~\ref{prop:ue-interior-supremum}.
\end{theorem}

The directional ratio associated with~\eqref{eq:ue-envelope} is
$\lambda_{\max}(\overline M_s)/(\alpha\lossstar)$ only when
$\alpha\lossstar>0$.  It is not the corresponding raw Euclidean eigenvalue
of $M_s$, and no normalized ratio is formed at $s=m$ or $\lossstar=0$.

\subsection{Rank-size basis moments and uniform padding}
\label{subsec:ue-basis-coupling}

The following rank-size calculation includes the off-diagonal second
moments.  Rank-size unbiasedness and volume-sampling moments are known
ingredients~\citep{derezinski2017unbiased,derezinski2018reverse}; the direct
derivation fixes the signs and support needed below.

\begin{lemma}[Rank-size basis moments]
\label{lem:ue-basis-moments}
Let $A\in\R^{q\times d}$ have full column rank, suppose every indexed
$d$-row submatrix is nonsingular, and fix $b\in\R^q$.  Define
\[
  K=A^\top A,
  \qquad v=K^{-1}A^\top b,
  \qquad r=b-Av.
\]
Draw a $d$-set $T$ with probability $\det(A_T)^2/\det(K)$ and put
$v_T=A_T^{-1}b_T$.  Then
\begin{equation}
  \E v_T=v,
  \qquad
  \E[(v_T-v)(v_T-v)^\top]
  =\lVert r\rVert_2^2K^{-1}.
  \label{eq:ue-basis-moments}
\end{equation}
\end{lemma}

An entrywise Cauchy--Binet proof is given in Section~\ref{app:basis-details}.

Say that $X$ is in \emph{row general position} when every indexed set of
$d$ rows is nonsingular.  Assume this condition temporarily.  First sample
$S$ from~\eqref{eq:ue-normalizer}; conditional on $S$, sample a $d$-set
$T\subseteq S$ by the rank-size volume law within $X_S$.  The joint law is
\begin{equation}
  \PP_X(S,T)
  =\frac{D_S}{Z_{X,s}}\frac{\det(X_T)^2}{D_S}
  =\frac{\det(X_T)^2}
  {\binom{m-d}{s-d}\det(G)}.
  \label{eq:ue-joint-coupling}
\end{equation}
Consequently
\begin{equation}
  \PP_X(T)=\frac{\det(X_T)^2}{\det(G)},
  \qquad
  \PP_X(S\mid T)=\binom{m-d}{s-d}^{-1}
  \quad(T\subseteq S).
  \label{eq:ue-coupling-marginals}
\end{equation}
Thus $S\mid T$ is uniform over the size-$s$ supersets of the sampled basis.
This basis-plus-uniform-padding representation is also used in the
fixed-size volume-sampling literature~\citep{derezinski2019bias}.

\begin{proposition}[Exact row-general-position covariance decomposition]
\label{prop:ue-rgp-covariance}
If $X$ is in row general position, then
\begin{equation}
  \E_Xw_S=w^*,
  \qquad
  M_s=\lossstar G^{-1}-\E_X[L_SG_S^{-1}].
  \label{eq:ue-rgp-covariance}
\end{equation}
\end{proposition}

\begin{proof}
Apply Lemma~\ref{lem:ue-basis-moments} conditionally with
$(A,b)=(X_S,y_S)$ and then unconditionally with $(A,b)=(X,y)$.  The coupling
gives
\begin{equation}
  \E[w_T\mid S]=w_S,
  \qquad
  \operatorname{Cov}(w_T\mid S)=L_SG_S^{-1},
  \label{eq:ue-conditional-basis}
\end{equation}
and
\begin{equation}
  \E w_T=w^*,
  \qquad
  \operatorname{Cov}(w_T)=\lossstar G^{-1}.
  \label{eq:ue-full-basis}
\end{equation}
Iterated expectation proves $\E w_S=w^*$.  The matrix law of total
covariance applied to~\eqref{eq:ue-conditional-basis} and
\eqref{eq:ue-full-basis} gives
\[
  \lossstar G^{-1}
  =\E_X[L_SG_S^{-1}]+\operatorname{Cov}(w_S),
\]
which is~\eqref{eq:ue-rgp-covariance}.  No trace has been taken.
\end{proof}

\subsection{The selected residual first moment}
\label{subsec:ue-selected-residual}

\begin{lemma}[Selected residual mean]
\label{lem:ue-selected-residual}
For every full-column-rank $X$, $m>d$, and $d\le s\le m$,
\begin{equation}
  \E_XL_S=\beta\lossstar.
  \label{eq:ue-selected-residual}
\end{equation}
This identity does not require $\lossstar>0$.
\end{lemma}

\begin{proof}
Let $H=[X\ \ y]\in\R^{m\times(d+1)}$.  On a supported set, the Schur
complement gives
\begin{equation}
  \det(H_S^\top H_S)=D_SL_S,
  \qquad
  \det(H^\top H)=\det(G)\lossstar.
  \label{eq:ue-schur-residual}
\end{equation}
If $D_S=0$, then $\operatorname{rank}(X_S)<d$, so
$\operatorname{rank}(H_S)\le d$ and the determinant on the left is also
zero.  Thus, with unsupported terms understood as zero determinant weights,
the first equality may be summed over all $S$.

For $s\ge d+1$, fixed-cardinality Cauchy--Binet applied to the $d+1$
columns of $H$ yields
\[
  \sum_{|S|=s}\det(H_S^\top H_S)
  =\binom{m-d-1}{s-d-1}\det(H^\top H).
\]
If $\operatorname{rank}(H)\le d$, both sides are zero, so this step does not
divide by $\lossstar$.  Dividing only by the positive ordinary normalizer
from Proposition~\ref{prop:ue-normalizer} gives
\[
  \E_XL_S
  =\frac{\binom{m-d-1}{s-d-1}}
  {\binom{m-d}{s-d}}\lossstar
  =\frac{s-d}{m-d}\lossstar.
\]
When $s=d$, every supported square system interpolates and both sides of
\eqref{eq:ue-selected-residual} are zero.
\end{proof}

\subsection{The Loewner envelope in row general position}
\label{subsec:ue-rgp-envelope}

For row-general-position $X$, combine
Proposition~\ref{prop:ue-rgp-covariance} and
Lemma~\ref{lem:ue-selected-residual}.  Since
$G_S\preceq G$, inverse order gives $G_S^{-1}\succeq G^{-1}$, and hence
\begin{equation}
\begin{aligned}
  \alpha\lossstar G^{-1}-M_s
  &=\E_X[L_SG_S^{-1}]-\beta\lossstar G^{-1}\\
  &=\E_X\!\left[L_S(G_S^{-1}-G^{-1})\right]\succeq0.
\end{aligned}
\label{eq:ue-rgp-slack}
\end{equation}
Congruence by $G^{1/2}$ proves both inequalities in
\eqref{eq:ue-envelope} on this interior class.

The full-pool loss consequence is also a matrix corollary.  The normal
equations $X^\top e=0$ imply, for every $w$,
\begin{equation}
  \lVert Xw-y\rVert_2^2
  =\lossstar+\lVert w-w^*\rVert_G^2.
  \label{eq:ue-pythagoras}
\end{equation}
Therefore
\begin{equation}
\begin{aligned}
  \E_X\lVert Xw_S-y\rVert_2^2
  &=\lossstar+\tr(GM_s)\\
  &=\lossstar+\tr(\overline M_s)
  \le (1+d\alpha)\lossstar,
\end{aligned}
\label{eq:ue-loss-from-envelope}
\end{equation}
which proves~\eqref{eq:ue-loss-corollary} in row general position.

\subsection{Arbitrary full-rank boundary designs}
\label{subsec:ue-boundary}

Unbiasedness and the covariance inequality require different boundary
arguments.  In particular, the exact decomposition
\eqref{eq:ue-rgp-covariance} is not asserted when some selected Gram matrices
are singular.

\begin{proposition}[Determinant-weighted continuation of the mean]
\label{prop:ue-boundary-mean}
Equation~\eqref{eq:ue-unbiasedness} holds for every full-column-rank $X$,
including designs outside row general position.
\end{proposition}

\begin{proof}
For a supported set,
\begin{equation}
  D_Sw_S=\operatorname{adj}(G_S)X_S^\top y_S.
  \label{eq:ue-adjugate-numerator}
\end{equation}
The right side is polynomial in the entries of $X$ and $y$.  It vanishes
when $G_S$ is singular: indeed, with
$A_S=\operatorname{adj}(G_S)X_S^\top$,
\[
  A_SA_S^\top
  =\operatorname{adj}(G_S)G_S\operatorname{adj}(G_S)=0,
\]
so $A_S=0$.  Thus singular subsets contribute zero to the
determinant-weighted numerator without receiving a selected fit.

On the dense set of row-general-position matrices,
Proposition~\ref{prop:ue-rgp-covariance} and
Proposition~\ref{prop:ue-normalizer} imply
\begin{equation}
  \sum_{|S|=s}\operatorname{adj}(G_S)X_S^\top y_S
  =\binom{m-d}{s-d}\operatorname{adj}(G)X^\top y.
  \label{eq:ue-polynomial-mean}
\end{equation}
Both sides are polynomial, so the identity holds for every $X$ by
polynomial continuation.  For full-column-rank $X$, divide
\eqref{eq:ue-polynomial-mean} by
$Z_{X,s}=\binom{m-d}{s-d}\det(G)$ and use the preceding zero-numerator fact.
The result is
\[
  \E_Xw_S=G^{-1}X^\top y=w^*.
\]
\end{proof}

\begin{proposition}[One-sided boundary passage for the covariance]
\label{prop:ue-boundary-envelope}
The Loewner inequalities~\eqref{eq:ue-envelope}, and consequently the loss
bound~\eqref{eq:ue-loss-corollary}, hold for every full-column-rank $X$.
\end{proposition}

\begin{proof}
Approximate $X$ by row-general-position matrices $X_t$ while keeping $y$
fixed. Every positive-volume term at $X$ is continuous along this sequence.
Dropping the additional nonnegative terms supported only at $X_t$ gives,
for each fixed coefficient direction $u$,
\[
 u^\top M_s(X)u\le\liminf_{t\to0}u^\top M_s(X_t)u
 \le\alpha L^*u^\top G^{-1}u.
\]
Only the full Gram inverse, residual loss, and normalizer are continued.
No inverse is assigned to a zero-volume selected set. Section~\ref{app:boundary-details} gives the explicit perturbation and finite-sum
argument. This proves the envelope, and the full-pool Pythagorean identity
gives its loss consequence.
\end{proof}

\subsection{A sharp core-plus-zero construction}
\label{subsec:ue-sharpness}

\begin{proposition}[Core-plus-zero exact matrix attainment]
\label{prop:ue-attainment}
Fix $m>d\ge1$, let $q=m-d-1$, and define
\begin{equation}
  X_0=\begin{bmatrix}I_d\\\1_d^\top\\0_{q\times d}\end{bmatrix},
  \qquad
  y=\begin{bmatrix}\1_d\\-1\\0_q\end{bmatrix}.
  \label{eq:ue-attaining-family}
\end{equation}
For every $d\le s\le m$, this family satisfies
\begin{equation}
  M_s=\alpha\lossstar G^{-1}.
  \label{eq:ue-exact-attainment}
\end{equation}
\end{proposition}

\begin{proof}
The first $d+1$ rows are the core.  Direct calculation gives
\begin{equation}
  G=I_d+\1_d\1_d^\top,
  \qquad
  G^{-1}=I_d-\frac{\1_d\1_d^\top}{d+1},
  \qquad
  w^*=0,
  \qquad
  \lossstar=d+1.
  \label{eq:ue-attaining-reference}
\end{equation}
A supported size-$s$ set contains either exactly $d$ core rows or all
$d+1$.  Every $d$-core minor has squared determinant one, whereas the
all-core Gram determinant is $d+1$.  With the convention that an
out-of-range binomial coefficient is zero, the aggregate determinant weights
are
\begin{equation}
  W_d=(d+1)\binom q{s-d},
  \qquad
  W_{d+1}=(d+1)\binom q{s-d-1}.
  \label{eq:ue-core-event-weights}
\end{equation}
Pascal's identity therefore gives
\begin{equation}
  \PP_X(\text{exactly $d$ core rows})
  =\frac{W_d}{W_d+W_{d+1}}
  =\frac{m-s}{m-d}=\alpha.
  \label{eq:ue-core-event-probability}
\end{equation}

On the all-core event, the normal equations give $w_S=0$.  Conditional on
exactly $d$ core rows, the omitted core row is uniform among the $d+1$
possibilities and the corresponding interpolating fits are
\begin{equation}
  a_0=\1_d,
  \qquad
  a_i=\1_d-(d+1)e_i,
  \quad i=1,\ldots,d,
  \label{eq:ue-core-fits}
\end{equation}
where $e_i$ is the $i$th coordinate vector in $\R^d$.  Their mean is zero,
and their complete matrix second moment is
\begin{equation}
\begin{aligned}
  \frac1{d+1}\left(a_0a_0^\top+\sum_{i=1}^da_ia_i^\top\right)
  &=(d+1)I_d-\1_d\1_d^\top\\
  &=\lossstar G^{-1}.
\end{aligned}
\label{eq:ue-core-conditional-covariance}
\end{equation}
Indeed, expanding the sum gives
$(d+1)^2I_d-(d+1)\1_d\1_d^\top$ before division by $d+1$.
Multiplying~\eqref{eq:ue-core-conditional-covariance} by the event
probability~\eqref{eq:ue-core-event-probability} proves
\eqref{eq:ue-exact-attainment}, including all off-diagonal entries.  The
same formulas cover $s=d$ and $s=m$ through $\alpha=1$ and $\alpha=0$,
respectively.  This proposition supplies an attaining family only.
\end{proof}

\subsection{Whitened residual geometry and the contact domain}
For $L^*>0$, set
\begin{equation}
 A=XG^{-1/2},\quad z=e/\sqrt{L^*},\quad
 \ell_i=\|a_i\|^2,\quad c_i=1-\ell_i,\quad n=m-d.
 \label{eq:augmentation-main}
\end{equation}
Then $A^\top A=I_d$, $A^\top z=0$, and $\|z\|=1$. Define
\begin{equation}
 \mathcal Z_A=\{z\in\ker A^\top:\|z\|=1\},\qquad
 R_A(z)=\sum_i(c_i-z_i^2)a_i a_i^\top\succeq0.
 \label{eq:robust-pre-response-sphere}
\end{equation}
\begingroup\makeatletter\edef\@currentlabel{\theequation}\edef\@currentHref{equation.\theHequation}\label{eq:response-geometry-main}\label{eq:robust-pre-response-psd}\label{eq:app-robust-pre-response-psd}\makeatother\endgroup
The inequality follows from $z_i^2\le c_i$, since $I-AA^\top$ is the
projection onto the residual space. Write $\overline M_s^{A,z}$ for the
covariance of the response $A\theta+\sqrt{L^*}z$; its ratio to $L^*$
does not depend on $\theta$ or the positive loss scale. For $d<s<m$, put
\begin{equation}
 \alpha_s=(m-s)/n,\qquad \beta_s=(s-d)/n,\qquad
 \gamma_s=(m-s)/(n-1).
 \label{eq:robust-pre-response-constants}
\end{equation}
Subscripts on $\alpha,\beta,\gamma$ are omitted when the size is fixed.
The \emph{contact domain} will mean a real Parseval $A$, $d\ge1$,
$n\ge2$, $L^*>0$, $d<s<m$, and
\begin{equation}
 c_i>0\quad\text{for every }i.
 \label{eq:robust-pre-response-no-coloop}
\end{equation}
The last condition says that deleting any one row preserves rank: no row
is a coloop. The universal envelope just proved has the larger domain of
all full-column-rank designs, including designs with coloops. That larger
domain is needed when applying it inside a row-deletion argument.

\section{Omission moments and comparison across sample sizes}
\label{sec:loo-comparison}\label{sec:query-margin}\label{app:loo-quantitative-geometry}
On the contact domain of Section~\ref{sec:setup}, fix $z\in\mathcal Z_A$.
LOO admits an explicit covariance. We use it as a common reference for all
strict-interior sizes, before classifying any contact space.
With $c_i=1-\ell_i>0$, define the leave-one-out (LOO) deficit operator
\begin{equation}
 D_{\rm LOO}(z):=\sum_{i=1}^m
 \left(1-\frac{z_i^2}{c_i}\right)a_i a_i^\top.
 \label{eq:loo-deficit-operator}
\end{equation}
The complement projection gives $0\leq z_i^2\leq c_i$, so this operator is
positive semidefinite. For $\Xi\succeq0$ with $\tr\Xi>0$, put
\begin{equation}
 \delta_{\rm LOO}(z,\Xi):=
 \frac{\tr(\Xi D_{\rm LOO}(z))}{\tr\Xi}.
 \label{eq:loo-query-deficit}
\end{equation}

Write $\Delta_s(z)=I_d-\overline M_s^{A,z}/(\alpha_sL^*)$ and define
\begin{equation}
 g_s^{(1)}(z,\Xi)=\frac{\tr(\Xi\Delta_s(z))}{\tr\Xi}
 =1-\frac{\tr(\Xi\overline M_s^{A,z})}{\alpha_sL^*\tr\Xi}.
 \label{eq:single-residual-ceiling-deficit}
\end{equation}
\begingroup\makeatletter\edef\@currentlabel{\theequation}\edef\@currentHref{equation.\theHequation}\label{eq:quantitative-single-deficit}\makeatother\endgroup
This measures the unused fraction of the covariance ceiling.

\begin{theorem}[All-size covariance-deficit comparison]
\label{thm:loo-all-size-comparison}
Under the assumptions above, for every integer $d<s<m$, let
$b_s=\dfrac{s-d}{(m-s)(m-d-1)}>0$. Then
\begin{equation}
 \boxed{\quad b_sD_{\rm LOO}(z)\preceq
 I_d-\frac{\overline M_s^{A,z}}{\alpha_sL^*}
 \preceq D_{\rm LOO}(z).\quad}
 \label{eq:all-size-loo-matrix-comparison}
\end{equation}
For every nonzero $\Xi\succeq0$, this gives
\begin{equation}
 b_s\delta_{\rm LOO}(z,\Xi)
 \leq g_s^{(1)}(z,\Xi)
 \leq\delta_{\rm LOO}(z,\Xi)
 =g_{m-1}^{(1)}(z,\Xi).
 \label{eq:all-size-loo-query-comparison}
\end{equation}
At $s=m-1$, $b_s=1$ and both matrix bounds are equalities.
\end{theorem}

\subsection{Supported-law proof of the matrix comparison}

Let $A\in\R^{m\times d}$ have orthonormal columns, $d\geq1$,
$n:=m-d\geq2$, and no coloop. Thus $c_i=1-\|a_i\|_2^2>0$.
Fix $z\in\ker A^\top$ with $\|z\|_2=1$. For an integer $d<s<m$, put
\[
 k:=m-s,\qquad \alpha_s:=\frac{k}{n},\qquad
 t_s:=\frac{k-1}{n-1},\qquad
 \eta_s:=\frac{k(k-1)}{n(n-1)}.
\]
Here $1\leq k\leq n-1$, so every displayed denominator is positive.
The letter $k$ counts omitted rows here, while $K$ elsewhere counts
output columns. On a supported set $S$, define
\[
 p_A(S)=\frac{\det(A_S^\top A_S)}{\binom n{s-d}},\qquad
 v_S=(A_S^\top A_S)^{-1}A_S^\top z_S.
\]
The corresponding selected fit for the source response
$A\theta+\sqrt{L^*}z$ obeys
$v_S=G^{1/2}(w_S-w^*)/\sqrt{L^*}$. The ordinary-law mean is zero, so
\begin{equation}
 \mathbb E_A[v_Sv_S^\top]=\frac{\overline M_s^{A,z}}{L^*}.
 \label{eq:loo-normalized-supported-covariance}
\end{equation}
No fit is assigned to a zero-volume set.

\paragraph{Omission moments.}
Put $P:=I_m-AA^\top$, so $P_{ii}=c_i$, and define
\begin{equation}
 f_i:=-\frac{z_i}{c_i}a_i,\qquad
 L(z):=\sum_i\frac{z_i^2}{c_i}a_i a_i^\top,
 \qquad F_S:=\sum_{i\notin S}f_i.
 \label{eq:loo-omission-predictor}
\end{equation}
Deleting row $i$ leaves the positive definite Gram matrix
$I_d-a_i a_i^\top$. Its full-pool least-squares fit for $z_{-i}$ is
$f_i$: the normal equations have right side $-z_i a_i$ and
$(I_d-a_i a_i^\top)^{-1}a_i=a_i/c_i$. The indexed-volume
normalizer on that deleted pool is
$\binom{n-1}{s-d}c_i$. Thus the original law, conditioned on
$i\notin S$, is exactly ordinary size-$s$ volume sampling from the
deleted pool. The boundary-valid ordinary-law unbiasedness and singleton
exclusion identity give
\begin{equation}
 \mathbb E_A[v_S\mid i\notin S]=f_i,\qquad
 \mathbb P_A(i\notin S)=\alpha_s c_i.
 \label{eq:loo-conditional-omission-mean}
\end{equation}
Consequently, for supported fits,
$\mathbb E_A[v_SF_S^\top]=\alpha_s L(z)$.

For distinct $i,j$ and $k\geq2$, fixed-size Cauchy--Binet applied after
deleting both rows gives
\begin{equation}
 \mathbb P_A(i,j\notin S)
 =\eta_s\det(P_{\{i,j\},\{i,j\}})
 =\eta_s(c_ic_j-P_{ij}^2).
 \label{eq:loo-pair-exclusion}
\end{equation}
It also holds at $k=1$, when both sides are zero because two rows cannot
be omitted. If the doubly deleted pool loses rank, its determinant and
pair-exclusion probability are both zero. Define
\begin{equation}
 H(z):=A^\top\operatorname{diag}(z_i/c_i)
       (P\circ P)\operatorname{diag}(z_i/c_i)A\succeq0,
 \label{eq:loo-omission-schur-moment}
\end{equation}
where $\circ$ is the entrywise product. The Schur product theorem makes
$P\circ P$ positive semidefinite. Since
$\sum_i c_i f_i=-A^\top z=0$, expansion of the singleton and pair
terms gives
\begin{equation}
 \mathbb E_A[v_SF_S^\top]=\alpha_s L(z),\qquad
 \mathbb E_A[F_SF_S^\top]=\alpha_s L(z)-\eta_s H(z).
 \label{eq:loo-omission-moments}
\end{equation}
For clarity, the $c_ic_j$ part of the off-diagonal pair sum equals
$-\sum_i c_i^2f_if_i^\top$; adding the $P_{ij}^2$ part and its diagonal
gives $-H(z)$. At $k=1$, the pair sum and $\eta_s$ are zero.

Expanding the positive semidefinite square of $v_S-F_S$ now yields the
exact supported-law decomposition
\begin{equation}
 \frac{\overline M_s^{A,z}}{L^*}
 =\alpha_sL(z)+\eta_sH(z)
 +\mathbb E_A[(v_S-F_S)(v_S-F_S)^\top]
 \succeq\alpha_sL(z).
 \label{eq:loo-psd-omission-remainder}
\end{equation}
Since $A^\top A=I_d$,
$D_{\rm LOO}(z)=I_d-L(z)$. Dividing
\eqref{eq:loo-psd-omission-remainder} by $\alpha_s$ proves
the right-hand inequality in \eqref{eq:all-size-loo-matrix-comparison}.

\paragraph{Reverse deletion and total covariance.}
Choose the first
omitted row $I=i$ with probability $c_i/n$. Conditional on $I=i$, draw
$S$ by ordinary size-$s$ volume sampling from the deleted pool. Its
mixture law is $p_A$: every supported $S$ has $k$ possible first omitted
rows, and
\[
 \frac{k}{n\binom{n-1}{s-d}}
 =\frac1{\binom n{s-d}}.
\]
Conditional unbiasedness gives $\mathbb E[v_S\mid I=i]=f_i$. The
deleted pool's full-fit residual energy and inverse Gram are
\begin{equation}
 e_i:=1-\frac{z_i^2}{c_i}\geq0,\qquad
 (I_d-a_i a_i^\top)^{-1}=I_d+\frac{a_i a_i^\top}{c_i}.
 \label{eq:loo-deleted-pool-geometry}
\end{equation}
For $k\geq2$, the universal covariance envelope in this full-rank pool
applies to the conditional sample and gives
\begin{equation}
 \operatorname{Cov}(v_S\mid I=i)
 \preceq t_s e_i\left(I_d+\frac{a_i a_i^\top}{c_i}\right).
 \label{eq:loo-conditional-envelope}
\end{equation}
If $e_i=0$, every supported selected fit equals $f_i$, so its
conditional covariance is zero. If $k=1$, the conditional sample is the
entire deleted pool and its covariance is again zero; thus
\eqref{eq:loo-conditional-envelope} holds with $t_s=0$ without applying
an interior formula at the endpoint.

The mixture mean is $\sum_i(c_i/n)f_i=0$. Also
$\sum_i c_i e_i=n-1$ and
$\sum_i e_i a_i a_i^\top=I_d-L(z)$. Total covariance therefore gives
\begin{align}
 \frac{\overline M_s^{A,z}}{L^*}
 &\preceq \frac{L(z)}{n}
 +\frac{t_s}{n}\sum_i c_i e_i
   \left(I_d+\frac{a_i a_i^\top}{c_i}\right)\notag\\
 &=t_sI_d+\frac{n-k}{n(n-1)}L(z).
 \label{eq:loo-reverse-deletion-envelope}
\end{align}
Subtracting this expression from $\alpha_sI_d=(k/n)I_d$ gives
\[
 I_d-\frac{\overline M_s^{A,z}}{\alpha_sL^*}
 \succeq\frac{n-k}{k(n-1)}(I_d-L(z))
 =b_sD_{\rm LOO}(z).
\]
Together with the omission remainder, this proves the matrix comparison.
Taking the trace against $\Xi\succeq0$ proves the scalar comparison.
When $k=1$, $F_S=f_I=v_S$ and $\eta_s=t_s=0$, so
$\overline M_{m-1}^{A,z}/L^*=L(z)/n$ exactly. Hence
\begin{equation}
 I_d-\frac{\overline M_{m-1}^{A,z}}{\alpha_{m-1}L^*}
 =D_{\rm LOO}(z),\qquad \alpha_{m-1}=1/n.
 \label{eq:loo-loo-endpoint}
\end{equation}
The proof uses neither augmented-law inverse moments nor row general
position; all moments stay on ordinary support.

\subsection{The exact leave-one-out endpoint}
\label{subsec:shared-loo}

Let $A^\top A=I_d$, $c_i=1-\|a_i\|^2>0$, $n=m-d\ge1$, and $s=m-1$.
Fix $z\in\ker A^\top$ with $\|z\|=1$ and
$y=A\theta+\sqrt{L^*}z$, $L^*>0$. Fixed-cardinality Cauchy--Binet gives
$\mathbb P_A(S=-i)=c_i/n$. The normal equations for the deleted
pool and Sherman--Morrison give
\[
 \frac{w_{-i}-\theta}{\sqrt{L^*}}
 =-z_i(I_d-a_ia_i^\top)^{-1}a_i=-\frac{z_i}{c_i}a_i.
\]
Summing on this support and using $\sum_i a_ia_i^\top=I_d$ proves
\begin{equation}
 \frac{\overline M_{m-1}^{A,z}}{L^*}
 =\frac1n\sum_i\frac{z_i^2}{c_i}a_ia_i^\top,
 \qquad
 n\left(\alpha_{m-1}I_d-\frac{\overline M_{m-1}^{A,z}}{L^*}\right)
 =\sum_i\frac{c_i-z_i^2}{c_i}a_ia_i^\top.
 \label{eq:quantitative-loo-deletion-identity}
\end{equation}
\begingroup
\makeatletter
\edef\@currentlabel{\theequation}
\edef\@currentHref{equation.\theHequation}
\label{eq:equal-leverage-loo-covariance}
\makeatother
\endgroup
These are pointwise identities in $z$, with $\alpha_{m-1}=1/n$.
The case $n=1$ is rank-size LOO; the strict-interior consequences below
retain $n\ge2$. No inverse or quotient by $c_i$ is used when a coloop is present.

The comparison controls both eigenvalues and query traces. Its lower
factor $b_s$ can be small away from LOO; no size-independent lower factor
is claimed. In particular, near contact the risk is near its covariance
ceiling, which is different from being small in absolute loss units.

\section{Contact spaces and the fixed-design phase}
\label{sec:query-geometry}\label{app:query-geometry}
The comparison across sizes reduces contact to a single PSD kernel. We now
identify its subspace structure. This classification is the geometric
content beyond the matrix comparison.
\begin{theorem}[Fixed-residual contact kernel]
\label{prop:query-contact-space}\label{lem:app-saturated-row-contact-kernel}
Let $d\ge1$, $m\ge d+2$, and $d<s<m$. Suppose that
$A\in\R^{m\times d}$ has orthonormal columns and no coloop. Thus deleting
any one row preserves column rank, or equivalently $\ell_i<1$ for every row.
For every $L^*>0$ and fixed compatible residual $z\in\mathcal Z_A$,
\begin{equation}
 \boxed{\ker\!\left(\alpha_sI_d-\overline M_s^{A,z}/L^*\right)
 =\ker R_A(z).}
 \label{eq:query-contact-kernel}
\end{equation}
Here $R_A(z)$ is the residual operator in \eqref{eq:response-geometry-main}.
The equality holds at every strict-interior size for this same fixed residual.
\end{theorem}

\begin{proof}
Set $c_{\min}=\min_i c_i>0$ and $c_{\max}=\max_i c_i$. The definitions give
\[
 c_{\min}D_{\rm LOO}(z)\preceq R_A(z)
 \preceq c_{\max}D_{\rm LOO}(z).
\]
Theorem~\ref{thm:loo-all-size-comparison} has positive comparison factors,
so its three PSD operators have the same kernel. Therefore
\begin{equation}
 \ker\Delta_s(z)=\ker D_{\rm LOO}(z)=\ker R_A(z).
 \label{eq:loo-kernel-consequence}
\end{equation}
This argument is valid for the same fixed $z$ at every interior size.
\end{proof}

\subsection{Classification by complement classes}
\label{subsec:shared-contact-classes}
Choose $N\in\R^{m\times n}$ with orthonormal columns spanning $\ker A^\top$.
Its rows satisfy $\|n_i\|^2=c_i>0$. Normalize them as
$\widehat n_i=n_i/\sqrt{c_i}$ and partition the unit rows into maximal
parallel classes, allowing opposite signs. For each class $C$, choose a
unit representative $\widehat n_C$ and put $V_C=\ker A_{-C}$.

\begin{theorem}[Contact-space and whole-query classification]
\label{cor:query-contact-classification}
On the contact domain, the possible nonzero contact spaces are exactly
$V_C$ for classes with $|C|\ge2$. They have dimension $|C|-1$ and are
mutually orthogonal. A residual has nonzero contact precisely when
$z=\pm N\widehat n_C$ for one such class; its contact space is then $V_C$.
For any nonzero $\Xi\succeq0$,
\begin{equation}
 \exists z\in\mathcal Z_A:\quad
 \tr(\Xi\overline M_s^{A,z})=\alpha_sL^*\tr\Xi
 \quad\Longleftrightarrow\quad
 \operatorname{range}\Xi\subseteq V_C\text{ for one class }C.
 \label{eq:query-contact-classification}
\end{equation}
The witness $\pm N\widehat n_C$ attains every direction in the query
range simultaneously, at every strict-interior size.
\end{theorem}
\begin{proof}
Write $z=N\xi$ with $\|\xi\|=1$. Then
$z_i^2=c_i|\widehat n_i^\top\xi|^2$, so saturation $z_i^2=c_i$ means
$\widehat n_i=\pm\xi$. The saturated rows are either empty or one entire
maximal class. If empty, every coefficient of $R_A(z)$ is positive and
$R_A(z)\succ0$. If they form $C$, the remaining positive coefficients give
\[
 \ker R_A(z)=\{v:\operatorname{supp}(Av)\subseteq C\}=V_C.
\]
The isometry $v\mapsto Av$ identifies $V_C$ with vectors supported on $C$
and orthogonal to the rank-one matrix $N_C$. Thus $\dim V_C=|C|-1$.
Images from distinct classes have disjoint supports, so the spaces are
orthogonal. A singleton contributes only the zero space.

Finally, a PSD trace deficit vanishes exactly when
$\Delta_s(z)^{1/2}\Xi^{1/2}=0$, equivalently
$\operatorname{range}\Xi\subseteq\ker\Delta_s(z)$. Apply the kernel
identity and the classification just proved.
\end{proof}
For a rank-$r$ query, one class needs at least $r+1$ rows and must contain
the query range in its deletion nullspace. For $\Xi\succ0$, contact requires
$V_C=\R^d$, hence $|C|=d+1$ and every row outside $C$ is zero.
A zero feature row has a complement direction orthogonal to all other
complement rows, so cannot belong to such a nontrivial class. Conversely,
a rank-$d$ no-coloop core with $d+1$ nonzero rows has a unique dependence
with no zero coordinate; its unit dependence extended by zeros saturates
the core and gives $R_A(z)=0$. This proves the full-rank classification.
For a full-rank query $\Xi\succ0$, when all feature rows are nonzero, compactness of the residual sphere gives
a positive response-uniform trace gap on that fixed design.

\subsection{The fixed-design attainability phase}
\label{sec:robust-pre-response-geometry}\label{app:proof-robust-pre-response-geometry}
Define the feature-only margin and the response-uniform spectral target by
\begin{equation}
 \nu_A=\min_{z\in\mathcal Z_A}\lambda_{\min}(R_A(z)),\qquad
 \mathsf T_{\rm spec}^{\rm RU}(A,s)
 =\max_{z\in\mathcal Z_A}\lambda_{\max}(\overline M_s^{A,z})/L^*.
 \label{eq:robust-pre-response-margin}
\end{equation}
\begingroup\makeatletter\edef\@currentlabel{\theequation}\edef\@currentHref{equation.\theHequation}\label{eq:robust-pre-response-target}\makeatother\endgroup
Spectral contact means equality in at least one coefficient direction.
It does not require equality of the whole covariance matrix.
\begin{theorem}[Fixed-design covariance-envelope attainability phase]
\label{thm:robust-pre-response-phase}
On the contact domain, $\nu_A\ge0$ and
\begin{align}
 \nu_A>0&\quad\Longleftrightarrow\quad
 \mathsf T_{\rm spec}^{\rm RU}(A,s)<\alpha_s,
 \label{eq:robust-pre-response-strict}\\
 \nu_A=0&\quad\Longleftrightarrow\quad
 \mathsf T_{\rm spec}^{\rm RU}(A,s)=\alpha_s.
 \label{eq:robust-pre-response-tight}
\end{align}
In the zero branch, one residual $z_*$ satisfies, simultaneously for all
$s'\in\{d+1,\ldots,m-1\}$,
\begin{equation}
 \lambda_{\max}(\overline M_{s'}^{A,z_*})/L^*=(m-s')/(m-d).
 \label{eq:robust-pre-response-common-witness}
\end{equation}
\end{theorem}
\begin{proof}
The operator $R_A(z)$ is PSD and continuous on the compact residual sphere.
For fixed $A,s$, the ordinary support and probabilities do not depend on
$z$, while each supported coefficient error is linear in $z$. Thus its
finite covariance sum is quadratic in $z$ and continuous. Both extrema
are attained. The kernel theorem now gives $\nu_A=0$ exactly when some
residual makes spectral contact; the same kernel works at every size.
If $\nu_A>0$, all residual covariances are strictly below the ceiling,
and compactness makes the response-uniform maximum strictly smaller.
No continuity across changing design supports is used.
\end{proof}
The quantitative size of the slack requires more information; the
residual-augmented bounds in Section~\ref{sec:mechanism} provide it.
\subsection{A feature-only phase test and lower margin}
The definition of $\nu_A$ is a residual-sphere optimization.  Its sign,
however, has an exact feature-only characterization on the same real
Parseval no-coloop domain.  Put
\[
  P_\perp=I_m-AA^\top,\qquad c_i=(P_\perp)_{ii}=1-\ell_i,
\]
and define
\begin{align*}
  \rho_\perp(A)&=\max_{i\ne j}
    \frac{\lvert(P_\perp)_{ij}\rvert}{\sqrt{c_ic_j}},
  &\kappa_\perp(A)&=\min_{k\in[m]}
    \lambda_{\min}\left(\sum_{i\ne k}c_i a_i a_i^\top\right),\\
  t_\perp(A)&=(1-\rho_\perp(A))\kappa_\perp(A).
\end{align*}
Here $P_\perp$ is the classical Naimark-complement Gram matrix
\citep{casazza2013naimark}.  The leave-one-out factor below is a particular
weighted lower-frame quantity of erasure-stability type
\citep{fickus2012nerf}.  Coupling these objects to the residual margin gives
the following phase test and certificate.

\begin{proposition}[Complement minor test and margin certificate]
\label{prop:global-complement-certificate}
Under the assumptions of Theorem~\ref{thm:robust-pre-response-phase},
\begin{equation}
  \nu_A=0
  \quad\Longleftrightarrow\quad
  \rho_\perp(A)=1
  \quad\Longleftrightarrow\quad
  \det\bigl((P_\perp)_{\{i,j\},\{i,j\}}\bigr)=0
  \quad\text{for some }i\ne j.
  \label{eq:global-complement-zero-equivalence}
\end{equation}
Moreover,
\begin{equation}
  \kappa_\perp(A)>0,
  \qquad
  0\le t_\perp(A)\le\nu_A.
  \label{eq:global-complement-certificate}
\end{equation}
\end{proposition}

\begin{proof}
The classification gives $\nu_A=0$ exactly when two normalized complement
rows are parallel. This is $\rho_\perp=1$, equivalently a zero pair minor.
For $J=\{i,j\}$, Sylvester's determinant identity also gives
\begin{equation}
 \det((P_\perp)_{J,J})=\det(I_d-A_J^\top A_J)
 =\det(A_{-J}^\top A_{-J}).
 \label{eq:global-complement-rank-deletion-main}
\end{equation}
Thus the zero phase is exactly rank loss after deleting a pair of rows.
For the lower certificate, write $u_i=\widehat n_i$. Then
\begin{equation}
 z_i=\sqrt{c_i}\,u_i^\top\xi,\qquad
 c_i-z_i^2=c_i(1-|u_i^\top\xi|^2).
 \label{eq:app-global-complement-residual-coordinate}
\end{equation}
It remains to prove the lower certificate.  Let
$c_{\min}:=\min_i c_i>0$ and
\[
  S_k:=\sum_{i\ne k}c_i a_i a_i^\top.
\]
Parsevalness and no-coloop give
\begin{equation}
  S_k
  \succeq c_{\min}\sum_{i\ne k}a_i a_i^\top
  =c_{\min}(I_d-a_ka_k^\top)
  \succeq c_{\min}c_k I_d
  \succeq c_{\min}^2I_d.
  \label{eq:app-global-complement-kappa-floor}
\end{equation}
The middle spectral inequality follows from
$a_ka_k^\top\preceq\lVert a_k\rVert_2^2I_d$.  Therefore
$\kappa_\perp(A)\ge c_{\min}^2>0$.

We use the following weighted Gram estimate.  If $u_1,\ldots,u_m$ are unit
vectors with $\max_{i\ne j}\lvert u_i^\top u_j\rvert\le\rho$ and
$w_i\ge0$, then
\begin{equation}
  \lambda_{\max}\left(\sum_iw_i u_i u_i^\top\right)
  \le
  \rho\sum_iw_i+(1-\rho)\max_iw_i.
  \label{eq:app-global-complement-weighted-gram}
\end{equation}
To verify its direction, let $U$ have rows $u_i^\top$,
$D=\operatorname{diag}(w_1,\ldots,w_m)$, and let $y\in\R^m$ be unit.
The nonzero eigenvalues of $U^\top DU$ and
$D^{1/2}UU^\top D^{1/2}$ agree, while
\begin{align*}
  y^\top D^{1/2}UU^\top D^{1/2}y
  &\leq \sum_iw_i y_i^2
       +\rho\sum_{i\ne j}\sqrt{w_iw_j}\lvert y_i y_j\rvert\\
  &=(1-\rho)\sum_iw_i y_i^2
       +\rho\left(\sum_i\sqrt{w_i}\lvert y_i\rvert\right)^2\\
  &\leq (1-\rho)\max_iw_i+\rho\sum_iw_i.
\end{align*}
This proves the upper, rather than lower, spectral bound in
\eqref{eq:app-global-complement-weighted-gram}; no common sign is imposed on
the off-diagonal complement inner products.

Fix unit $v\in\R^d$ and $z=N\xi\in\mathcal Z_A$, and set
\[
  w_i:=c_i(a_i^\top v)^2,\qquad
  W:=\sum_iw_i,\qquad
  w_\star:=\max_iw_i.
\]
Using \eqref{eq:app-global-complement-residual-coordinate} and then
\eqref{eq:app-global-complement-weighted-gram} with
$\rho=\rho_\perp(A)$ gives
\begin{align}
  v^\top R_A(z)v
  &=W-\xi^\top\left(\sum_iw_i u_i u_i^\top\right)\xi\notag\\
  &\geq(1-\rho_\perp(A))(W-w_\star).
  \label{eq:app-global-complement-weighted-direction}
\end{align}
Choose $k_\star$ with $w_{k_\star}=w_\star$.  Then
\begin{equation}
  W-w_\star
  =v^\top S_{k_\star}v
  \geq\lambda_{\min}(S_{k_\star})
  \geq\kappa_\perp(A).
  \label{eq:app-global-complement-kappa-direction}
\end{equation}
These bounds hold for every unit $v$ and every $z\in\mathcal Z_A$, so
\[
  R_A(z)\succeq t_\perp(A)I_d
  \qquad(z\in\mathcal Z_A).
\]
Minimizing over $z$ proves $t_\perp(A)\le\nu_A$, while
$\rho_\perp(A)\in[0,1]$ and $\kappa_\perp(A)>0$ give
$t_\perp(A)\ge0$.
Combining this conclusion with
\eqref{eq:app-global-complement-kappa-floor} gives the eigenvalue-free coarse
floor
\begin{equation}
 0\le (1-\rho_\perp(A))c_{\min}^2
 \le t_\perp(A)\le\nu_A.
 \label{eq:app-global-complement-cmin-floor}
\end{equation}
This floor discards the design-specific leave-one-out factor
$\kappa_\perp(A)$, so it is a supplementary analytic check rather than a
positive-margin approximation.
\end{proof}

\subsection{Directional contact, query contact, and endpoints}
\label{sec:tiny-spectral-separation}
\paragraph{Duplicated-basis enumeration.}
Consider the duplicated-basis design
\begin{equation}
 A=\frac1{\sqrt2}\begin{bmatrix}I_2\\ I_2\end{bmatrix},\qquad
 (m,d,s)=(4,2,3),\qquad \Xi=I_2.
 \label{eq:tiny-duplicated-basis}
\end{equation}
Every compatible unit residual is
$z=2^{-1/2}(u,-u)^\top$ for a unit $u=(u_1,u_2)^\top$. Enumeration of the
supported size-three subsets gives
\begin{equation}
 \frac{\overline M_3^{A,z}}{L^*}
 =\frac12\operatorname{diag}(u_1^2,u_2^2),
 \label{eq:tiny-duplicated-basis-covariance}
\end{equation}
so the unnormalized trace ratio is $1/2$ although simultaneous isotropic
contact is impossible.

The core-plus-zero family~\eqref{eq:ue-attaining-family} has full matrix
contact at $(d,m,s)=(2,4,3)$; a repeated-coordinate extension has rank-two
query contact at $(3,5,4)$. By contrast, the duplicated-basis design above
has spectral target $1/2$ but isotropic normalized trace target $1/4$.

The assumptions matter: the coloop obstruction is
\eqref{eq:app-robust-pre-response-coloop-example}. At rank size, the
row-general-position design
$X=[(1,0);(0,1);(1,1);(1,2)]$ and $y=(-1,-1,1,0)^\top$ have
$\overline M_d/L^*=I_2$ although $R_A(z)\succ0$.
Full, zero loss, and the zero query metric have zero unnormalized excess;
$m=d+1$ has no strict-interior size.
\subsection{Scope boundaries}

The no-coloop condition is sufficient, not asserted to be necessary.  It
cannot be removed wholesale: for
\begin{equation}
  d=1,\qquad m=3,\qquad s=2,\qquad A=(1,0,0)^\top,
  \label{eq:app-robust-pre-response-coloop-example}
\end{equation}
the margin is zero, but every positive-volume sample contains the
leverage-one row and the selected coefficient is deterministic, so
$\mathsf T_{\mathrm{spec}}^{\mathrm{RU}}(A,2)=0<1/2=\alpha_2$.  The theorem
also requires $L^*>0$ and the strict-interior regime $d<s<m$; its normalized
residual and covariance ratio are not evaluated at $L^*=0$, $s=d$, or $s=m$.
The ordinary sampler, selected unweighted estimator, and centered
full-Gram-whitened spectral covariance are part of the statement.  No
rescaling, replacement, reweighting, ridge term, alternative estimator, or
different covariance or loss metric is used.

\section{Residual augmentation and moment bounds}
\label{sec:mechanism}\label{sec:residual-augmented-resolvent}
The LOO operator controls the common kernel and gives global two-sided
comparison. To obtain further upper covariance bounds, append the realized
residual to the design. This introduces a response-dependent analysis law;
ordinary volume sampling remains the actual sampler.
Let $B=[A\ z]$, $K_S=A_S^\top A_S$, $h_i=\ell_i+z_i^2$, and
$R=R_A(z)$. The auxiliary law is
\begin{equation}
 P_B(S)=\frac{\det(B_S^\top B_S)}{\binom{m-d-1}{s-d-1}}.
 \label{eq:augmented-law-main}
\end{equation}
\begin{theorem}[Residual-augmented representation and response-aware resolvent]
\label{thm:residual-mechanism}
Suppose $X$ has full column rank, $\lossstar>0$, and $d<s<m$.
On ordinary-supported sets the change of measure, together with the
augmented-law first moment, is
\begin{equation}
 P_X(S)L_S=\beta\lossstar P_B(S),\qquad
 \E_BK_S=I_d-\gamma R.
 \label{eq:augmented-transform-main}
\end{equation}
The centered covariance satisfies these one-sided bounds:
\begin{equation}
 \boxed{\overline M_s\preceq
 \lossstar\left[I_d-\beta(I_d-\gamma R)^{-1}\right]}
 \preceq\alpha\lossstar I_d-\beta\gamma\lossstar R.
 \label{eq:resolvent-main}
\end{equation}
If $X$ is additionally in row general position, the exact identity
\begin{equation}
 \overline M_s/\lossstar=I_d-\beta\E_B[K_S^{-1}]
 \label{eq:augmented-covariance-identity}
\end{equation}
also holds. For an arbitrary full-column-rank boundary design,
the raw-coordinate resolvent remains one-sided. With
$Q=\sum_i(1-h_i)x_ix_i^\top$,
\begin{equation}
 M_s\preceq\lossstar\left[G^{-1}-\beta(G-\gamma Q)^{-1}\right].
 \label{eq:boundary-resolvent-main}
\end{equation}
The exact inverse-moment covariance identity can fail when the support
changes at a boundary design.
\end{theorem}
\subsection{Whitening and the exact change of measure}
\label{subsec:residual-augmented-change-measure}

Assume $X$ has full column rank, $L^*>0$, and $d<s<m$. Define
\begin{equation}
  A=XG^{-1/2},\qquad z=\frac{e}{\sqrt{L^*}},\qquad B=[A\ \ z].
  \label{eq:residual-augmented-whitening}
\end{equation}
The full normal equations give
\begin{equation}
  A^\top A=I_d,\qquad A^\top z=0,\qquad
  \lVert z\rVert_2=1,\qquad B^\top B=I_{d+1}.
  \label{eq:residual-augmented-orthonormality}
\end{equation}
For each $S$, let $K_S=A_S^\top A_S$.  The augmented volume law is a
different law, of rank $d+1$, with its own support and normalizer:
\begin{equation}
  Z_{B,s}=\sum_{|S|=s}\det(B_S^\top B_S)
  =\binom{m-d-1}{s-d-1},\qquad
  \mathbb P_B(S)=\frac{\det(B_S^\top B_S)}{Z_{B,s}}.
  \label{eq:residual-augmented-law}
\end{equation}
Here $P_B$ is a response-dependent analysis law induced by the full residual,
not the sampler for $w_S$; the estimator continues to sample and fit under
the ordinary feature-only law $P_X$.
An augmented-supported set has $\operatorname{rank}(B_S)=d+1$, hence
$\operatorname{rank}(A_S)=d$ and $K_S\succ0$.  Thus every inverse under
$\mathbb E_B$ below is taken only on augmented support.  No inverse is
assigned to a zero-augmented-volume set.  Augmented support is contained in,
and can be strictly smaller than, ordinary support.

For an ordinary-supported set, the Schur determinant formula and
$y_S=X_Sw^*+\sqrt{L^*}z_S$ give
\begin{equation}
  \det(B_S^\top B_S)
  =\det(K_S)\left(
    \lVert z_S\rVert_2^2-z_S^\top A_SK_S^{-1}A_S^\top z_S
  \right)
  =\det(K_S)\frac{L_S}{L^*}.
  \label{eq:residual-augmented-setwise-schur}
\end{equation}
If $S$ is not ordinary-supported, then
$\operatorname{rank}(A_S)<d$ and $\operatorname{rank}(B_S)\le d$, so both
determinants in the corresponding zero-equals-zero rank statement vanish;
neither $w_S$ nor $L_S$ is defined there.  Also
$\det(K_S)=D_S/\det(G)$.  Consequently, if the finite measure
$\mu_X$ is defined by
\begin{equation}
  \mu_X(\{S\})=
  \begin{cases}
    \mathbb P_X(S)L_S,&D_S>0,\\
    0,&D_S=0,
  \end{cases}
  \label{eq:residual-augmented-loss-measure}
\end{equation}
then the two distinct normalizers in
\eqref{eq:law} and
\eqref{eq:residual-augmented-law} yield the exact setwise change of measure
\begin{equation}
  \begin{gathered}
    \boxed{
    \mathbb P_X(S)L_S=\beta L^*\mathbb P_B(S)\quad(D_S>0),
    \qquad
    \mu_X(\{S\})=\beta L^*\mathbb P_B(S)\quad(\text{all }S)} ,\\
    \dfrac{Z_{B,s}}{\binom{m-d}{s-d}}=\beta.
  \end{gathered}
  \label{eq:residual-augmented-measure-change}
\end{equation}
The zero branch in \eqref{eq:residual-augmented-loss-measure} extends only a
measure, not an inverse fit.

\subsection{Covariance transform and resolvent}
\label{subsec:residual-augmented-covariance-transform}
For row-general-position $X$, Proposition~\ref{prop:ue-rgp-covariance}
already gives $M_s=L^*G^{-1}-\E_X[L_SG_S^{-1}]$.
Congruence by $G^{1/2}$ and the change of measure therefore yield
\begin{equation}
 \overline M_s/L^*=I_d-\beta\E_B[K_S^{-1}].
 \label{eq:residual-augmented-exact-covariance-transform}
\end{equation}
Only this exact covariance identity requires row general position.
\subsection{Augmented exclusion marginal and the augmented Gram first moment}
\label{subsec:residual-augmented-deletion}

Let $a_i^\top$ and $b_i^\top$ denote row $i$ of $A$ and $B$, respectively,
and define
\begin{equation}
  h_i=\lVert b_i\rVert_2^2,\qquad
  R=\sum_{i=1}^m(1-h_i)a_ia_i^\top.
  \label{eq:residual-augmented-response-geometry}
\end{equation}
The fixed-cardinality Cauchy--Binet step is part of the ordinary volume-law
machinery~\citep{derezinski2018reverse}; applying it to $B_{-i}$ gives the
augmented exclusion marginal
\begin{align}
  \mathbb P_B(i\notin S)
  &=\frac{\binom{m-d-2}{s-d-1}
     \det(B_{-i}^\top B_{-i})}
     {\binom{m-d-1}{s-d-1}} \notag\\
  &=\frac{m-s}{m-d-1}\det(I_{d+1}-b_ib_i^\top)
    =\gamma(1-h_i).
  \label{eq:residual-augmented-deletion-probability}
\end{align}
The last equality is the rank-one determinant lemma.  In particular, the
augmented exclusion factor is $\gamma(1-h_i)$, not a reciprocal factor.  Since
$\sum_i a_ia_i^\top=I_d$, linearity now gives
\begin{equation}
  \boxed{
  \mathbb E_BK_S
  =\sum_i\mathbb P_B(i\in S)a_ia_i^\top
  =I_d-\gamma R\succ0.}
  \label{eq:residual-augmented-first-moment}
\end{equation}
Strict positivity follows because every augmented-supported $S$ has
$K_S\succ0$ and the augmented law has nonempty support.

\subsection{Classical operator Jensen and two resolvent bounds}
\label{subsec:residual-augmented-jensen}

The classical operator Jensen inequality for the inversion map on the
positive-definite cone~\citep{hansen2003jensen} gives
\begin{equation}
  \mathbb E_B[K_S^{-1}]
  \succeq(\mathbb E_BK_S)^{-1}
  =(I_d-\gamma R)^{-1}.
  \label{eq:residual-augmented-operator-jensen}
\end{equation}
This generic Jensen step is classical; the volume-specific inputs are the
change of measure and augmented exclusion first moment above.  Substitution into
\eqref{eq:residual-augmented-exact-covariance-transform} reverses the final
order because of its minus sign and proves
\begin{equation}
  \boxed{
  \overline{M}_s\preceq
  L^*\left[I_d-\beta(I_d-\gamma R)^{-1}\right].}
  \label{eq:residual-augmented-resolvent-envelope}
\end{equation}
Moreover, $R\succeq0$ and $I_d-\gamma R\succ0$, so scalar functional
calculus gives
$(I_d-\gamma R)^{-1}\succeq I_d+\gamma R$.  The resulting linear bound is
\begin{equation}
  \boxed{
  \overline{M}_s
  \preceq \alpha L^*I_d-\beta\gamma L^*R.}
  \label{eq:residual-augmented-linear-corollary}
\end{equation}
The preceding change-of-measure and augmented exclusion-moment identities use
the volume-law structure; no high-dimensional sharpness statement is made for
\eqref{eq:residual-augmented-resolvent-envelope}.

\subsection{Second-centered-Gram refinement of the response-aware resolvent}
\label{subsec:residual-augmented-second-moment}

\begin{proposition}[Contraction-specific second-moment refinement]
\label{prop:residual-augmented-second-moment}
Assume that $X$ has full column rank and that
$L^*>0$ and $d<s<m$.  Let $A,z,B$ and $\mathbb P_B$ be as in
\eqref{eq:residual-augmented-whitening}--\eqref{eq:residual-augmented-law},
and write
\begin{equation}
  \Omega_B=\{S\subseteq[m]: |S|=s,\ 
  \det(B_S^\top B_S)>0\}.
  \label{eq:second-moment-positive-support}
\end{equation}
The law $\mathbb P_B$ is analysis-only, and every selected inverse below is
restricted to $\Omega_B$.  Define
\begin{equation}
  H_s=\mathbb E_BK_S,\qquad
  \Delta_S=K_S-H_s,\qquad
  V_s=\mathbb E_B[\Delta_S^2],\qquad
  \mathcal Q_s=H_s^{-1}V_sH_s^{-1}.
  \label{eq:second-moment-definitions}
\end{equation}
Then $0\prec K_S\preceq I_d$ on $\Omega_B$, and
\begin{equation}
  \mathbb E_B[K_S^{-1}]-H_s^{-1}
  =\mathcal Q_s+H_s^{-1}\mathbb E_B\!\left[
    \Delta_S(K_S^{-1}-I_d)\Delta_S
  \right]H_s^{-1}
  \succeq\mathcal Q_s\succeq0.
  \label{eq:second-moment-ordered-remainder}
\end{equation}
The coefficient of $V_s$ is one: this is an exact ordered inverse remainder,
not a Taylor approximation.  If $X$ is in row general position, the exact
covariance transform gives
\begin{equation}
  \overline M_s\preceq
  L^*\left[I_d-\beta H_s^{-1}-\beta\mathcal Q_s\right]
  \preceq L^*\left[I_d-\beta H_s^{-1}\right].
  \label{eq:second-moment-interior-hierarchy}
\end{equation}

For every full-column-rank design in the same strict domain, without a row
general position assumption, define
\begin{equation}
  F_s=\mathbb E_BG_S,\qquad
  W_s=\mathbb E_B[(G_S-F_s)G^{-1}(G_S-F_s)]\succeq0.
  \label{eq:second-moment-raw-definitions}
\end{equation}
For the ordinary-law supported, full-fit-centered second moment $M_s$ in
\eqref{eq:covariance-metric}, only the one-sided
hierarchy
\begin{equation}
  M_s\preceq L^*\left[G^{-1}-\beta F_s^{-1}
    -\beta F_s^{-1}W_sF_s^{-1}\right]
  \preceq L^*\left[G^{-1}-\beta F_s^{-1}\right]
  \label{eq:second-moment-boundary-hierarchy}
\end{equation}
is asserted.  Outside the row-general-position interior this is not an exact
inverse-moment covariance identity or an exact covariance-gap decomposition.
The proposition is not invoked at the existing zero branches $L^*=0$,
$s=m$, or $m=d$, nor at $s=d<m$; when $m=d+1$ the strict domain is empty.
\end{proposition}

\begin{proof}
For $S\in\Omega_B$, augmented support implies $K_S\succ0$, while
$I_d-K_S=A_{S^c}^\top A_{S^c}\succeq0$.  Thus $K_S$ is a positive
contraction and $H_s\succ0$.  With $K=K_S$, $H=H_s$, and $\Delta=K-H$,
direct multiplication, without commuting any factors, gives
\begin{equation}
  K^{-1}=H^{-1}-H^{-1}\Delta H^{-1}
  +H^{-1}\Delta K^{-1}\Delta H^{-1}.
  \label{eq:second-moment-resolvent-identity}
\end{equation}
Taking expectations removes the centered linear term.  Splitting
$K_S^{-1}=I_d+(K_S^{-1}-I_d)$ in the last term yields
\eqref{eq:second-moment-ordered-remainder}.  Its remainder is positive
semidefinite because $K_S^{-1}-I_d\succeq0$.  Thus the first-order comparison
is the classical operator-Jensen step for inversion
\citep{hansen2003jensen}, whereas \eqref{eq:second-moment-ordered-remainder}
is the support-restricted resolvent calculation for these contractions.
Combining it with
\eqref{eq:residual-augmented-exact-covariance-transform} proves
\eqref{eq:second-moment-interior-hierarchy}.

On a row-general-position design,
$F_s=G^{1/2}H_sG^{1/2}$ and
$W_s=G^{1/2}V_sG^{1/2}$, so congruence of
\eqref{eq:second-moment-interior-hierarchy} gives the raw-coordinate
hierarchy.  The following subsection supplies the directional-liminf passage
that retains only this one-sided hierarchy on an arbitrary full-column-rank
boundary, without continuing any selected inverse through zero support.
\end{proof}

\subsection{One-sided extension to arbitrary full-column-rank boundary designs}
\label{subsec:residual-augmented-boundary-extension}

Now let $X$ be an arbitrary full-column-rank design, still with $L^*>0$ and
$d<s<m$; row general position is no longer assumed.  The quantities
$A,z,B,h_i$ remain defined by
\eqref{eq:residual-augmented-whitening} and
\eqref{eq:residual-augmented-response-geometry}.  In original coordinates
put
\begin{equation}
  Q=\sum_{i=1}^m(1-h_i)x_ix_i^\top.
  \label{eq:residual-augmented-original-response-geometry}
\end{equation}
The augmented exclusion calculation itself did not require row general position.
Because augmented support implies $G_S\succ0$, it gives the exact first
moment
\begin{equation}
  F_s=\mathbb E_BG_S
  =G-\gamma Q\succ0.
  \label{eq:residual-augmented-original-first-moment}
\end{equation}
The first-order member of
Proposition~\ref{prop:residual-augmented-second-moment} is therefore the
response-aware upper inequality
\begin{equation}
  \boxed{
  M_s\preceq L^*\left[G^{-1}
    -\beta(G-\gamma Q)^{-1}\right].}
  \label{eq:residual-augmented-boundary-resolvent}
\end{equation}

For the extension, apply the common support argument of
Proposition~\ref{prop:ue-boundary-envelope} to row-general-position
$X_t\to X$ with fixed $y$. The full quantities $G,w^*,L^*,A,z$ are
continuous near a positive-loss full-rank design. The finite determinant
sums defining $F_s$ and $W_s$ are continuous even when augmented support
changes, since their summands involve Grams, not selected inverses.
Moreover $F_s\succ0$ on augmented support. Thus the right sides of
\eqref{eq:second-moment-boundary-hierarchy} are continuous, whereas the
ordinary covariance obeys the one-sided directional liminf. This proves
the hierarchy on every full-rank design. It also proves the resolvent and
linear bounds, completing Theorem~\ref{thm:residual-mechanism}.
\begin{proposition}[Boundary failure of the exact inverse-moment identity]
\label{prop:residual-augmented-boundary-counterexample}
Let
\begin{equation}
  X=\begin{bmatrix}
    1&0\\0&1\\1&0\\0&1\\0&0
  \end{bmatrix},\qquad
  y=(1,-1,2,0,3)^\top,\qquad s=3.
  \label{eq:residual-augmented-boundary-counterexample-data}
\end{equation}
Then $G=2I_2$, $L^*=10$, and $M_s=I_2/6$, whereas
\begin{equation}
  \mathbb E_B[G_S^{-1}]=\frac{39}{40}I_2,\qquad
  M_s-L^*\left\{G^{-1}-\beta\mathbb E_B[G_S^{-1}]\right\}
  =-\frac{19}{12}I_2\ne0.
  \label{eq:residual-augmented-boundary-identity-defect}
\end{equation}
The boundary resolvent inequality nevertheless remains strict, with slack
$(209/126)I_2\succ0$.
\end{proposition}

\begin{proof}
Here $w^*=(3/2,-1/2)^\top$ and
$e=(-1/2,-1/2,1/2,1/2,3)^\top$, which gives the stated $G$ and $L^*$.
Put $C_1=\{1,3\}$ and $C_2=\{2,4\}$.  The eight ordinary-supported
triples, whose normalizer is $Z_{X,3}=12$, split as follows:
\begin{center}
\begin{tabular}{c@{\quad}c@{\quad}c@{\quad}c@{\quad}c}
type of $S$ & count & $D_S$ & $w_S-w^*$ & $L_S$ \\
\midrule
$C_1\cup\{j\}$, $j\in C_2$ & $2$ & $2$ & $(0,\,\pm1/2)^\top$ & $1/2$ \\
$\{i\}\cup C_2$, $i\in C_1$ & $2$ & $2$ & $(\pm1/2,\,0)^\top$ & $1/2$ \\
$\{i,j,5\}$, $(i,j)\in C_1\times C_2$ & $4$ & $1$ & $(\pm1/2,\,\pm1/2)^\top$ & $9$
\end{tabular}
\end{center}
All four sign combinations occur in the last row.  Weighting by
$\mathbb P_X(S)=D_S/12$, the three rows contribute respectively
$\operatorname{diag}(0,1/12)$, $\operatorname{diag}(1/12,0)$, and
$I_2/12$ to the centered second moment.  Hence $M_s=I_2/6$.

For the augmented law, $Z_{B,3}=1$.  The setwise Schur determinant gives
\begin{equation}
  \mathbb P_B(S)=\det(B_S^\top B_S)
  =\frac{D_SL_S}{\det(G)L^*}=\frac{D_SL_S}{40}.
  \label{eq:residual-augmented-boundary-counterexample-weights}
\end{equation}
Thus the per-triple probabilities in the three rows are, respectively,
$1/40$, $1/40$, and $9/40$.  Their inverse Gramians are
$\operatorname{diag}(1/2,1)$, $\operatorname{diag}(1,1/2)$, and $I_2$.
Exact summation therefore gives
\begin{equation}
  \mathbb E_B[G_S^{-1}]
  =\frac{2}{40}\operatorname{diag}(1/2,1)
   +\frac{2}{40}\operatorname{diag}(1,1/2)
   +\frac{36}{40}I_2
  =\frac{39}{40}I_2.
\end{equation}
Since $\beta=1/3$, substitution yields the defect in
\eqref{eq:residual-augmented-boundary-identity-defect}.

Finally,
$h_i=x_i^\top G^{-1}x_i+e_i^2/L^*=21/40$ for $i\le4$, while
$h_5=9/10$.  Hence $\gamma=1$, $Q=(19/20)I_2$, and
$G-\gamma Q=(21/20)I_2$.  The exact boundary slack is therefore
\begin{equation}
  L^*\left[G^{-1}-\beta(G-\gamma Q)^{-1}\right]-M_s
  =10\left(\frac12-\frac13\frac{20}{21}\right)I_2-\frac16I_2
  =\frac{209}{126}I_2\succ0.
\end{equation}
\end{proof}

\subsection{Quantitative transfer from geometry to covariance}
\begin{corollary}[Margin-to-slack bound]
\label{cor:margin-slack}
On the contact domain,
\begin{equation}
 \alpha_s-\mathsf T_{\rm spec}^{\rm RU}(A,s)
 \ge\beta_s\gamma_s\nu_A\ge\beta_s\gamma_s t_\perp(A).
 \label{eq:robust-pre-response-slack}
\end{equation}
\begingroup\makeatletter\edef\@currentlabel{\theequation}\edef\@currentHref{equation.\theHequation}\label{eq:global-complement-covariance-consequence}\makeatother\endgroup
More generally any verified $0\le t\le\nu_A$ may replace $t_\perp$.
\end{corollary}
\begin{proof}
For every compatible residual, the linear resolvent bound gives
\begin{equation}
 \overline M_s^{A,z}/L^*
 \preceq\alpha_sI_d-\beta_s\gamma_sR_A(z)
 \preceq(\alpha_s-\beta_s\gamma_s\nu_A)I_d.
 \label{eq:robust-pre-response-linear-bound}
\end{equation}
\begingroup\makeatletter\edef\@currentlabel{\theequation}\edef\@currentHref{equation.\theHequation}\label{eq:app-robust-pre-response-one-sided}\makeatother\endgroup
Maximize over $z$ and use Proposition~\ref{prop:global-complement-certificate}.
\end{proof}
The phase theorem identifies when this gap is positive; the displayed
quantity is a lower bound on its magnitude, not an equality formula.
Generic computable lower certificates are given in Section~\ref{app:computation}.

\section{Quantitative contact geometry}
\label{sec:quantitative-geometry}\label{app:quantitative-contact-stability}
The exact classification is supplemented by three kinds of control:
response-uniform query margins, global angular concentration, and local
stability around a contact class. A separate energy-fraction bound describes
the obstruction created by several residual directions.
\subsection{Two response-uniform query margins}
\label{sec:response-uniform-certificates}\label{app:query-residual-margin}
Let $N$ be the complement basis and put $W_Q=\tr\Xi>0$.
On Theorem~\ref{prop:query-contact-space}'s structural domain, for
nonzero $\Xi\succeq0$ define
\begin{equation}
 t_Q=\min_{z\in\mathcal Z_A}\frac{\tr(\Xi R_A(z))}{\tr\Xi},\qquad
 \Bres(s;t)=1-\frac{\beta_s}{1-\gamma_st}.
 \label{eq:query-residual-margin}
\end{equation}
The feature-only margin is
$\nu_A=\min_{z\in\mathcal Z_A}\lambda_{\min}(R_A(z))$.
Unlike $\nu_A$, $t_Q$ can be positive despite spectral contact
(duplicated basis: $t_Q=1/4$); it is a trace margin, not a Loewner lower bound.

\begin{corollary}[Response-uniform query certificate]
\label{cor:query-phase-certificate}
On this domain, $\nu_A\le t_Q<(m-d-1)/(m-d)$ and
\begin{equation}
 \mathsf T_Q^{\rm RU}(A,s):=\max_{z\in\mathcal Z_A}
 \frac{\tr(\Xi\overline M_s^{A,z})}{L^*\tr\Xi}
 \le\Bres(s;t_Q),\qquad
 t_Q=0\Longleftrightarrow\mathsf T_Q^{\rm RU}(A,s)=\alpha_s.
 \label{eq:query-phase-certificate}
\end{equation}
\end{corollary}
\begin{proof}
The trace-minimum definition~\eqref{eq:query-residual-margin} has the
equivalent complement-compression expression
\begin{equation}
 u_i=\frac{a_i^\top\Xi a_i}{W_Q},\qquad
 t_Q=\sum_i c_i u_i-
 \lambda_{\max}\!\left(N^\top\operatorname{diag}(u_i)N\right).
 \label{eq:query-complement-compression}
\end{equation}
Here $u_i\ge0$ and $\sum_i u_i=1$; the normalized query metric assigns
weight $u_i$ to whitened source row $i$. Indeed, the residual-dependent trace is
\[
 \frac{\tr(\Xi R_A(N\xi))}{W_Q}
 =\sum_i c_i u_i-\xi^\top N^\top\operatorname{diag}(u_i)N\xi.
\]
Minimizing over the unit sphere proves
\eqref{eq:query-complement-compression}. In particular,
\[
 t_Q=\min_{z\in\mathcal Z_A}
 \frac{\tr(\Xi R_A(z))}{W_Q}\ge\nu_A\ge0.
\]
The minimum is attained. Averaging over a uniform residual direction gives
\[
 t_Q\le\frac{n-1}{n}\sum_i c_i u_i<\frac{n-1}{n},
\]
because $\sum_i\ell_i u_i>0$. Thus the scalar denominator and the
first-pass parameter $\delta=1-nt_Q/(n-1)$ are positive.

Diagonalize $R=\sum_k r_kv_kv_k^\top$ and use weights
$\omega_k=v_k^\top\Xi v_k/W_Q\ge0$, $\sum_k\omega_k=1$.
Scalar Jensen yields
\begin{equation}
 \frac{\tr[\Xi(I-\gamma_sR)^{-1}]}{W_Q}
 =\sum_k\frac{\omega_k}{1-\gamma_s r_k}
 \ge\frac{1}{1-\gamma_s\tr(\Xi R)/W_Q}.
 \label{eq:query-trace-jensen}
\end{equation}
Apply the one-sided matrix resolvent and minimize its trace margin to obtain
\eqref{eq:query-phase-certificate}. No commutation assumption on $\Xi$ and
$R$ is needed. In particular,
\[
 \alpha_s-\mathsf T_Q^{\rm RU}(A,s)
 \ge\frac{\beta_s\gamma_s t_Q}{1-\gamma_s t_Q}.
\]
The attained minimum has $t_Q=0$ precisely when some residual has zero
query trace slack in $R_A(z)$; Theorem~\ref{prop:query-contact-space}
then gives exact contact with $\alpha_s$. This proves both directions of
the phase equivalence.

\end{proof}
\subsection{Response-uniform LOO margin}

For fixed $A$ and nonzero $\Xi$, define
\[
 \tau_Q(A,\Xi):=\min_{z\in\mathcal Z_A}
 \delta_{\rm LOO}(z,\Xi).
\]
The pointwise identity
\eqref{eq:quantitative-loo-deletion-identity} also gives
$\tau_Q(A,\Xi)=\min_z g_{m-1}^{(1)}(z,\Xi)$.
Compactness gives attainment. With $z=N\xi$,
\[
 \delta_{\rm LOO}(N\xi,\Xi)
 =1-\xi^\top N^\top\operatorname{diag}(u_i/c_i)N\xi.
\]
Therefore
\begin{equation}
 \tau_Q(A,\Xi)=1-\lambda_{\max}\!\left[
 N^\top\operatorname{diag}(u_i/c_i)N\right].
 \label{eq:loo-response-uniform-loo-margin}
\end{equation}
Taking minima over the same unit residual sphere in
\eqref{eq:all-size-loo-query-comparison} gives
\begin{equation}
 b_s\tau_Q(A,\Xi)
 \leq1-\frac{\mathsf T_Q^{\rm RU}(A,s)}{\alpha_s}
 \leq\tau_Q(A,\Xi).
 \label{eq:loo-response-uniform-transfer}
\end{equation}
A top eigenvector in \eqref{eq:loo-response-uniform-loo-margin} supplies
one residual whose deficit is at most $\tau_Q$ at every interior size.
It need not optimize the risk at every size. The minimized margin is an
existential statement; it does not upper-bound the deficit of an arbitrary
observed residual.
\begin{proposition}[Exact leave-one-out query risk]
\label{prop:query-loo-risk}
Let $A\in\R^{m\times d}$ have orthonormal columns and no coloops, let
$m\ge d+2$, $s=m-1$, and let $\Xi\succeq0$ have $W_Q=\tr\Xi>0$.  Then the
response-uniform fixed-query target is
\begin{equation}
 \mathsf T_Q^{\rm RU}(A,m-1)
 =\frac1n\lambda_{\max}\!\left[
 N^\top\operatorname{diag}\!\left(\frac{u_i}{c_i}\right)N\right].
 \label{eq:query-loo-risk}
\end{equation}
If all $c_i=c$, its exact gap is
$\alpha_{m-1}-\mathsf T_Q^{\rm RU}=t_Q/(nc)$.  In particular, for equal
leverage and every compatible positive-loss residual,
$\tr(\overline M_{m-1})/L^*=d/(m-d)^2$.
\end{proposition}

\begin{proof}
The pointwise identity \eqref{eq:quantitative-loo-deletion-identity} gives
$\mathsf T_Q^{\rm RU}(A,m-1)=\alpha_{m-1}(1-\tau_Q(A,\Xi))$;
the independent eigenvalue evaluation
\eqref{eq:loo-response-uniform-loo-margin} yields \eqref{eq:query-loo-risk}.
If $c_i=c$, then
$t_Q=c-\lambda_{\max}(N^\top\operatorname{diag}(u_i)N)$, so direct
substitution gives the stated gap. Under equal leverage,
$\ell_i=d/m$, $c_i=n/m$, and $\sum_i z_i^2=1$; hence
\eqref{eq:equal-leverage-loo-covariance} gives
$\tr(\overline M_{m-1})/L^*=(nc_i)^{-1}\sum_i z_i^2\ell_i=d/n^2$.
\end{proof}

To read the bound geometrically, let $N$ have orthonormal columns spanning
$\ker A^\top$, write $z=N\xi$ with $\|\xi\|_2=1$, and let
$\widehat n_i$ be its unit complement row. With
$u_i=a_i^\top\Xi a_i/\tr\Xi$, we have $u_i\geq0$, $\sum_i u_i=1$, and
\begin{equation}
 \delta_{\rm LOO}(z,\Xi)=\sum_i u_i
 \bigl(1-|\widehat n_i^\top\xi|^2\bigr).
 \label{eq:loo-weighted-angular-defect}
\end{equation}
The query weights each row's angular error. Zero weighted error gives
whole-query contact. Small weighted error keeps the covariance trace near
its ceiling at every interior size. The converse loses a factor $1/b_s$,
which can be large.
\subsection{Angular concentration and approximate deletion spaces}

Choose an orthonormal complement $N\in\R^{m\times n}$ with rows
$n_i^\top$, and set $\widehat n_i=n_i/\sqrt{c_i}$. Write
$z=N\xi$, $\|\xi\|_2=1$. For a nonzero query metric define
\[
 u_i:=\frac{a_i^\top\Xi a_i}{\tr\Xi},\qquad
 e_i=1-|\widehat n_i^\top\xi|^2.
\]
Then $\sum_i u_i=1$ and
$\delta_{\rm LOO}=\sum_i u_i e_i$, as in
\eqref{eq:loo-weighted-angular-defect}. For $0<\eta<1$, let
$C_\eta=\{i:e_i\leq\eta\}$ and
$E_\eta=\sum_{i\notin C_\eta}u_i$. Since
$\eta E_\eta\leq\delta_{\rm LOO}
\leq\eta+(1-\eta)E_\eta$, Theorem~\ref{thm:loo-all-size-comparison}
gives
\begin{equation}
 b_s\eta E_\eta\leq g_s^{(1)}(z,\Xi)
 \leq\eta+(1-\eta)E_\eta.
 \label{eq:loo-cap-mass-comparison}
\end{equation}
Thus a small deficit forces query row weight into a narrow angular cap,
with loss $1/b_s$ in the reverse direction. Conversely, concentrated
weight in a narrow cap gives a small deficit at every interior size.

The cap also has a low-dimensional approximate deletion space. The next
result records the precise range where this follows from spectral
separation. We use $r_C$ for the number of cap rows, not query rank.

\begin{proposition}[Approximate deletion space of a narrow cap]
\label{prop:loo-approximate-deletion-space}
Fix a nonzero $\Xi\succeq0$ and $0\leq\eta<1/2$. Suppose
$C=C_\eta$ is nonempty. Set
\[
 r_C=|C|,\qquad \rho=\|z_C\|_2^2,\qquad
 \tau=\sum_{i\in C}(c_i-z_i^2),\qquad
 T_C=A_{-C}^\top A_{-C}.
\]
Then $\rho>0$, $r_C\leq d+1$, and
$0\leq\tau\leq\eta\rho/(1-\eta)<\rho$.
The bottom $r_C-1$ eigenvalues of $T_C$ are at most $\tau$.
If $r_C\leq d$, every remaining eigenvalue is at least $\rho$.
Let $V_C^{\rm app}$ be the bottom $r_C-1$ eigenspace, with
$V_C^{\rm app}=\{0\}$ when $r_C=1$ and
$V_C^{\rm app}=\R^d$ when $r_C=d+1$. Then
\begin{equation}
 \frac{\tr[\Xi(I_d-\Pi_{V_C^{\rm app}})]}{\tr\Xi}
 \leq\frac{E_\eta}{\rho}.
 \label{eq:loo-approximate-deletion-space}
\end{equation}
For a rank-$r$ query, let $\lambda_r$ be the least positive eigenvalue of
$\Xi/\tr\Xi$. If $r_C\leq r$, then
\begin{equation}
 E_\eta\geq\rho(r-r_C+1)\lambda_r.
 \label{eq:loo-cap-rank-obstruction}
\end{equation}
\end{proposition}

\begin{proof}
The matrix $J:=I_m-AA^\top-zz^\top$ is an orthogonal projector. Its
$C\times C$ principal block has trace $\tau$, since
$J_{ii}=c_i-z_i^2$. On $C$,
$z_i^2=(1-e_i)c_i\geq(1-\eta)c_i>0$, so $\rho>0$ and
$\tau=\sum_{i\in C}e_ic_i\leq\eta\rho/(1-\eta)<\rho\leq1$.
If $r_C>d+1$, a unit left-null vector of $[A\ z]_C$ is an eigenvector of
$J_{C,C}$ with eigenvalue one. This contradicts $\tr J_{C,C}=\tau<1$.

Because $A_CA_C^\top=I_C-z_Cz_C^\top-J_{C,C}$ and
$0\preceq J_{C,C}\preceq\tau I_C$, its Rayleigh quotient on
$z_C^\perp$ is at least $1-\tau$. Min--max gives $r_C-1$
eigenvalues of $A_CA_C^\top$ at least $1-\tau$. On the other hand,
$A_CA_C^\top\preceq I_C-z_Cz_C^\top$, whose last eigenvalue is
$1-\rho$. Transfer the nonzero eigenvalues to $A_C^\top A_C$ and use
$T_C=I_d-A_C^\top A_C$. The stated two blocks follow. The conventions
for $r_C=1$ and $r_C=d+1$ avoid nonexistent eigenvalue indices.

In particular, $T_C\succeq\rho(I_d-\Pi_{V_C^{\rm app}})$ when the
transverse space is nonzero. Since
$\tr[(\Xi/\tr\Xi)T_C]=\sum_{i\notin C}u_i=E_\eta$,
this proves \eqref{eq:loo-approximate-deletion-space}; it is trivial when
$V_C^{\rm app}=\R^d$. If $\Xi/\tr\Xi$ has rank $r$, it dominates
$\lambda_r$ times the projector onto its range. The trace of that
projector outside an $(r_C-1)$-dimensional space is at least
$r-r_C+1$. Combining this with the previous inequality proves
\eqref{eq:loo-cap-rank-obstruction}.
\end{proof}

At $\eta=0$, the deletion space is exactly $\ker A_{-C}$, as in the
contact classification. The approximate result requires $\eta<1/2$ to
separate the two spectral blocks. It does not assert that the design is
close to a design with exact contact.

\subsection{Local geometry: definitions}
Fix $A,s$ on the structural domain of
Theorem~\ref{prop:query-contact-space}, with its ordinary unrescaled indexed
volume law and selected unweighted OLS, and fix $L^*>0$. Retain
$c_i>0$, the complement basis $N$, its normalized rows $\widehat n_i$,
and the residual sphere
$\mathcal Z_A=\{N\xi:\|\xi\|_2=1\}$ from
Section~\ref{subsec:shared-contact-classes}. Retain $R_A(z)$ and the
interior-size factors $\alpha_s,\beta_s,\gamma_s$ from
\eqref{eq:app-robust-pre-response-psd} and
\eqref{eq:app-robust-pre-response-one-sided}. Put
$c_{\min}=\min_i c_i$, $c_{\max}=\max_i c_i$, and
$\vartheta_s=\beta_s\gamma_s/\alpha_s=(s-d)/(m-d-1)>0$.

For $z=N\xi\in\mathcal Z_A$, define the covariance deficit
\begin{equation}
 D_s^{\rm cov}(z):=\alpha_sI_d-\frac{\overline M_s^{A,z}}{L^*}.
 \label{eq:quantitative-covariance-deficit}
\end{equation}
Thus $D_s^{\rm cov}/\alpha_s$ is the normalized matrix ceiling deficit.
For a nonzero PSD query metric write $W_Q=\tr\Xi>0$.

The envelope, one-sided residual contraction, and interior contact kernel
give
\begin{equation}
 0\preceq\frac{\overline M_s^{A,z}}{L^*}\preceq\alpha_sI_d,
 \qquad
 D_s^{\rm cov}(z)\succeq\beta_s\gamma_sR_A(z),
 \qquad
 \ker D_s^{\rm cov}(z)=\ker R_A(z).
 \label{eq:quantitative-three-baseline-facts}
\end{equation}
These are Theorems~\ref{thm:main-envelope} and
\ref{thm:residual-mechanism} and
Theorem~\ref{prop:query-contact-space}. The complement coordinates give
\begin{equation}
 R_A(N\xi)=\sum_{i=1}^m c_i
 \bigl(1-|\widehat n_i^\top\xi|^2\bigr)a_i a_i^\top.
 \label{eq:quantitative-complement-residual}
\end{equation}

\subsection{A fixed ordinary-law quadratic form}

\begin{lemma}[Ordinary-law residual quadratic form]
\label{lem:ordinary-law-residual-quadratic}
Fix $A,s$ and a nonzero $\Xi\succeq0$. There is a matrix
$\mathcal K_{s,\Xi}\in\R^{(m-d)\times(m-d)}$ such that
\begin{equation}
 0\preceq\mathcal K_{s,\Xi}\preceq I_{m-d},
 \qquad
 g_s^{(1)}(N\xi,\Xi)=\xi^\top\mathcal K_{s,\Xi}\xi
 \quad(\|\xi\|_2=1).
 \label{eq:ordinary-law-residual-quadratic}
\end{equation}
The sampling support in this representation is the fixed feature-only
ordinary-volume support.
\end{lemma}

\begin{proof}
For an ordinary-supported $s$-set $S$, put
\[
 K_S=A_S^\top A_S,\qquad
 T_S=K_S^{-1}A_S^\top N_S,\qquad
 p_A(S)=\frac{\det K_S}{\binom{m-d}{s-d}}.
\]
For the response $A\theta+\sqrt{L^*}N\xi$, selected least squares gives
$(w_S-\theta)/\sqrt{L^*}=T_S\xi$.  Hence the finite supported sum under
the ordinary feature-only law gives
\[
 \frac{\tr(\Xi\overline M_s^{A,N\xi})}{L^*W_Q}
 =\xi^\top H_{s,\Xi}\xi,\qquad
 H_{s,\Xi}:=\frac1{W_Q}
 \sum_{S:\det K_S>0}p_A(S)T_S^\top\Xi T_S\succeq0.
\]
The universal covariance ceiling, applied to every unit $\xi$, implies
$H_{s,\Xi}\preceq\alpha_sI_{m-d}$. Therefore
\[
 \mathcal K_{s,\Xi}:=I_{m-d}-H_{s,\Xi}/\alpha_s
\]
is a positive contraction, and substitution in
\eqref{eq:quantitative-single-deficit} proves
\eqref{eq:ordinary-law-residual-quadratic}. This argument never invokes the
response-dependent augmented law or an inverse outside ordinary support.
\end{proof}

\subsection{Joint residual--query stability}
\label{app:quantitative-joint-alignment}

Fix a maximal projective class $C$ of the normalized complement rows. Choose
a unit representative $\widehat n_C$ and define
\[
 z_C:=N\widehat n_C,\qquad V_C:=\ker A_{-C},\qquad
 \Pi_C:=\operatorname{proj}_{V_C}.
\]
Assume throughout this subsection that
$\{0\}\ne V_C\ne\R^d$. In particular, $-C$ is nonempty. Define
\begin{align}
 \delta_C&:=\min_{i\notin C}
 \bigl(1-|\widehat n_i^\top\widehat n_C|^2\bigr),\notag\\
 \sigma_C^2&:=\min\bigl\{v^\top A_{-C}^\top A_{-C}v:
 v\in V_C^\perp,\ \|v\|_2=1\bigr\},\notag\\
 \kappa_C&:=\frac{c_{\min}\delta_C\sigma_C^2}{4}.
 \label{eq:local-alignment-errors}
\end{align}
For $z=N\xi\in\mathcal Z_A$, set
\begin{equation}
 \varepsilon_z:=1-(z^\top z_C)^2
 =1-(\xi^\top\widehat n_C)^2,
 \qquad
 \varepsilon_\Xi:=
 \frac{\tr[\Xi(I_d-\Pi_C)]}{W_Q}.
 \label{eq:quantitative-two-errors}
\end{equation}
Changing the signs of $\widehat n_C,z_C$ leaves $\varepsilon_z$,
$R_A(z_C)$, and the covariance unchanged. In
\eqref{eq:local-alignment-errors}, $\sigma_C^2$ is the transverse minimum
over $V_C^\perp$. The following fixed-design, proper-class result holds
with constants independent of subset size.
\begin{theorem}[Joint local residual--query alignment]
\label{thm:joint-local-residual-query-alignment}
\label{thm:local-query-stability}
For one fixed $A,C,z_C,z$, and $\Xi$, simultaneously for every
$s\in\{d+1,\ldots,m-1\}$, if
\[
 \varepsilon_z\leq\frac{\delta_C}{16},\qquad
 \varepsilon_\Xi\leq\frac12,
\]
then
\begin{equation}
 \vartheta_s\kappa_C(\varepsilon_z+\varepsilon_\Xi)
 \leq g_s^{(1)}(z,\Xi)
 \leq\min\left\{1,
 (\sqrt{\varepsilon_z}+\sqrt{\varepsilon_\Xi})^2\right\}.
 \label{eq:local-query-stability}
\end{equation}
The upper bound holds for every $z\in\mathcal Z_A$ and every nonzero
$\Xi\succeq0$, without either neighborhood restriction, on the same fixed
design and proper-class domain.
\end{theorem}

\begin{proof}
The constants are positive. No coloops give $c_{\min}>0$; maximality and
finiteness give $\delta_C>0$; and
$\ker(A_{-C}^\top A_{-C})=V_C$ gives $\sigma_C^2>0$ on $V_C^\perp$.
Moreover, $c_{\min},\delta_C,\sigma_C^2\leq1$.

Orient $\widehat n_C$ so that $\xi^\top\widehat n_C\geq0$ and write
\[
 \xi=\sqrt{1-\varepsilon_z}\,\widehat n_C
 +\sqrt{\varepsilon_z}\,\eta,
 \qquad \eta\perp\widehat n_C,\quad\|\eta\|_2=1,
\]
with the evident interpretation at $\varepsilon_z=0$. If $i\in C$, then
$\widehat n_i=\pm\widehat n_C$ and
\begin{equation}
 1-|\widehat n_i^\top\xi|^2=\varepsilon_z.
 \label{eq:quantitative-inside-angle}
\end{equation}
If $i\notin C$, let $P_i=I-\widehat n_i\widehat n_i^\top$. The reverse
triangle inequality gives
\begin{align*}
 \sqrt{1-|\widehat n_i^\top\xi|^2}
 &=\|P_i\xi\|_2\\
 &\geq\sqrt{1-\varepsilon_z}\,\|P_i\widehat n_C\|_2
       -\sqrt{\varepsilon_z}\,\|P_i\eta\|_2\\
 &\geq\sqrt{1-\varepsilon_z}\sqrt{\delta_C}
       -\sqrt{\varepsilon_z}
 \geq\frac{\sqrt{\delta_C}}2.
\end{align*}
Indeed, $\varepsilon_z\leq\delta_C/16\leq1/16$ implies
$\sqrt{1-\varepsilon_z}\geq3/4$ and
$\sqrt{\varepsilon_z}\leq\sqrt{\delta_C}/4$. Thus
\eqref{eq:quantitative-complement-residual} yields
\begin{equation}
 R_A(z)\succeq c_{\min}\left[
 \varepsilon_z A_C^\top A_C+
 \frac{\delta_C}{4}A_{-C}^\top A_{-C}\right].
 \label{eq:quantitative-residual-lower}
\end{equation}

Let $H_C=A_{-C}^\top A_{-C}$. Since $0\preceq H_C\preceq I$,
$\ker H_C=V_C$, and $A_C^\top A_C=I-H_C$, the invariant orthogonal
decomposition $V_C\oplus V_C^\perp$ gives
\[
 A_C^\top A_C\succeq\Pi_C,\qquad
 A_{-C}^\top A_{-C}\succeq\sigma_C^2(I_d-\Pi_C).
\]
Taking the $\Xi$-trace of \eqref{eq:quantitative-residual-lower}, applying
\eqref{eq:quantitative-three-baseline-facts}, and dividing by
$\alpha_sW_Q$ gives the stronger proof-level estimate
\begin{equation}
 g_s^{(1)}(z,\Xi)\geq\vartheta_sc_{\min}\left[
 \varepsilon_z(1-\varepsilon_\Xi)+
 \frac{\delta_C\sigma_C^2}{4}\varepsilon_\Xi\right].
 \label{eq:quantitative-refined-lower}
\end{equation}
Set $\lambda_C=\delta_C\sigma_C^2/4\leq1/4$. Since
$\varepsilon_\Xi\leq1/2$,
\[
 \varepsilon_z(1-\varepsilon_\Xi)+\lambda_C\varepsilon_\Xi
 -\lambda_C(\varepsilon_z+\varepsilon_\Xi)
 =\varepsilon_z(1-\varepsilon_\Xi-\lambda_C)\geq0.
\]
This proves the frozen lower bound with $\kappa_C=c_{\min}\lambda_C$.

For the upper bound, exact contact at $z_C$ gives
$\ker D_s^{\rm cov}(z_C)=V_C$. Since
$0\preceq D_s^{\rm cov}(z_C)\preceq\alpha_sI_d$, its $V_C$ block and
cross blocks vanish, so
\begin{equation}
 0\preceq D_s^{\rm cov}(z_C)\preceq\alpha_s(I_d-\Pi_C),
 \qquad g_s^{(1)}(z_C,\Xi)\leq\varepsilon_\Xi.
 \label{eq:quantitative-query-anchor}
\end{equation}
Let $\mathcal K=\mathcal K_{s,\Xi}$ be the positive contraction from
Lemma~\ref{lem:ordinary-law-residual-quadratic}. Put
$\rho=\xi^\top\widehat n_C$ and, when $\varepsilon_z>0$,
$\eta=(\xi-\rho\widehat n_C)/\sqrt{\varepsilon_z}$. Then
\begin{align*}
 \sqrt{g_s^{(1)}(z,\Xi)}
 &=\|\mathcal K^{1/2}\xi\|_2\\
 &\leq|\rho|\,\|\mathcal K^{1/2}\widehat n_C\|_2
      +\sqrt{\varepsilon_z}\,\|\mathcal K^{1/2}\eta\|_2\\
 &\leq\sqrt{g_s^{(1)}(z_C,\Xi)}+\sqrt{\varepsilon_z}
 \leq\sqrt{\varepsilon_\Xi}+\sqrt{\varepsilon_z}.
\end{align*}
The case $\varepsilon_z=0$ is immediate. Squaring proves the second upper
endpoint, while $0\preceq\mathcal K\preceq I$ gives $g_s^{(1)}\leq1$.
Neither step used a neighborhood restriction. The geometric constants are
independent of $s$, proving the simultaneous size quantifier.
\end{proof}

\begin{corollary}[Query-only slice]
\label{cor:quantitative-query-only-slice}
At $z=z_C$, let
\[
 r_C:=\min\{v^\top R_A(z_C)v:
 v\in V_C^\perp,\ \|v\|_2=1\}>0.
\]
Then, for every nonzero $\Xi\succeq0$ and every strict-interior size,
\begin{equation}
 \vartheta_sr_C\varepsilon_\Xi
 \leq g_s^{(1)}(z_C,\Xi)\leq\varepsilon_\Xi.
 \label{eq:quantitative-query-only-slice}
\end{equation}
\end{corollary}

\begin{proof}
The contact-kernel identity gives
$\ker R_A(z_C)=\ker D_s^{\rm cov}(z_C)=V_C$. Hence
$R_A(z_C)\succeq r_C(I_d-\Pi_C)$, while the PSD-kernel argument in
\eqref{eq:quantitative-query-anchor} gives
$D_s^{\rm cov}(z_C)\preceq\alpha_s(I_d-\Pi_C)$. Combine these relations
with $D_s^{\rm cov}(z_C)\succeq\beta_s\gamma_sR_A(z_C)$, take the
$\Xi$-trace, and divide by $\alpha_sW_Q$. The same comparison also gives
$r_C\geq c_{\min}\delta_C\sigma_C^2=4\kappa_C$.
\end{proof}

\begin{corollary}[Residual-only slice]
\label{cor:quantitative-residual-only-slice}
Suppose $\operatorname{range}\Xi\subseteq V_C$, and define
\[
 u_i:=\frac{a_i^\top\Xi a_i}{W_Q},\qquad
 \overline c_{\Xi,C}:=\sum_{i\in C}c_iu_i.
\]
Then $\sum_{i\in C}u_i=1$, $\overline c_{\Xi,C}>0$, and for every
$z\in\mathcal Z_A$ and every strict-interior size,
\begin{equation}
 \vartheta_s\overline c_{\Xi,C}\varepsilon_z
 \leq g_s^{(1)}(z,\Xi)\leq\varepsilon_z.
 \label{eq:quantitative-residual-only-slice}
\end{equation}
No neighborhood restriction is required.
\end{corollary}

\begin{proof}
For $i\notin C$, every vector in $V_C$ is orthogonal to $a_i$, so PSD range
containment gives $u_i=0$. Since $A^\top A=I_d$, $\sum_i u_i=1$ and thus
$\sum_{i\in C}u_i=1$. For $i\in C$,
$\widehat n_i=\pm\widehat n_C$, and
\eqref{eq:quantitative-complement-residual} gives
\[
 \frac{\tr[\Xi R_A(z)]}{W_Q}
 =\overline c_{\Xi,C}\varepsilon_z.
\]
The one-sided contraction proves the lower bound. For the upper bound,
range containment and exact contact give
$\widehat n_C^\top\mathcal K_{s,\Xi}\widehat n_C
=g_s^{(1)}(z_C,\Xi)=0$, whence positivity gives
$\mathcal K_{s,\Xi}\widehat n_C=0$. Writing
$\xi=\rho\widehat n_C+\sqrt{\varepsilon_z}\eta$ therefore yields
\[
 g_s^{(1)}(z,\Xi)
 =\varepsilon_z\eta^\top\mathcal K_{s,\Xi}\eta
 \leq\varepsilon_z.
\]
\end{proof}

\subsection{Exact unequal-leverage leave-one-out consequences}
\label{app:quantitative-loo-consequences}

Set $n=m-d$ and $s=m-1$, so $\alpha_{m-1}=1/n$.

\begin{proposition}[Pointwise and minimized LOO inequalities and angle law]
\label{prop:quantitative-unequal-leverage-loo}
For every $z\in\mathcal Z_A$,
\begin{equation}
 \frac{R_A(z)}{c_{\max}}
 \preceq nD_{m-1}^{\rm cov}(z)
 \preceq\frac{R_A(z)}{c_{\min}}.
 \label{eq:quantitative-loo-matrix-sandwich}
\end{equation}
The matrix sandwich also yields trace and smallest-eigenvalue inequalities
pointwise in $z$. For a fixed nonzero $\Xi\succeq0$, recall the margins
$t_Q$ from \eqref{eq:query-residual-margin} and
$\nu_A=\min_{z\in\mathcal Z_A}\lambda_{\min}[R_A(z)]$.
Use the response-uniform $\tau_Q(A,\Xi)$ from
\eqref{eq:loo-response-uniform-loo-margin}, equivalently
$\min_{z\in\mathcal Z_A}g_{m-1}^{(1)}(z,\Xi)$, and define
$\tau_{\rm spec}(A):=\min_{z\in\mathcal Z_A}
n\lambda_{\min}[D_{m-1}^{\rm cov}(z)]$.
The minima are attained and satisfy
\begin{equation}
 \frac{t_Q}{c_{\max}}
 \leq\tau_Q(A,\Xi)\leq\frac{t_Q}{c_{\min}},
 \qquad
 \frac{\nu_A}{c_{\max}}
 \leq\tau_{\rm spec}(A)\leq\frac{\nu_A}{c_{\min}}.
 \label{eq:quantitative-loo-minimized}
\end{equation}
Equivalently,
\[
 \tau_Q(A,\Xi)=1-
 \frac{\mathsf T_Q^{\rm RU}(A,m-1)}{\alpha_{m-1}},\qquad
 \tau_{\rm spec}(A)=1-
 \frac{\mathsf T_{\rm spec}^{\rm RU}(A,m-1)}{\alpha_{m-1}}.
\]
Finally, let $C$ be a maximal projective class of normalized complement rows,
with unit representative $\widehat n_C$, anchor $z_C=N\widehat n_C$, and
$V_C=\ker A_{-C}$. For any nonzero $\Xi\succeq0$ with
$\operatorname{range}\Xi\subseteq V_C$, put
$\varepsilon_z=1-(z^\top z_C)^2$. Then
\begin{equation}
 g_{m-1}^{(1)}(z,\Xi)=1-(z^\top z_C)^2=\varepsilon_z
 \qquad(z\in\mathcal Z_A).
 \label{eq:quantitative-loo-angle-law}
\end{equation}
No equal-leverage assumption is used.
\end{proposition}

\begin{proof}
By the pointwise identity \eqref{eq:quantitative-loo-deletion-identity},
$nD_{m-1}^{\rm cov}(z)=\sum_i(c_i-z_i^2)a_ia_i^\top/c_i$, whereas
$R_A(z)=\sum_i(c_i-z_i^2)a_ia_i^\top$. Since
$z_i^2\leq\|n_i\|_2^2=c_i$, comparison of the positive summands proves
\eqref{eq:quantitative-loo-matrix-sandwich}. Taking PSD traces and smallest
eigenvalues for each $z$, then minima over the compact residual sphere,
proves \eqref{eq:quantitative-loo-minimized}; continuity gives attainment.

For the angle law, set $u_i=a_i^\top\Xi a_i/W_Q$. Range containment gives
$u_i=0$ off $C$ and $\sum_{i\in C}u_i=1$. On every row of the maximal class,
$z_i^2/c_i=|\widehat n_C^\top\xi|^2=(z^\top z_C)^2$.
The $\Xi$-trace of the first identity in
\eqref{eq:quantitative-loo-deletion-identity}, with
$\alpha_{m-1}=1/n$, proves \eqref{eq:quantitative-loo-angle-law}.
\end{proof}

\subsection{Multi-output residual energy fraction}
\label{app:quantitative-energy-fraction}\label{sec:aggregate-energy-orientation}

Several residual directions force a positive aggregate gap, with an
explicit lower bound determined by their energy distribution.
\begin{proposition}[Residual energy-fraction gap]
\label{prop:multioutput-energy-fraction-gap}
Let $d\geq1$, $m\geq d+2$, and let $A\in\R^{m\times d}$ be real with
$A^\top A=I_d$ and $c_i=1-\|a_i\|_2^2>0$ for every row. Use the same
ordinary unrescaled indexed size-$s$ volume law and selected unweighted OLS
at an integer $d<s<m$. Let
\begin{align*}
 E&=[e_1,\ldots,e_K]\ne0,& A^\top E&=0,\\
 L_j&:=\|e_j\|_2^2,&
 L_{\rm src}&:=\sum_jL_j=\|E\|_F^2>0.
\end{align*}
Zero residual columns have zero covariance and are omitted. For each retained
column, let $\overline M_{s,j}$ be its full-Gram-whitened
selected-coefficient covariance. Define
\[
 \zeta_E:=1-\frac{\|E\|_2^2}{\|E\|_F^2},\qquad
 B_A:=A^\top\operatorname{diag}(c_1,\ldots,c_m)A.
\]
Then
\begin{equation}
 I_d-\frac{\sum_j\overline M_{s,j}}{\alpha_sL_{\rm src}}
 \succeq\vartheta_s\zeta_EB_A
 \succeq\vartheta_s\zeta_Ec_{\min}I_d.
 \label{eq:multioutput-energy-fraction-gap}
\end{equation}
Here $\alpha_s=(m-s)/(m-d)>0$ and
$\vartheta_s=(s-d)/(m-d-1)>0$ are repeated locally because the displayed
normalization is invalid at an endpoint. Moreover, $\zeta_E>0$ exactly when
$\operatorname{rank}(E)\geq2$.
\end{proposition}

\begin{proof}
For each $L_j>0$, put $z_j=e_j/\sqrt{L_j}$. Applying the boundary-valid
contraction output by output and summing gives
\[
 \sum_j\overline M_{s,j}
 \preceq\alpha_sL_{\rm src}I_d
 -\beta_s\gamma_s\sum_{j:L_j>0}L_jR_A(z_j).
\]
Define the pooled residual operator
\[
 \overline R_E:=\frac1{L_{\rm src}}
 \sum_{j:L_j>0}L_jR_A(z_j).
\]
Expanding row by row gives the exact identity
\begin{equation}
 \overline R_E=A^\top\operatorname{diag}\left(
 c_i-\frac{(EE^\top)_{ii}}{L_{\rm src}}
 \right)A.
 \label{eq:multioutput-average-residual-operator}
\end{equation}
Let $P_\perp=I_m-AA^\top$. Since $A^\top E=0$, $P_\perp E=E$. For every
$x\in\R^m$,
\[
 x^\top EE^\top x=\|E^\top P_\perp x\|_2^2
 \leq\|E\|_2^2\|P_\perp x\|_2^2,
\]
so
\[
 \frac{EE^\top}{L_{\rm src}}
 \preceq(1-\zeta_E)P_\perp.
\]
Taking diagonal entries and using $(P_\perp)_{ii}=c_i$ gives
$(EE^\top)_{ii}/L_{\rm src}\leq(1-\zeta_E)c_i$. Substitution into
\eqref{eq:multioutput-average-residual-operator} yields
$\overline R_E\succeq\zeta_EB_A$. Also
\[
 B_A=\sum_i c_ia_ia_i^\top
 \succeq c_{\min}\sum_i a_ia_i^\top=c_{\min}I_d.
\]
Divide the summed contraction by $\alpha_sL_{\rm src}>0$ and use
$\beta_s\gamma_s/\alpha_s=\vartheta_s$ to prove
\eqref{eq:multioutput-energy-fraction-gap}. Finally, for nonzero $E$,
equality $\|E\|_2^2=\|E\|_F^2$ holds exactly at rank one.
\end{proof}

For every nonzero query metric $\Xi\succeq0$, define the aggregate deficit
separately from \eqref{eq:quantitative-single-deficit} by
\begin{equation}
 g_s^{\rm agg}(E,\Xi):=
 1-\frac{\tr(\Xi\sum_j\overline M_{s,j})}
 {\alpha_sL_{\rm src}W_Q}.
 \label{eq:multioutput-aggregate-deficit}
\end{equation}
PSD trace monotonicity gives
\begin{equation}
 g_s^{\rm agg}(E,\Xi)
 \geq\vartheta_s\zeta_E
 \sum_i c_i\frac{a_i^\top\Xi a_i}{W_Q}.
 \label{eq:multioutput-query-trace-gap}
\end{equation}

\begin{corollary}[Multi-output residual-rank strictness]
\label{cor:multioutput-residual-rank-strictness}
Let $E=[e_1,\ldots,e_K]$ be the exact Full residual matrix, let
$L_j=\|e_j\|_2^2$, and let $L_{\rm src}=\sum_jL_j$. On the real Parseval,
no-coloop structural domain with $m\ge d+2$ and $d<s<m$,
\begin{equation}
 \operatorname{rank}(E)\ge2
 \quad\Longrightarrow\quad
 \sum_{j=1}^K\overline M_{s,j}\prec\alpha_sL_{\rm src}I_d.
 \label{eq:multioutput-residual-rank-strictness}
\end{equation}
\end{corollary}

\begin{proof}
For $E\ne0$, Proposition~\ref{prop:multioutput-energy-fraction-gap} gives
\eqref{eq:multioutput-energy-fraction-gap}. If $\operatorname{rank}(E)\ge2$,
then $\zeta_E>0$; also $\vartheta_s,c_{\min}>0$ on the stated domain.
The right-hand side of that bound is positive definite, proving
\eqref{eq:multioutput-residual-rank-strictness}. Its trace against every
nonzero $\Xi\succeq0$ is positive. The converse fails: a rank-one residual
may saturate a nontrivial complement class, while one with no saturated class has
strict covariance deficit by the contact-kernel identity.
\end{proof}

\subsection{Same spectrum, different query risk}
\label{app:alignment-same-spectrum}
\label{app:alignment-literal-designs}
\label{app:alignment-local-records}

The following six-row design has a three-row contact class:
\[
 A=\begin{pmatrix}
 2/3&2/3&0\\1/3&-2/3&0\\-2/3&1/3&0\\
 0&0&2/3\\0&0&-2/3\\0&0&1/3
 \end{pmatrix},\qquad
 N=\begin{pmatrix}
 1/3&0&0\\2/3&0&0\\2/3&0&0\\
 0&1/3&2/3\\0&2/3&1/3\\0&2/3&-2/3
 \end{pmatrix}.
\]
Here $A^\top A=N^\top N=I_3$, $A^\top N=0$, $C=\{1,2,3\}$, and
$V_C=\operatorname{span}(e_1,e_2)$. Let $X=A$, $z_C=N_{:,1}$,
$w=N_{:,2}$, and $\Xi=(e_1+e_2)(e_1+e_2)^\top/2$. Compatible residuals make
the Full coefficient zero. For $z_t=(z_C+tw)/\sqrt{1+t^2}$, the LOO states
$t=0,1/4,1$ have
$(\varepsilon_z,g_5^{(1)})=(0,0),(1/17,1/17),(1/2,1/2)$. This realizes
\eqref{eq:quantitative-loo-angle-law}, which requires neither equal leverage nor the
joint-local neighborhood.

For the multi-output comparison, fix $s=5$, $K=2$, and
\[
 E_A(a)=[z_C,aw],\qquad E_B(a)=[az_C,w],
 \qquad a\in\{0,1/3,1/2,1\}.
\]
Each pair has the same design, query, law, size, output count, total energy
$1+a^2$, singular spectrum $\{1,a\}$, and $\zeta_E=a^2/(1+a^2)$. At
$a=1/3$, $(1-g_A^{\rm agg},1-g_B^{\rm agg})=(9/10,1/10)$, so spectrum and
energy do not determine fixed-query risk. For $a<1$, the residual matrices
and raw responses differ; only their spectra match. The pooled residual
operator distinguishes the arms, so this does not show that second moments
are necessary or that a learning algorithm improves.

\section{Exact families and limits of geometric control}
\label{sec:exact-families}\label{app:critical-geometry}
These constructions are part of the theory. They distinguish phase from
gap magnitude, spectral contact from trace contact, and continuous feature
geometry from covariance jumps when sampling support changes.
\subsection{Critical equal leverage}
\label{sec:critical-geometry}

The global complement phase test has an explicit geometric interpretation at
critical redundancy.  Suppose $m=2d$ and every whitened row has leverage
$\ell_i=1/2$.  With $P=AA^\top$ and $P_\perp=I_{2d}-P$, each compatible
residual gives
\begin{equation}
  Q_z=P_\perp-zz^\top,\qquad
  R_A(z)=A^\top\operatorname{Diag}(\operatorname{diag}Q_z)A,
  \label{eq:critical-fixed-complement}
\end{equation}
with residual-coupled weights
\begin{equation}
  (\operatorname{diag}Q_z)_i=\frac12-z_i^2.
  \label{eq:critical-fractional-weights}
\end{equation}
These are neither freely chosen frame weights nor a binary erasure mask.
Define the normalized coherence
\begin{equation}
  \rho_A=\max_{i\ne j}2\lvert a_i^\top a_j\rvert.
  \label{eq:critical-coherence}
\end{equation}
Since $(P_\perp)_{ij}=-a_i^\top a_j$ off the diagonal and
$(P_\perp)_{ii}=1/2$, the global quantities in
Proposition~\ref{prop:global-complement-certificate} reduce to
\begin{equation}
  \rho_\perp(A)=\rho_A,\qquad
  \kappa_\perp(A)=\frac14,\qquad
  t_\perp(A)=\frac{1-\rho_A}{4}.
  \label{eq:critical-global-reduction}
\end{equation}
The theorem below retains what is special to the critical class: an explicit
original-row zero locus and a complementary upper bound on $\nu_A$.

\begin{theorem}[Critical residual geometry]
\label{thm:critical-geometry}
Let $d\ge2$ and let $A\in\R^{2d\times d}$ have orthonormal columns and
$\lVert a_i\rVert_2^2=1/2$ for every $i$.  Then
\begin{equation}
  \boxed{
  \frac{1-\rho_A}{4}
  \le \nu_A
  \le \frac{(1-\rho_A)(3+\rho_A)}{8}.}
  \label{eq:critical-coherence-sandwich}
\end{equation}
Moreover, $\nu_A=0$ if and only if there are distinct indices $i,j$, a unit
vector $v\in\R^d$, and
signs $\sigma_i,\sigma_j\in\{\pm1\}$ such that
\begin{equation}
  a_i=\frac{\sigma_i}{\sqrt2}v,\qquad
  a_j=\frac{\sigma_j}{\sqrt2}v,\qquad
  a_k^\top v=0\quad(k\notin\{i,j\}).
  \label{eq:critical-zero-locus}
\end{equation}
Hence the zero boundary is a repeated projective pair with an orthogonal
remainder, while~\eqref{eq:critical-coherence-sandwich} controls the margin
away from that boundary.
\end{theorem}

This specialization makes the global sign test geometrically visible:
projectively separated rows have positive uniform slack, whereas the
repeated-pair boundary admits the common spectral-contact residual from
Theorem~\ref{thm:robust-pre-response-phase}.  Its upper bound supplies
additional information not provided by the global lower certificate.
The repeated-pair form makes the critical-class interpretation explicit.

\begin{proof}
Since $\rho_\perp=\rho_A$ and $\kappa_\perp=1/4$, the lower bound and
zero phase follow from Proposition~\ref{prop:global-complement-certificate}.
If $\rho_A=1$, the corresponding rows are parallel with norm $1/\sqrt2$;
their two squared directional energies already sum to one, so Parsevalness
forces every other row to be orthogonal. Conversely such a pair has
$\rho_A=1$, proving the stated zero locus.
Put $u_i=\sqrt2a_i$ and $x_i=(u_i^\top v)^2$. For a unit residual,
$v^\top R_A(z)v=1/2-(1/2)\sum_i z_i^2x_i$.
For the other direction, choose $i,j$ attaining $\rho_A$ and choose
$\epsilon\in\{\pm1\}$ so that
$u_i^\top(\epsilon u_j)=\rho_A$.  Define
\begin{equation}
  b=\frac{e_i-\epsilon e_j}{\sqrt2},
  \qquad
  h=b^\top P_\perp b=\frac{1+\rho_A}{2},
  \qquad
  z=\frac{P_\perp b}{\sqrt h}.
  \label{eq:app-critical-upper-witness}
\end{equation}
The vector $z$ is a unit residual, and a direct coordinate calculation gives
\begin{equation}
  z_i^2=z_j^2=\frac{1+\rho_A}{4}.
  \label{eq:app-critical-upper-coordinates}
\end{equation}
For
\begin{equation}
  v=\frac{u_i+\epsilon u_j}{\sqrt{2(1+\rho_A)}},
\end{equation}
one has $x_i=x_j=(1+\rho_A)/2$.  Keeping only these two nonnegative terms
in this Rayleigh quotient yields
\begin{equation}
  \nu_A
  \le\frac12-\frac{(1+\rho_A)^2}{8}
  =\frac{(1-\rho_A)(3+\rho_A)}8.
  \label{eq:app-critical-upper}
\end{equation}
Combining the two bounds proves the sandwich.\end{proof}

For the repeated pair, choose $z_i=\sigma_i/\sqrt2$,
$z_j=-\sigma_j/\sqrt2$, and zero elsewhere. Its selected coefficient
error is $\pm\sqrt{L^*}v$ when exactly one pair member is selected and
zero when both are selected. The singleton exclusion identity gives
probability $\alpha_s$ for the former event, with equally weighted signs.
Consequently
\begin{equation}
 \overline M_s^{A,z}=\alpha_sL^*vv^\top\quad(d<s<2d).
 \label{eq:critical-sharp-witness}
\end{equation}
\begingroup\makeatletter\edef\@currentlabel{\theequation}\edef\@currentHref{equation.\theHequation}\label{eq:app-critical-exact-covariance}\makeatother\endgroup
This is directional contact with a rank-one covariance.
\subsection{A matched rotation family}
\label{app:matched-matched-phase-pair}

We instantiate the two branches of Theorem~\ref{thm:robust-pre-response-phase}
on a matched critical pair. The exact calculation verifies ordinary-volume
support, selected-map semantics, and the common-witness covariance exactly.
For $t\in\R$, set
\begin{equation}
\begin{aligned}
 G(t)&=\frac{1}{1+t^2}
 \begin{bmatrix}1-t^2&-2t\\2t&1-t^2\end{bmatrix},\\
 Q_t&=\operatorname{diag}(G(t),G(t)),\\
 A_t&=2^{-1/2}\begin{bmatrix}I_4\\Q_t\end{bmatrix},\qquad
 J_t=2^{-1/2}\begin{bmatrix}I_4\\-Q_t\end{bmatrix}.
\end{aligned}
 \label{eq:app-matched-phase-pair-design}
\end{equation}
Thus $A_t^\top A_t=J_t^\top J_t=I_4$, $A_t^\top J_t=0$, and every row
leverage is $1/2$. Both designs use ordinary size-$s$ volume sampling and
selected unweighted OLS at $s\in\{5,6,7\}$.

At $t=0$, the residual $J_0e_1$ gives
$\overline M_s=\alpha_s e_1e_1^\top$ for $s=5,6,7$.
For $t=1/2$, the coherence is $4/5$, so the design lies in the positive
phase despite having the same dimensions and leverage vector. The exact
margin and covariance make the comparison more precise.
\begin{proposition}[Exact margin of the matched rotation family]
\label{prop:matched-family-margin}
For every real $t$,
\begin{equation}
 \nu_{A_t}=\frac{1-|\cos(4\arctan t)|}{8}
 =\frac18\left(1-\frac{|1-6t^2+t^4|}{(1+t^2)^2}\right).
 \label{eq:matched-family-exact-margin}
\end{equation}
In particular, $\nu_{A_0}=0$ and $\nu_{A_{1/2}}=9/100$.
\end{proposition}
\begin{proof}
Every compatible unit residual is $z=J_tv$ for a unit $v\in\mathbb R^4$.
Writing $D_x=\operatorname{diag}(x_i^2)$ and using the equal row leverage,
\[
 R_{A_t}(J_tv)=\frac12I_4-\frac14
 \bigl(D_v+Q_t^\top D_{Q_tv}Q_t\bigr).
\]
For unit $u,v$, define
$H(u,v)=\sum_i u_i^2v_i^2+\sum_i(Q_tu)_i^2(Q_tv)_i^2$ and
$F(u)=H(u,u)$. Cauchy--Schwarz on the two concatenated squared-coordinate
vectors gives
$H(u,v)\le\sqrt{F(u)F(v)}\le\max_{\|w\|=1}F(w)$, with equality
when $u=v$ maximizes $F$. Hence
$\nu_{A_t}=1/2-(1/4)\max_{\|w\|=1}F(w)$.
The two rotation blocks are identical. Degree-four homogeneity places a
maximizer in one block: if their squared norms are $r_1+r_2=1$, then
$F(w)\le F_{\max}^{(2)}(r_1^2+r_2^2)\le F_{\max}^{(2)}$.
In a unit block write $w=(\cos\phi,\sin\phi)$ and
$\theta=2\arctan t$. Direct expansion gives
$F(w)=3/2+(1/2)\cos(2\theta)\cos(4\phi+2\theta)$.
Its maximum is $3/2+|\cos(2\theta)|/2$, proving the formula.
At $t=1/2$, $w=(3,1,0,0)^\top/\sqrt{10}$ attains the maximum.
\end{proof}

\subsection{Exact covariance targets in the matched rotation family}
\label{app:matched-exact-covariance}

After simultaneous row and column permutation, the family in
Section~\ref{app:matched-matched-phase-pair} is a direct sum of two copies of
$B_t=2^{-1/2}[I_2;G(t)]$. A local complement coordinate $v\in\R^2$
represents the actual residual
$r(v)=2^{-1/2}[v;-G(t)v]\in\R^4$, with $\|r(v)\|^2=\|v\|^2$.
Writing $b_i^\top$ for the local design rows, define
\[
 \mathcal W_t(v):=\sum_{i=1}^4r_i(v)^2b_ib_i^\top
 =\frac14\{\operatorname{diag}(v_i^2)
 +G(t)^\top\operatorname{diag}((G(t)v)_i^2)G(t)\},
 \qquad \tr\mathcal W_t(v)=\frac12\|v\|^2.
\]
For $0<t<1$, every two-row local basis has positive volume. The local
covariances at sizes $2,3,4$ are respectively $\|v\|^2I_2$,
$\mathcal W_t(v)$ (by \eqref{eq:equal-leverage-loo-covariance}), and zero.
The local volume normalizers $1,2,1$ give allocation weights $(1/2,1/2)$
at global size $5$, $(1/6,4/6,1/6)$ at size $6$, and $(1/2,1/2)$ at size $7$.
Conditional on the two sizes, local sampling factorizes and both fits are
unbiased, so cross-block covariances vanish. For residual coordinates
$v^{(j)}$ with energies $E_j=\|v^{(j)}\|^2$, this proves
\begin{equation}
\begin{gathered}
 \overline M_s=\operatorname{diag}_{j=1,2}
 \{a_sE_jI_2+c_s\mathcal W_t(v^{(j)})\},\\
 (a_5,c_5)=(1/2,1/2),\quad (a_6,c_6)=(1/6,2/3),\quad
 (a_7,c_7)=(0,1/2).
\end{gathered}
\label{eq:matched-common-block-covariance}
\end{equation}
Homogeneity places a response-uniform spectral optimizer in one block.
Proposition~\ref{prop:matched-family-margin} gives
\[
\max_{\|v\|=1}\lambda_{\max}(\mathcal W_t(v))
=\frac12-\nu_{A_t}.
\]
Substitution into \eqref{eq:matched-common-block-covariance} yields
\begin{equation}
 \mathsf T_{\rm spec}^{\rm RU}(A_t,s)
 =\alpha_s-2\beta_s\gamma_s\nu_{A_t},
 \qquad s\in\{5,6,7\},\quad 0<t<1.
 \label{eq:matched-exact-all-interior}
\end{equation}
\begingroup
\makeatletter
\edef\@currentlabel{\theequation}
\edef\@currentHref{equation.\theHequation}
\label{eq:matched-rotation-exact-spectral}
\makeatother
\endgroup
At $t=0,1$, $\nu_{A_t}=0$ and the separate common-witness result supplies
the equality in~\eqref{eq:matched-exact-all-interior}.

At $t=1/2$, $\nu_{A_t}=9/100$, so the exact targets are
\begin{equation}
 \mathsf T_{\rm spec}^{\rm RU}(A_{1/2},5)=141/200,\quad
 \mathsf T_{\rm spec}^{\rm RU}(A_{1/2},6)=11/25,\quad
 \mathsf T_{\rm spec}^{\rm RU}(A_{1/2},7)=41/200.
 \label{eq:matched-exact-half-targets}
\end{equation}
Their ratios to $\alpha_s$ are $47/50$, $22/25$, and $41/50$.
One normalized residual attains all three values:
\[
z=J_{1/2}(3,1,0,0)^\top/\sqrt{10}.
\]
The matching covariance eigenvector gives the maximizing query direction.
\par
\subsection{Exact trace and fixed-query loss separation}
\label{app:matched-query-loss}

For $0<t<1$, write $c=(1-t^2)/(1+t^2)$, $h=2t/(1+t^2)$ and
$J_t=2^{-1/2}[I_4;-Q_t]$. Every compatible residual is $J_tv$,
with $L^*=\|v\|^2$. Taking the trace in the common block identity
\eqref{eq:matched-common-block-covariance} and using
$\tr\mathcal W_t(v^{(j)})=E_j/2$ proves, for every nonzero fixed residual,
\begin{equation}
 \frac{\tr\overline M_s}{L^*}
 =\begin{cases}5/4,&s=5,\\2/3,&s=6,\\1/4,&s=7,\end{cases}
 \qquad 0<t<1.
 \label{eq:matched-isotropic-trace}
\end{equation}
At $t=0,1$, the rotation is a signed permutation and the local rank-size
covariance becomes $\operatorname{diag}(v_i^2)$; vanished zero-volume bases
are omitted. Hence $\overline M_s=\alpha_s\operatorname{diag}(v_i^2)$ and
the trace ratio is $3/4,1/2,1/4$ at $s=5,6,7$. The first two second moments
therefore do not continue across the changing support.

For the fully matched comparison, fix $X=A_t$ and the same numeric response
$y=\sqrt2[e_1;0_4]$. Direct multiplication gives $w^*=e_1$ and residual
$J_t e_1$, of energy one, for every $t$. At $s=7$, each omitted row has
probability $(1-1/2)/(8-4)=1/8$. Put $p=c^2h^2$; the nonzero covariance
block is
\begin{equation}
 \overline M_7\big|_{1:2,1:2}
 =\frac18\begin{bmatrix}
 2-2p&ch(h^2-c^2)\\ch(h^2-c^2)&2p
 \end{bmatrix},\qquad
 \lambda_{\max}(\overline M_7)=\frac{1+\sqrt{1-3p}}8.
 \label{eq:matched-fixed-response-query-covariance}
\end{equation}
At $t_1=1/4$ and $t_2=1/2$,
$p_1=14400/83521<p_2=144/625$, so the covariance maximum strictly
decreases. The feature margins are $16/289$ and $9/100$; the separate
response-uniform spectral targets are $257/1156$ and $41/200$.

Tracing the fixed-response covariance against the three fixed query metrics
gives the exact SSE increments and rounded values in
Table~\ref{tab:matched-query-loss}. The raw response, residual
energy, Full coefficient, each query, dimensions, size, leverage vector, and
entire indexed sampling distribution are fixed across designs. The design
and its residual projection vary. Every fixed numerical query target is
allowed in \eqref{eq:query-trace-identity}; the Full query predictions are
unchanged. At sizes $5,6$, the trace identities still hold but the indexed
probability vectors vary with $t$.

The block covariance formula proves these values directly. Exact rational
enumeration of the positive-volume support provides an independent finite check.

\begin{table}[htbp]
\centering

\caption{\textbf{Matched designs, different query behavior.}
Expected query excess SSE at $s=7$. Each design cell shows the exact fraction
above its six-decimal rounded value; change is relative to $t=1/4$.}
\label{tab:matched-query-loss}
\setlength{\tabcolsep}{3pt}
\fontsize{8}{9.5}\selectfont

\begin{tabular}{@{}lrrr@{}}
\toprule
$Z$ & $t=1/4$ & $t=1/2$ & Change \\
\midrule
$I_4$ & \shortstack[r]{$1/4$\\$(0.250000)$}
      & \shortstack[r]{$1/4$\\$(0.250000)$} & $+0.00\%$\\
$\operatorname{diag}(2,1,1,1)$
      & \shortstack[r]{$72721/83521$\\$(0.870691)$}
      & \shortstack[r]{$517/625$\\$(0.827200)$} & $-5.00\%$\\
$\operatorname{diag}(1,2,1,1)$
      & \shortstack[r]{$126721/334084$\\$(0.379309)$}
      & \shortstack[r]{$1057/2500$\\$(0.422800)$} & $+11.47\%$\\
\bottomrule\end{tabular}

\end{table}

\subsection{Complement and deletion geometry along the rotation path}
\label{app:boundary-path-diagnostic}
For the same family $A_t$ in \eqref{eq:app-matched-phase-pair-design},
restrict $0\le t\le1/3$. Write $c=(1-t^2)/(1+t^2)$ and
$h=2t/(1+t^2)$, so $c\ge h\ge0$. The complement has diagonal $1/2$
and largest absolute off-diagonal entry $c/2$. Therefore its normalized
coherence, least pair minor, least two-row deletion eigenvalue, and
complement certificate are
\[
 \rho_\perp=c,\qquad
 \mu_{\rm pair}=\frac{1-c^2}{4}=\frac{t^2}{(1+t^2)^2},\qquad
 \delta_{\rm del}=\frac{1-c}{2}=\frac{t^2}{1+t^2},\qquad
 t_\perp=\frac{1-c}{4}=\frac{t^2}{2(1+t^2)}.
\]
The deletion identity follows from the two-row Gram eigenvalues
$1/2\pm|a_i^\top a_j|$ and $A_{-\{i,j\}}^\top A_{-\{i,j\}}=
I-a_ia_i^\top-a_ja_j^\top$. Thus pair deletion loses rank at $t=0$
and preserves rank at every $t>0$ in this interval. For the fixed residual
$J_te_1$, the LOO spectral value is already given by
\eqref{eq:matched-fixed-response-query-covariance}; the response-uniform
target and exact margin are given by
\eqref{eq:matched-exact-all-interior} and~\eqref{eq:matched-family-exact-margin}.
These different quantities need no separate finite-grid diagnostic.
\subsection{Equal-leverage leave-one-out separation}
\label{app:equal-leverage-separation}

For no-coloop $A^\top A=I_d$, the common LOO identity
\eqref{eq:equal-leverage-loo-covariance} specializes under equal leverage
$\ell$ to $(m-d)^{-1}[I_d-R_A(z)/(1-\ell)]$.
Optimizing its top eigenvalue over compatible unit residuals gives the
spectral counterpart of Proposition~\ref{prop:query-loo-risk}.

For a union of $k\ge2$ orthobases scaled by $k^{-1/2}$, every leverage is $1/k$.
Write $t_{\rm lev}=\max\{0,1-2\max_i\ell_i\}=1-2/k$.
If $\mu$ is the largest cross-basis absolute inner product, direct frame
algebra gives
\begin{equation}
 \rho_\perp=\frac{\mu}{k-1},\qquad
 \kappa_\perp=\left(\frac{k-1}{k}\right)^2,\qquad
 t_\perp-t_{\rm lev}=\frac{1-(k-1)\mu}{k^2}.
 \label{eq:union-bases-pair-gain}
\end{equation}
Thus pair information is quantitatively stronger whenever
$\mu<1/(k-1)$, including regimes with positive maximum-leverage
certification.

For $k=3,d=16$, compare three identical repeated row orthobases with three real
mutually unbiased row orthobases. Write
$A=3^{-1/2}[I;H;HD]$ for the latter. Both have $(m,d,s)=(48,16,47)$ and the complete
leverage vector $1/3$.  In the repeated construction the rows are
$e_j/\sqrt3$. Compatibility makes each coordinate triple sum to zero; if
$E_j$ is its residual energy, the residual operator is diagonal with entry
$2/3-E_j/3$. Thus its smallest eigenvalue is at least $1/3$, and a unit
sum-zero residual supported on one triple attains equality. Hence $\nu_A=1/3$
and the exact target is $1/64$, or ratio $1/2$ to $\alpha_{47}$.

For completeness, take $U_1=I$, $U_2=H$, and $U_3=HD$, where
$H_{xy}=(-1)^{x\cdot y}/4$ for $x,y\in\{0,1\}^4$ and
$D_{yy}=(-1)^{y_1y_2+y_3y_4}$. The Walsh transform of this bent sign vector
has absolute value $4$, so direct cross-inner-product calculation gives
$\mu=1/4$. This mutually unbiased construction has $t_\perp=7/18>t_{\rm lev}=1/3$,
and therefore target at most $5/384$, or ratio at most $5/12$, versus the
exact repeated-design target $1/64$. This is at least one sixth smaller
relative to the repeated-basis target. At tolerance $1/2$, on this
incoherent construction the pair-route resolvent first pass is $28$ and the
maximum-leverage-route first pass is $29$. Repeated bases and mutually
unbiased frames are classical erasure geometry \citep{bodmann2005frames,fickus2012nerf};
the displayed consequence uses the response-uniform ordinary-volume covariance
target of this paper.
\subsection{Sharpness within row general position}
\begin{proposition}[Strict interior slack and row-general-position supremum]
\label{prop:ue-interior-supremum}
If $X$ is in row general position, $\lossstar>0$, $d\ge2$, and $d<s<m$, then
\begin{equation}
  M_s\prec\alpha\lossstar G^{-1}.
  \label{eq:ue-strict-interior}
\end{equation}
For every fixed $(m,d,s)$ with $m>d$, $d\ge2$, and $d\le s<m$, row-general-position
perturbations $X^{(a)}\to X_0$ of the family
\eqref{eq:ue-attaining-family} can be chosen so that, with the same fixed
$y$,
\begin{equation}
  \left\lVert
    \frac{G_a^{1/2}M_{s,a}G_a^{1/2}}{\lossstar_a}-\alpha I_d
  \right\rVert_{\mathrm{op}}\longrightarrow0.
  \label{eq:ue-operator-supremum}
\end{equation}
Thus, for $d\ge2$ and $d<s<m$, the coefficient is not attained in row general
position but is its supremum.  At the rank-size endpoint $s=d$, the basis
identity already gives equality for every positive-loss row-general-position
design; at $s=m$ no normalized ratio is formed.
\end{proposition}
\begin{proof}The strictness and perturbative supremum argument is given in Section~\ref{app:supremum-details}.\end{proof}
\subsection{Near-attainment with a uniformly positive margin}

\label{app:critical-class-calibration}

Set $\sigma_0=\sqrt{39}/8$ and let $\mathcal C_d$ contain the matrices
$A\in\mathbb R^{2d\times d}$ with $A^\top A=I_d$, row leverage $1/2$, and
projective separation $\min_{i\ne j}\sin\angle_{\rm proj}(a_i,a_j)\ge\sigma_0$.
For this class, $\rho_A=\max_{i\ne j}2|a_i^\top a_j|\le5/8$ exactly. For a
positive-loss response and $s=d+r$, define
\[
 \Delta(A,y;r)=1-
 \frac{\lambda_{\max}(\overline M_{d+r})}{(1-r/d)L^*}.
\]
The class is empty at $d=2$; the following pointwise statement is therefore
conditional on membership there.

\begin{proposition}[Critical separated-class covariance gap]
\label{prop:critical-class-rate-calibration}
For every $d\ge2$, $A\in\mathcal C_d$, positive-loss fixed response $y$, and
$1\le r\le d-1$,
\begin{equation}
 \nu_A\ge\frac3{32},\qquad
 \Delta(A,y;r)\ge\frac{3r}{32(d-1)}.
 \label{eq:critical-class-derived-lower}
\end{equation}
\end{proposition}
\begin{proof}
Theorem~\ref{thm:critical-geometry} gives
$\nu_A\ge(1-\rho_A)/4\ge3/32$. Since ordinary leverage is $1/2<1$,
Corollary~\ref{cor:margin-slack} applies. Here
$\beta_s\gamma_s/\alpha_s=r/(d-1)$, which yields the stated
response-uniform fixed-response covariance gap.
\end{proof}

\begin{lemma}[Totally nonsingular orthogonal perturbation]
\label{lem:critical-class-orthogonal-block}
For every $d=4^p$, $p\ge1$, there exists a rational $Q_d\in O(d)$ whose
square minors are all nonzero and for which
$\lVert Q_d-H_d\rVert_{\max}<1/8$, where $H_d$ is the normalized Sylvester
matrix.
\end{lemma}
\begin{proof}
For $d=4^p$, $H_d$ is rational. Rational Givens words are dense in $SO(d)$,
and each proper-minor zero set has empty interior on either component of
$O(d)$. The finite intersection of their open dense complements meets the
open entrywise ball around $H_d$. Rational orthogonal matrices are dense
in that component: multiply rational $SO(d)$ matrices by $H_d$. Hence
the nonempty open intersection contains a rational $Q_d$.
\end{proof}

\begin{proposition}[Restricted near-attainment in the positive phase]
\label{prop:critical-class-near-attainment}
For $d=4^p$, $p\ge1$, choose $Q_d$ from Lemma~\ref{lem:critical-class-orthogonal-block}
and set
\begin{equation}
 A_d=2^{-1/2}\begin{bmatrix}I_d\\Q_d\end{bmatrix},\qquad
 z_d=2^{-1/2}\begin{bmatrix}e_1\\-Q_de_1\end{bmatrix},\qquad y_d=z_d.
 \label{eq:critical-class-witness}
\end{equation}
For every $1\le r<d/2$, this one design-response pair has $L^*=1$, ordinary
leverage $1/2$, $\rho_{A_d}<5/8$, and
\begin{equation}
 \frac{3r}{32(d-1)}\le\Delta(A_d,y_d;r)<\frac{r}{d-r}.
 \label{eq:critical-class-squeeze}
\end{equation}
\end{proposition}
\begin{proof}
Orthogonality gives $A_d^\top A_d=I_d$, $A_d^\top z_d=0$, and $\|z_d\|=1$.
Total nonsingularity gives ordinary row general position, while
$|(Q_d)_{ji}|<d^{-1/2}+1/8\le5/8$ gives the stated class membership. Ordinary
Cauchy--Binet gives $Z_A=\binom dr$. For $B=[A_d\ z_d]$, the augmented
normalizer is $Z_B=\binom{d-1}{r-1}$, so $Z_B/Z_A=r/d$. The first augmented
leverage is $h_1=1$, which forces row $1$ into augmented support; this is not
an ordinary coloop, since every ordinary leverage is $1/2$. 
For an augmented-supported set, let $J$ index the selected bottom rows and
$I$ the selected top rows other than row $1$. Write
$q=Q_{J,1}$, $Q'=Q_{J,2:d}$,
$D=\operatorname{diag}(\mathbf1_I)+(Q')^\top Q'$, and $b=(Q')^\top q$.
Rotating the first feature column and the residual column of $B$ isolates
the forced unit row. The remaining augmented Gram is
\[
 \begin{bmatrix}\|q\|^2&b^\top/\sqrt2\\b/\sqrt2&D/2\end{bmatrix}\succ0.
\]
Hence $D\succ0$ and $\delta=\|q\|^2-b^\top D^{-1}b>0$. The ordinary
Gram is $K_S=\tfrac12[\begin{smallmatrix}1+\|q\|^2&b^\top\\b&D\end{smallmatrix}]$,
so block inversion gives
\[
 e_1^\top K_S^{-1}e_1=\frac{2}{1+\delta}<2.
\]
The exact residual-augmentation transform therefore gives
$e_1^\top\overline M_{d+r}e_1>1-2r/d$, yielding the upper inequality after
normalization. The lower inequality is
Proposition~\ref{prop:critical-class-rate-calibration}.
\end{proof}

Thus the lower side is response-uniform over the stated separated class, while
the upper side is one fixed-response witness. Along $d=4^p$ and $r=o(d)$,
the normalized witness ratio approaches one although $\nu_{A_d}\ge3/32$;
this is a class-restricted positive-phase rate calibration, not a new
universal envelope or an equality statement.

\section{Consequences for query risk}
\label{sec:risk}\label{app:aggregate-task-fidelity}
We now translate covariance into expected excess loss relative to the
full-pool least-squares fit, denoted Full. The labels and query features
are fixed before sampling. No claim about the full fit's predictive
accuracy is needed for this identity.
\subsection{A common fixed-query loss identity}
Let $X\in\R^{m\times d}$ have full column rank, with $m\ge d$, and fix
$Y\in\R^{m\times K}$, $Z\in\R^{q\times d}$, and
$Y_Q\in\R^{q\times K}$, where $q,K\ge1$. For each fixed
$d\le s\le m$, draw one ordinary size-$s$ volume subset and fit every
output by unweighted OLS on that same positive-volume support. When $m=d$,
this is just Full; $s=d<m$ uses supported square systems, and $s=m$ is
exact Full. Write $W_S$ and $W_F$ for the selected and Full fits,
$G=X^\top X$, $M_{s,j}$ for the coefficient covariance of output $j$,
$\overline M_{s,j}=G^{1/2}M_{s,j}G^{1/2}$, and
$\Xi=G^{-1/2}Z^\top ZG^{-1/2}\succeq0$. With
$L_{Q,S}=\|ZW_S-Y_Q\|_F^2/(qK)$ and
$L_{Q,F}=\|ZW_F-Y_Q\|_F^2/(qK)$, expansion gives
\begin{align}
 \E_X L_{Q,S}-L_{Q,F}
 &=\frac{2\langle ZW_F-Y_Q,Z\E_X(W_S-W_F)\rangle_F
          +\E_X\|Z(W_S-W_F)\|_F^2}{qK}\notag\\
 &=\frac1{qK}\sum_{j=1}^K\tr(ZM_{s,j}Z^\top)
  =\frac1{qK}\sum_{j=1}^K\tr(\Xi\overline M_{s,j})
  =\mathcal D_s.
 \label{eq:app-multioutput-query-decomposition}
\end{align}
\begingroup
\makeatletter
\edef\@currentlabel{\theequation}
\edef\@currentHref{equation.\theHequation}
\label{eq:aggregate-fixed-label-identity}
\makeatother
\endgroup
Selected-OLS unbiasedness cancels the cross term; the Frobenius column sum
needs no independence between outputs or query rows. This identity needs
neither no coloops, positive residual loss, nor a nonzero query metric.
For $K=1$, multiplying by $q$ recovers the single-output SSE identity
\eqref{eq:query-trace-identity}.

\subsection{Phase-aware fixed-query variance}
\label{sec:task-conditioned-risk}

The phase controls coefficient variation, not task fit.  Its bounded
linear-readout consequence becomes explicit on a fixed external query pool.
Let $Z\in\R^{q\times d}$ contain the query features and let $y_Q\in\R^q$
contain fixed scalar targets, both fixed before the subset draw.  Define
\begin{equation}
  F^Q=\lVert Zw^*-y_Q\rVert_2^2,
  \qquad
  V_s^Q=\tr(ZM_sZ^\top),
  \qquad
  \kappa^Q=\tr(ZG^{-1}Z^\top).
  \label{eq:query-risk-terms}
\end{equation}
Here $F^Q$ is the squared loss of the full-training-pool head on the query
pool, while $V_s^Q$ is the additional expected loss due only to subset
refitting.

\begin{corollary}[Phase-aware fixed-query squared loss]
\label{cor:task-conditioned-query-risk}
Assume the no-coloop positive-loss strict-interior domain of
Theorem~\ref{thm:robust-pre-response-phase}, and let a verified feature-only
certificate satisfy $0\le t(A)\le\nu_A$.  With $(Z,y_Q)$ fixed before the
subset draw,
\begin{equation}
\begin{split}
  R^Q(s)&:=\E_X\lVert Zw_S-y_Q\rVert_2^2=F^Q+V_s^Q,\\
  V_s^Q&\le
  \bigl(\alpha_s-\beta_s\gamma_s t(A)\bigr)
  \lossstar\kappa^Q.
  \label{eq:query-risk-decomposition}
\end{split}
\end{equation}
Thus a positive certificate tightens the certified subset-refit variance term
relative to the universal envelope.
\end{corollary}

\begin{proof}
Expand the squared loss and use unbiasedness to cancel its cross term.
Trace the linear covariance bound against $\Xi$ and use $R_A(z)\succeq tI$.
\end{proof}
\subsection{Upper bounds for fixed multi-output responses}
\label{sec:moment-certificates-budgets}

For fixed $Y\in\R^{m\times K}$ and query $Z\in\R^{q\times d}$, let
$W_F$ be Full, $E=Y-XW_F=[e_1,\ldots,e_K]$,
$L_j=\|e_j\|_2^2$, and $L_{\rm src}=\sum_jL_j=\|E\|_F^2$.
Write $\overline M_{s,j}$ for output $j$'s full-Gram-whitened coefficient
covariance under the ordinary volume law. The absolute expected
query-MSE excess and its source-query scale are
\begin{equation}
 \mathcal D_s=\frac{\E_X\|Z(W_S-W_F)\|_F^2}{qK}
 =\frac1{qK}\sum_j\tr(\Xi\overline M_{s,j}),\qquad
 \Gamma_Q=\frac{L_{\rm src}\tr\Xi}{qK}.
 \label{eq:response-aware-risk-scale}
\end{equation}
By unbiasedness, $\mathcal D_s$ is expected query-MSE excess over Full.
If $L_{\rm src}=0$, every supported subset fit equals Full, so
$\mathcal D_s=0$ and no residual normalization is needed.

For $L_{\rm src}>0$ and $d<s<m$, the upper bounds below hold for every
full-column-rank $X$. Let $J_+=\{j:L_j>0\}$,
$z_j=e_j/\sqrt{L_j}$, $R_j=R_A(z_j)$, and
$\bar R=\sum_{j\in J_+}(L_j/L_{\rm src})R_j$.
Under the analysis-only augmented law of $B_j=[A,z_j]$, set
$K_S=A_S^\top A_S$ and
\begin{equation}
 H_{s,j}=\E_{B_j}K_S=I-\gamma_sR_j,\qquad
 V_{s,j}=\E_{B_j}(K_S-H_{s,j})^2
 \label{eq:aggregate-second-moment-definitions}.
\end{equation}
The square is a matrix square. These are augmented-law moments;
$\mathcal D_s$ uses the ordinary law. In absolute MSE units, define
\begin{align}
 U_{\rm agg}(s)&=\frac{L_{\rm src}}{qK}\tr[
 \Xi\{I-\beta_s(I-\gamma_s\bar R)^{-1}\}],\notag\\
 U_{\rm per}(s)&=\frac1{qK}\sum_{j\in J_+}L_j
 \tr[\Xi\{I-\beta_sH_{s,j}^{-1}\}],\notag\\
 C_s&=\frac{\beta_s}{qK}\sum_{j\in J_+}L_j
 \tr(\Xi H_{s,j}^{-1}V_{s,j}H_{s,j}^{-1})\geq0,\qquad
 U_2(s)=U_{\rm per}(s)-C_s.
 \label{eq:second-moment-query-refinement}
\end{align}
$U_{\rm agg}$ pools residuals; $U_{\rm per}$ keeps output-specific moments;
$U_2$ also uses source-Gram fluctuations. All share the same risk units.

\begin{corollary}[Response-aware query certificates]
\label{cor:response-aware-query-certificates}
Under these conditions,
\begin{equation}
 \boxed{\mathcal D_s\leq U_2(s)\leq U_{\rm per}(s)\leq U_{\rm agg}(s).}
 \label{eq:aggregate-query-displacement}
\end{equation}
\end{corollary}

\begin{proof}
For each $j\in J_+$, put $z_j=e_j/\sqrt{L_j}$ and $B_j=[A,z_j]$.
Under the analysis-only augmented law $\mathbb P_{B_j}$, recall
$H_{s,j}$ and $V_{s,j}$ from
Equation~\eqref{eq:aggregate-second-moment-definitions}, and put
$\mathcal Q_{s,j}=H_{s,j}^{-1}V_{s,j}H_{s,j}^{-1}$.
Here $K_S=A_S^\top A_S$; selected inverses under this law are restricted to
positive augmented volume. For every full-column-rank design and $d<s<m$,
$H_{s,j}\succ0$ and $\mathcal Q_{s,j}\succeq0$.
Equation~\eqref{eq:second-moment-query-refinement} defines $C_s$ and
$U_2=U_{\rm per}-C_s$.
The boundary-safe part of
Proposition~\ref{prop:residual-augmented-second-moment} gives
\begin{equation}
 \mathcal D_s\le U_2(s)\le U_{\rm per}(s)\le U_{\rm agg}(s).
 \label{eq:aggregate-second-moment-hierarchy}
\end{equation}
Indeed, its raw-coordinate matrices obey
$F_{s,j}=G^{1/2}H_{s,j}G^{1/2}$ and
$W_{s,j}^{\rm raw}=G^{1/2}V_{s,j}G^{1/2}$, directly from their finite-sum
definitions. Congruence by $G^{1/2}$, contraction with $\Xi\succeq0$, and
summation over outputs prove the first inequality. The second uses $C_s\ge0$;
the last is Equation~\eqref{eq:aggregate-inversion-jensen}.
These bounds hold for every full-column-rank design on the stated interior
domain; the section's endpoint conventions apply separately.

The boundary-valid output-wise resolvent inequality and trace monotonicity give
\begin{equation}
 \mathcal D_s\le U_{\rm per}(s):=
 \frac1{qK}\sum_{j\in J_+}L_j\tr\!\left[
 \Xi\{I-\beta_s(I-\gamma_sR_j)^{-1}\}\right].
 \label{eq:aggregate-per-output-bound}
\end{equation}
For $L_{\rm src}>0$, set $p_j=L_j/L_{\rm src}$. Operator convexity of
inversion gives
\begin{equation}
 \sum_{j\in J_+}p_j(I-\gamma_sR_j)^{-1}
 \succeq
 \left(I-\gamma_s\sum_{j\in J_+}p_jR_j\right)^{-1}
 =(I-\gamma_s\bar R)^{-1}.
 \label{eq:aggregate-inversion-jensen}
\end{equation}
Multiplication by $-\beta_s$, congruence with $\Xi^{1/2}$, and the trace prove
$U_{\rm per}\le U_{\rm agg}$. Thus aggregation replaces the $|J_+|$
output-specific inverses by one inverse, at the possible cost of a looser
bound. A directly checked aggregate condition
$0\le t<1/\gamma_s$ and $\bar R\succeq tI$ gives the scalar comparison
\begin{equation}
 U_{\rm agg}(s)\le
 U_{\rm scalar}(s;t):=\frac{L_{\rm src}}{qK}\Bres(s;t)\tr(\Xi).
 \label{eq:aggregate-scalar-comparator}
\end{equation}
Indeed, $(I-\gamma_s\bar R)^{-1}\succeq(1-\gamma_st)^{-1}I$ proves
\eqref{eq:aggregate-scalar-comparator}. A common output-wise lower
bound $R_j\succeq tI$ is sufficient but is not required for the aggregate
comparison.

\end{proof}

At Full ($s=m$) or $L_{\rm src}=0$, set $U_2=U_{\rm per}=U_{\rm agg}=0$.
At rank size $s=d<m$, use the universal envelope without the augmented
rank-$(d+1)$ formulas. These branches precede residual normalization.
On the contact domain, with $W_Q=\tr\Xi>0$ and $L_{\rm src}>0$,
the realized multi-output weighted average
$\tr(\Xi\bar R)/W_Q$ is at least $t_Q$. Hence
\[
 U_{\rm agg}(s)
 \le \frac{L_{\rm src}W_Q}{qK}
 \Bres\!\left(s;\frac{\tr(\Xi\bar R)}{W_Q}\right)
 \le\frac{L_{\rm src}W_Q}{qK}\Bres(s;t_Q).
\]
Thus the response-aware aggregate retains more information, while $t_Q$
supplies a guarantee before source responses are known.

For the scalar hierarchy write $\Blin(s;t)=\alpha_s-\beta_s\gamma_st$.
\paragraph{A common information hierarchy.}
On the common real Parseval/no-coloop domain with $L_{\rm src}>0$,
$\tr\Xi>0$, $d<s<m$, and $0\leq t\leq\nu_A$,
\begin{equation}
 \mathcal D_s/\Gamma_Q\leq U_2/\Gamma_Q\leq U_{\rm per}/\Gamma_Q
 \leq U_{\rm agg}/\Gamma_Q\leq\Bres(s;t_Q)
 \leq\Bres(s;t)\leq\Blin(s;t)\leq\alpha_s.
 \label{eq:information-hierarchy-main}
\end{equation}
Rightward bounds discard second moments, output-specific inverses, realized
geometry, resolvents, and slack (weighted-trace link: Section~\ref{app:query-residual-margin}).

\subsection{Omission lower bounds}
\label{sec:omission-lower-certificates}\label{app:loo-lower-certificate}
The matrix comparison already gives
\begin{equation}
 \mathcal D_{s,g}\ge(m-s)\mathcal D_{m-1,g}
 \label{eq:loo-basic-risk-lower}
\end{equation}
for any fixed query workload $g$ on the no-coloop contact domain. We sharpen
it by optimizing a scalar multiple of the omission predictor.
The upper certificates in Section~\ref{sec:moment-certificates-budgets} apply to
every full-column-rank design. The lower results here additionally require
$m\ge d+2$ and no coloops: every one-row deletion retains column rank. Write
$G=X^\top X$, let $x_i^\top$ be row $i$, and set
\begin{equation}
 \begin{aligned}
 P&=I-XG^{-1}X^\top,& c_i&=P_{ii}>0,& n&=m-d,\\
 k&=m-s,& \alpha&=k/n,& \lambda&=(k-1)/(n-1).
 \end{aligned}
 \label{eq:loo-lower-domain}
\end{equation}
Here $d<s<m$, so $1\le k<n$. Ordinary indexed volume sampling assigns
probability only to positive-volume subsets. We define no inverse or fit on
the zero-volume subsets.

\paragraph{Supported-law omission moments.}
For one response column, write $\beta_F=w^*$ and $\beta_S=w_S$ for the
full and supported subset fits of Section~\ref{sec:setup}. Let
$e=y-X\beta_F$, $v_S=\beta_S-\beta_F$, and
\begin{equation}
 f_i=-\frac{e_i}{c_i}G^{-1}x_i,\qquad
 F_S=\sum_{i\notin S}f_i,\qquad
 L_0=\sum_i\frac{e_i^2}{c_i}G^{-1}x_ix_i^\top G^{-1}.
 \label{eq:loo-raw-omission-predictor}
\end{equation}
If $e=0$, every supported fit equals $\beta_F$, so $v_S=f_i=F_S=0$
and $L_0=L_1=0$ in the formulas below. We can therefore assume
$L^*:=\|e\|_2^2>0$ and set
\[
 A=XG^{-1/2},\qquad z=e/\sqrt{L^*},\qquad
 T=\sqrt{L^*}\,G^{-1/2}.
\]
Then $A^\top A=I$, $A^\top z=0$, and its residual projector is the
same $P$. For every supported $S$,
$\det(X_S^\top X_S)=\det(G)\det(A_S^\top A_S)$.
The factor $\det(G)>0$ cancels from the volume law, so the supported
sets and their probabilities agree in the two coordinates. Writing
$v_S^{A,z}$, $f_i^{A,z}$, and $F_S^{A,z}$ for the predictors in
Section~\ref{app:loo-quantitative-geometry}, direct substitution gives
\[
 v_S=T v_S^{A,z},\qquad f_i=T f_i^{A,z},\qquad
 F_S=T F_S^{A,z},\qquad
 L_0=T L(z)T^\top,\qquad L_1=T H(z)T^\top.
\]
Here $L(z)$ and $H(z)$ are defined in
\eqref{eq:loo-omission-predictor} and
\eqref{eq:loo-omission-schur-moment}, respectively. The singleton and
pair exclusion probabilities from that supported law read
\begin{equation}
 \Pr(i\notin S)=\alpha c_i,\qquad
 \Pr(i,j\notin S)=\alpha\lambda(c_ic_j-P_{ij}^2)\quad(i\ne j).
 \label{eq:loo-exclusion-marginals}
\end{equation}
In particular, $\eta_s=\alpha\lambda$. Applying the exact normalized
moments \eqref{eq:loo-omission-moments} through $T$ gives
\begin{align}
 \mathbb E[v_SF_S^\top]&=\alpha L_0,\notag\\
 \mathbb E[F_SF_S^\top]&=\alpha(L_0-\lambda L_1),\quad
 L_1=\sum_{i,j}\frac{e_ie_jP_{ij}^2}{c_ic_j}
       G^{-1}x_ix_j^\top G^{-1}.
 \label{eq:loo-raw-omission-moments}
\end{align}
At $k=1$, only row $i$ is omitted, $\lambda=0$, and the selected fit is
the full fit of $X_{-i}$; hence $v_S=f_i=F_S$ on each supported set.
All identities above include this endpoint and assign no fit to an
unsupported set.

Let $C=P\circ P$, $D=\operatorname{diag}(c_i)$, and let row $i$ of $B$
be $(e_i/c_i)x_i^\top G^{-1}$. Then $L_0=B^\top DB$ and
$L_1=B^\top CB$. The Schur product theorem gives $C\succeq0$.
Since $P$ is a projector,
$\sum_jP_{ij}^2=c_i$ and
\begin{equation}
 h^\top(D-C)h=\frac12\sum_{i,j}P_{ij}^2(h_i-h_j)^2\ge0.
 \label{eq:loo-schur-laplacian}
\end{equation}
Thus $0\preceq L_1\preceq L_0$. For any fixed query matrix $Z$, define
$U_0=\tr(ZL_0Z^\top)$ and $U_1=\tr(ZL_1Z^\top)$. In particular,
$0\le U_1\le U_0$.

\begin{proposition}[Scalar omission lower bound]
\label{prop:scalar-omission-lower}
For the preceding full-rank no-coloop design, fixed query $Z$, and
$d<s<m$, if $U_0>0$ then
\begin{equation}
 \mathbb E\|Zv_S\|_2^2\ge
 \alpha\frac{U_0^2}{U_0-\lambda U_1}
 \ge\alpha(U_0+\lambda U_1)\ge\alpha U_0.
 \label{eq:loo-scalar-chain}
\end{equation}

If $U_0=0$, set the lower bound to zero. For a query workload $Z_g$ with
$q_g$ rows and $K$ outputs, let $U_{0,j},U_{1,j}$ be the corresponding
moments of residual column $j$. With $\alpha_s=\alpha$ and
$\lambda_s=\lambda$, the aggregate risk satisfies
\begin{equation}
 \mathcal D_{s,g}\geq \mathcal L^{\rm sc}_{s,g}:=
 \frac{\alpha_s}{q_gK}\sum_{j=1}^K
 \begin{cases}U_{0,j}^2/(U_{0,j}-\lambda_sU_{1,j}),&U_{0,j}>0,\\
 0,&U_{0,j}=0.
 \end{cases}
 \label{eq:loo-scalar-risk-lower}
\end{equation}

This is at least $(m-s)\mathcal D_{m-1,g}$, with equality of the lower
bound and actual risk at LOO. The zero-moment branch does not imply zero
interior risk.
\end{proposition}
\begin{proof}
For every real multiplier $a$, the supported moments give
\begin{equation}
 \mathbb E\|Z(v_S-aF_S)\|_2^2
 =\mathbb E\|Zv_S\|_2^2
 -\alpha\{(2a-a^2)U_0+a^2\lambda U_1\}.
 \label{eq:loo-scalar-remainder}
\end{equation}
Since $0\le U_1\le U_0$ and $0\le\lambda<1$, the brace is maximized at
$a_*=U_0/(U_0-\lambda U_1)$ when $U_0>0$. Nonnegativity of the left side
proves the first lower bound. The inequality
$1/(1-x)\ge1+x$ for $0\le x<1$ gives the remaining chain.
Equality in its first step requires $Z(v_S-a_*F_S)=0$ on every supported
subset. At LOO, $v_S=F_S$, $\lambda=0$, and the risk equals $U_0/n$.
Sum over outputs and divide by $q_gK$ to obtain the aggregate bound.
Optimizing each output separately can only improve on a shared multiplier.
\end{proof}
\subsection{Monotonicity under nested volume sampling}
\label{app:task-conditioned-query-risk}

For fixed source responses $Y$, query features $Z$, and query targets $Y_Q$,
use the common-support identity
\eqref{eq:app-multioutput-query-decomposition}. Write $M_{s,j}$ for the
coefficient covariance of output $j$, $L_j^*$ for its Full residual energy,
and $L_{Q,s}$ for query MSE normalized by $qK$. On the real
no-coloop domain with $d<s<m$ and $0\le t\le\nu_A$, scalarization yields
\begin{equation}
 M_{s,j}\preceq\Bres(s;t)L_j^*G^{-1}
             \preceq\Blin(s;t)L_j^*G^{-1}.
 \label{eq:app-multioutput-column-covariance}
\end{equation}
Zero-loss columns contribute zero. Summing and contracting gives
$\mathcal D_s\le B(s;t)\Gamma_Q$ for either scalar envelope, where
$\Gamma_Q=(\sum_jL_j^*)\tr(ZG^{-1}Z^\top)/(qK)$. If $L_{Q,F}>0$,
\begin{equation}
 C_Q=\Gamma_Q/L_{Q,F},\qquad
 \frac{\E L_{Q,s}-L_{Q,F}}{L_{Q,F}}\le B(s;t)C_Q.
 \label{eq:query-scale}
\end{equation}
The additive identity remains valid when a relative ratio is undefined.

\paragraph{Nested sizes and paired loss.}
\label{app:nested-conditional-volume}

A nested coupling draws one rank-$d$ volume basis and one uniformly
permuted complement, then takes the required prefixes
\citep{derezinski2017unbiased,derezinski2018reverse}.
For fixed $d\le a<b\le m$ and positive-volume $U$ of size $b$, this gives
\begin{equation}
 \PP(S_a=T\mid S_b=U)
 =\frac{\det(X_T^\top X_T)}{\binom{b-d}{a-d}\det(X_U^\top X_U)}.
 \label{eq:nested-conditional-volume-law}
\end{equation}
Indeed, summing over bases inside $T\subseteq U$ by Cauchy--Binet and counting
complement orders gives the joint mass
\begin{equation}
 \PP(S_a=T,S_b=U)=
 \frac{\det(X_T^\top X_T)}{
 \det(G)\binom{m-d}{a-d}\binom{m-a}{b-a}}.
 \label{eq:app-nested-volume-joint-law}
\end{equation}
Division by the ordinary size-$b$ mass and the binomial identity
$\binom{m-d}{a-d}\binom{m-a}{b-a}
 =\binom{m-d}{b-d}\binom{b-d}{a-d}$ prove the conditional law.
Zero-volume subsets have zero mass. Conditional selected-OLS unbiasedness
therefore gives
\begin{equation}
 \E[W_a-W_b\mid S_b]=0.
 \label{eq:app-conditional-selected-ols-unbiasedness}
\end{equation}
Expanding the squared query loss and cancelling the conditional cross term,
\begin{equation}
 H_{a,b}:=\frac{\|Z(W_a-W_b)\|_F^2}{qK},\qquad
 \E[L_{Q,a}-L_{Q,b}]=\E H_{a,b}.
 \label{eq:nested-conditional-loss-identity}
\end{equation}
On the scalar-certificate domain with $d<a<b\le m$, the same argument gives
\begin{equation}
 \E H_{a,b}\le\E L_{Q,a}-L_{Q,F}\le\Bres(a;t)\Gamma_Q,
 \label{eq:nested-loss-envelope-consequence}
\end{equation}
since $\E L_{Q,b}\ge L_{Q,F}$. This does not subtract separately constructed
upper bounds or equate realized loss differences with squared displacement.

\begin{lemma}[Conditional Pythagoras]
\label{lem:nested-volume-conditional-pythagoras}
For the preceding path, let $\mathcal G_s=\sigma(S_s,\ldots,S_m)$ and
$P_s=ZW_s$. For fixed $a<b$, $\E[W_a\mid\mathcal G_b]=W_b$. For any
square-integrable $\mathcal G_b$-measurable matrix $A$ shaped as $P_s$,
\[
 \E[\|P_a-A\|_F^2\mid\mathcal G_b]
 =\|P_b-A\|_F^2+\E[\|P_a-P_b\|_F^2\mid\mathcal G_b].
\]
Consequently, for $a<b<c$,
$\E[H_{a,c}\mid\mathcal G_b]=\E[H_{a,b}\mid\mathcal G_b]+H_{b,c}$;
with $I_s=P_s-P_{s+1}$,
$\E H_{a,b}=\sum_{s=a}^{b-1}\E\|I_s\|_F^2/(qK)$.
\end{lemma}
\begin{proof}
Given $S_b$, the inner path is conditionally independent of the later
complement order. Thus conditional unbiasedness extends to $\mathcal G_b$.
Expanding $P_a-A=(P_a-P_b)+(P_b-A)$ cancels the conditional cross term.
Taking $A=P_c$ gives additivity, and iterating it gives the increment sum.
\end{proof}

\begin{corollary}[Coefficient-covariance increments]
\label{cor:nested-volume-covariance-increments}
On this fixed full-rank, common-support path, for each output and fixed
$d\le a<b\le m$,
\begin{equation}
 M_{a,j}-M_{b,j}
 =\E[(w_{a,j}-w_{b,j})(w_{a,j}-w_{b,j})^\top]\succeq0.
 \label{eq:nested-volume-covariance-increment}
\end{equation}
\end{corollary}
\begin{proof}
Expand $w_{a,j}-w_j^*=(w_{a,j}-w_{b,j})+(w_{b,j}-w_j^*)$.
The cross terms vanish conditional on $\mathcal G_b$; summing over outputs
also preserves the identity and positive-semidefinite order.
\end{proof}

These nested-law statements need no margin or no-coloop assumption, but
sizes must be fixed independently of path outcomes. For example, with
$X=\mathbf1_3$, $y=(0,0,3)^\top$, choose size one when its fitted value is
$3$ and size two otherwise. The selected coefficient then has expectation
$\frac13\,3+\frac23\,\frac34=\frac32$, whereas the Full coefficient is $1$.
Fixed-size unbiasedness therefore does not justify outcome-adaptive stopping.

\section{Discussion}
\label{sec:limitations}
Contact is a subspace property on a fixed design. The LOO comparison makes
the contact kernel independent of interior sample size; complement classes
identify its dimension, orientation, and compatible residuals. Directional
attainment and whole-query attainment are therefore different questions.
The feature-only phase test resolves the former, while query containment
resolves the latter.

Quantitative bounds retain several distinct forms of information.
The LOO operator measures angular mismatch globally, residual augmentation
provides a resolvent and an ordered second-moment correction, and local
class geometry controls simultaneous residual and query misalignment.
None of the positive-margin bounds is a general exact formula for the
covariance gap. The near-attainment construction shows why sampling budget
and dimension cannot be suppressed from such a formula. Residual spectra
alone also cannot determine a fixed query's risk.

The support restrictions are substantive. The universal envelope and
augmented upper inequalities hold on arbitrary full-column-rank designs;
the contact and omission lower results use no coloops. Exact augmented
inverse-moment covariance identities require row general position, and
covariance may jump when supported subsets disappear. Rank-size, Full,
and zero-loss branches are handled separately.

The results concern a finite-pool randomized least-squares primitive.
They do not establish typical-case performance or an ordering of feature
representations on unspecified prediction tasks. For a fixed task, they
provide upper and lower bounds on expected subset-refit excess loss.
The appendices describe how to compute and rigorously enclose those bounds.

\paragraph{Acknowledgement of AI assistance.}
Generative AI tools assisted with proof development and review, literature
checks, manuscript editing, and code diagnostics. The authors are responsible
for the final content.

\bibliographystyle{plainnat}
\bibliography{references}
\clearpage
\appendix
\section{Determinant and support details}\label{app:determinants}
\subsection{Entrywise basis moments}\label{app:basis-details}
\begin{proof}[Proof of Lemma~\ref{lem:ue-basis-moments}]
The normal equations give $A^\top r=0$.  For $j\in[d]$, let $B_j$ be
$A$ with column $j$ replaced by $r$.  Cramer's rule, using increasing row
order throughout, gives
\begin{equation}
  \det(A_T)(v_T-v)_j=\det((B_j)_T).
  \label{eq:ue-cramer}
\end{equation}
The polarized Cauchy--Binet identity
\begin{equation}
  \sum_{|T|=d}\det(C_T)\det(D_T)=\det(C^\top D)
  \label{eq:ue-polarized-cauchy-binet}
\end{equation}
therefore gives
\[
  \det(K)\E[(v_T-v)_j]=\det(A^\top B_j)=0,
\]
because column $j$ of $A^\top B_j$ is $A^\top r=0$.

For the full second moment, let $\rho=\lVert r\rVert_2^2$.  Equations
\eqref{eq:ue-cramer}--\eqref{eq:ue-polarized-cauchy-binet} give
\begin{equation}
  \det(K)\E[(v_T-v)_j(v_T-v)_k]=\det(B_j^\top B_k).
  \label{eq:ue-basis-entry}
\end{equation}
If $j=k$, expansion through the replaced row and column gives
\[
  \det(B_j^\top B_j)
  =\rho\det(K_{-j,-j})
  =\rho\det(K)(K^{-1})_{jj}.
\]
If $j\ne k$, row $j$ of $B_j^\top B_k$ has the sole nonzero entry
$\rho$ in column $k$, and hence
\[
  \det(B_j^\top B_k)
  =(-1)^{j+k}\rho\det(K_{-j,-k})
  =\rho\det(K)(K^{-1})_{kj}.
\]
Symmetry of $K^{-1}$ completes every entry of
\eqref{eq:ue-basis-moments}.
\end{proof}
\subsection{A common one-sided boundary passage}\label{app:boundary-details}
\begin{proof}[Proof of Proposition~\ref{prop:ue-boundary-envelope}]
Let $V\in\R^{m\times d}$ be the Vandermonde matrix
$V_{ij}=i^{j-1}$.  For each $d$-set $T$,
$\det(X_T+tV_T)$ is a nonzero polynomial in $t$, since its leading
coefficient is $\det(V_T)\ne0$.  Choose a sequence $t_a\downarrow0$ avoiding
the finitely many roots of all these polynomials and set
$X^{(a)}=X+t_aV$.  Then every $X^{(a)}$ is in row general position and
$X^{(a)}\to X$.

Use a subscript $a$ for quantities formed from $X^{(a)}$, keeping $y$ fixed.
Fix $u\in\R^d$, and let
$\mathcal P=\{S:|S|=s,\ D_S>0\}$.  For $S\in\mathcal P$, ordinary inverse
continuity gives
\[
  D_{S,a}\{u^\top(w_{S,a}-w_a^*)\}^2
  \longrightarrow
  D_S\{u^\top(w_S-w^*)\}^2.
\]
No selected inverse is continued for $S\notin\mathcal P$.  All such omitted
perturbed terms are nonnegative, so the finite sum obeys
\begin{equation}
\begin{aligned}
  Z_{X,s}\,u^\top M_su
  &=\lim_a\sum_{S\in\mathcal P}
    D_{S,a}\{u^\top(w_{S,a}-w_a^*)\}^2\\
  &\le\liminf_a\sum_{|S|=s}
    D_{S,a}\{u^\top(w_{S,a}-w_a^*)\}^2.
\end{aligned}
\label{eq:ue-directional-liminf}
\end{equation}
For every $a$, the row-general-position envelope bounds the last sum by
\[
  Z_{X^{(a)},s}\,\alpha\lossstar_a\,
  u^\top G_a^{-1}u.
\]
Full column rank persists near $X$, and the normalizer, full fit, residual
loss, Gram inverse, and right side all converge ordinarily.  Hence
\eqref{eq:ue-directional-liminf} gives
\[
  u^\top M_su\le\alpha\lossstar u^\top G^{-1}u.
\]
Because this holds for every fixed $u$, it is the first inequality in
\eqref{eq:ue-envelope}; congruence gives the second.  Finally,
Proposition~\ref{prop:ue-boundary-mean} makes $M_s$ a covariance, and
\eqref{eq:ue-pythagoras}--\eqref{eq:ue-loss-from-envelope} give the loss
corollary.  The argument retains limiting positive-volume terms and drops
only newly supported nonnegative terms; it does not continue
\eqref{eq:ue-rgp-covariance} to the boundary.
\end{proof}
\subsection{Row-general-position supremum}\label{app:supremum-details}
\begin{proof}[Proof of Proposition~\ref{prop:ue-interior-supremum}]
For strictness, the slack is the positive-semidefinite expectation in
\eqref{eq:ue-rgp-slack}.  We show that it has no nonzero null direction.
Let $H=[X\ \ e]$.  Since $X^\top e=0$ and $\lossstar>0$,
$\operatorname{rank}(H)=d+1$; since $d\ge2$ and $d<s<m$, also
$m\ge d+2$.
The rank-$(d+1)$ row matroid of $H$ has a circuit.  No set of at most $d$
rows of $H$ can be dependent, because projection of such a dependence onto
the first $d$ coordinates would contradict row general position of $X$.
Hence the circuit contains at least $d+1$ rows.  Every circuit element is a
non-coloop, so $H$ has at least $d+1$ non-coloops.

Fix $u\ne0$ and put $v=G^{-1}u\ne0$.  At most $d-1$ rows can satisfy
$x_i^\top v=0$: any $d$ such rows would form a singular $d$-row submatrix.
Thus some non-coloop $i$ satisfies $x_i^\top v\ne0$.  Because $i$ is not a
coloop, $H$ has a row basis $R$ avoiding $i$.  Extend it to a size-$s$ set
$R\subseteq S\subseteq[m]\setminus\{i\}$, possible because
$d+1\le s\le m-1$.  Then $\operatorname{rank}(H_S)=d+1$, and the Schur
identity~\eqref{eq:ue-schur-residual} together with row general position
gives $L_S>0$ and $D_S>0$.

Let $C_S=G-G_S=X_{S^c}^\top X_{S^c}$.  Since $u=Gv$, direct expansion gives
\begin{equation}
\begin{aligned}
  u^\top(G_S^{-1}-G^{-1})u
  &=v^\top\{C_S+C_SG_S^{-1}C_S\}v\\
  &\ge v^\top C_Sv
   =\lVert X_{S^c}v\rVert_2^2>0,
\end{aligned}
\label{eq:ue-directionwise-strictness}
\end{equation}
where the last inequality uses $i\notin S$ and $x_i^\top v\ne0$.
This $S$ has positive probability and positive $L_S$.  Every other summand
in~\eqref{eq:ue-rgp-slack} is positive semidefinite, so the slack is positive
in every nonzero direction $u$, proving~\eqref{eq:ue-strict-interior}.

For the supremum, use the Vandermonde perturbations from
Proposition~\ref{prop:ue-boundary-envelope} with $X=X_0$.  The directional
liminf~\eqref{eq:ue-directional-liminf}, applied to a fixed basis of
coefficient directions and summed, together with exact attainment at the
limit, gives
\begin{equation}
  \liminf_a\tr(G_aM_{s,a})\ge d\alpha\lossstar.
  \label{eq:ue-trace-lower-limit}
\end{equation}
For completeness, one may take the fixed directions
$u_j=G^{1/2}e_j$: their limiting sum is
$\tr(GM_s)=d\alpha\lossstar$.  Replacing $G$ by $G_a$ in the trace changes
the sum by $o(1)$, because $G_a\to G$ and
$M_{s,a}\preceq\alpha\lossstar_aG_a^{-1}$ keeps $M_{s,a}$ bounded.
The envelope also gives the matching upper bound
\begin{equation}
  \tr(G_aM_{s,a})\le d\alpha\lossstar_a
  \longrightarrow d\alpha\lossstar.
  \label{eq:ue-trace-upper-limit}
\end{equation}
Consequently the matrices
\[
  R_a:=\alpha\lossstar_aI_d-G_a^{1/2}M_{s,a}G_a^{1/2}
\]
are positive semidefinite and satisfy $\tr(R_a)\to0$.  For a
positive-semidefinite matrix, $\lVert R_a\rVert_{\mathrm{op}}\le\tr(R_a)$.
Since $\lossstar_a\to\lossstar=d+1>0$, division by $\lossstar_a$ proves
\eqref{eq:ue-operator-supremum}.  When $d\ge2$ and $d<s<m$, combining this
convergence with~\eqref{eq:ue-strict-interior} distinguishes supremum from
attainment.
\end{proof}

\section{Computable certificates and shared moment formulas}
\label{app:computation}\label{app:pre-response-certificate}
This section gives lower endpoints for geometric margins and computes the
moment bounds proved above. These constructions are sufficient
certificates; generic zero endpoints do not characterize the phase.
\subsection{Feature-only lower certificates}
\label{sec:pre-response-certificate}

The positive branch can be certified without computing $\nu_A$: a sound
pre-response certificate is any feature-only scalar $t(A)$ with
\begin{equation}
  0\le t(A)\le\nu_A.
  \label{eq:pre-response-certificate-interface}
\end{equation}
On its real Parseval no-coloop domain, the global complement quantity
$t_\perp$ is a special case with the exact sign interpretation
$t_\perp=0$ if and only if $\nu_A=0$; other sufficient certificates below
need not have this converse.

\begin{corollary}[Sound computable pre-response certificates]
\label{cor:pre-response-certificate}
For every full-column-rank $X$ with $m>d$, the general spectral certificate
$c_X$ defined in Section~\ref{app:pre-response-certificate} satisfies
$0\le c_X\le\nu_A$.  Under the no-coloop assumptions of
Theorem~\ref{thm:robust-pre-response-phase}, the global complement
certificate $t_\perp(A)$ of
Proposition~\ref{prop:global-complement-certificate} satisfies
$0\le t_\perp(A)\le\nu_A$; under the assumptions of
Theorem~\ref{thm:critical-geometry}, the critical-coherence certificate
$t_\rho(A)$ satisfies $0\le t_\rho(A)\le\nu_A$.  More generally, every
verified $t(A)$ satisfying~\eqref{eq:pre-response-certificate-interface}
obeys, on the no-coloop positive-loss strict-interior domain of
Theorem~\ref{thm:robust-pre-response-phase},
\begin{equation}
  \boxed{
  \alpha_s-\mathsf T_{\mathrm{spec}}^{\mathrm{RU}}(A,s)
  \ge \beta_s\gamma_s t(A).}
  \label{eq:pre-response-certified-slack}
\end{equation}
\end{corollary}

\subsection{The general spectral certificate}

Let $N\in\R^{m\times(m-d)}$ have orthonormal columns spanning
$\ker(A^\top)$, let $D_\ell=\operatorname{diag}(\ell)$, and set
$P_\perp=I_m-AA^\top=NN^\top$.  Define
\begin{equation}
  R_0=\sum_{i=1}^m(1-\ell_i)a_i a_i^\top,
  \quad
  \tau_X=\lambda_{\max}\!\left(N^\top D_\ell N\right),
  \quad
  c_X=\bigl[\lambda_{\min}(R_0)-\tau_X\bigr]_+.
  \label{eq:pre-response-certificate}
\end{equation}
For a unit
$z\in\ker(A^\top)$, write
\begin{equation}
  T(z)=\sum_i z_i^2a_i a_i^\top,
  \qquad R_A(z)=R_0-T(z).
  \label{eq:pre-response-app-decomposition}
\end{equation}
Because $[A\ z]$ has orthonormal columns, its row leverages satisfy
$\ell_i+z_i^2\le1$, and hence
\begin{equation}
  R_A(z)=\sum_i(1-\ell_i-z_i^2)a_i a_i^\top\succeq0.
  \label{eq:pre-response-app-psd}
\end{equation}
For every $v\in\R^d$,
\begin{align}
  v^\top T(z)v
  &=\sum_i z_i^2(a_i^\top v)^2
    \le \lVert v\rVert_2^2\sum_i\ell_i z_i^2 \notag\\
  &=\lVert v\rVert_2^2z^\top P_\perp D_\ell P_\perp z
    \le\tau_X\lVert v\rVert_2^2.
  \label{eq:pre-response-app-compression}
\end{align}
Thus $T(z)\preceq\tau_XI_d$ and
\begin{equation}
  R_A(z)\succeq
  [\lambda_{\min}(R_0)-\tau_X]I_d.
  \label{eq:pre-response-app-floor-raw}
\end{equation}
Combining this with~\eqref{eq:pre-response-app-psd} gives
$R_A(z)\succeq c_XI_d$ for every admissible $z$.  Minimizing over $z$ proves
\begin{equation}
  0\le c_X\le\nu_A.
  \label{eq:pre-response-cx-lower}
\end{equation}

The computation can use the residual-space compression
\begin{equation}
  \tau_X=\lambda_{\max}(N^\top D_\ell N),
  \label{eq:pre-response-app-residual-computation}
\end{equation}
without forming a dense $m\times m$ projector.  The value is invariant under
an invertible feature-coordinate change $X\mapsto XH$: the corresponding
whitened matrices differ by a right orthogonal factor, so $R_0$ changes by
orthogonal congruence while $P_\perp,D_\ell,\tau_X$, and $c_X$ are unchanged.

\subsection{Capped, top-K, and coherence routes}

Further sufficient routes are as follows. Set $\ell_-:=\min_i\ell_i$, $\ell_+:=\max_i\ell_i$,
and $c_i=1-\ell_i$. Define
\begin{equation}
  \mathcal P_A=\left\{p\in\R^m:
  0\le p_i\le c_i,\ \sum_i p_i=1\right\},
  \qquad
  \underline\nu_{\rm cap}(A)=
  \min_{p\in\mathcal P_A}
  \lambda_{\min}\!\left(A^\top\operatorname{diag}(c-p)A\right).
  \label{eq:pre-response-app-capped}
\end{equation}
For every residual direction, $p=z^{\odot2}$ belongs to $\mathcal P_A$:
the augmented projection gives $0\le z_i^2\le1-\ell_i$ and unit norm gives
$\sum_i z_i^2=1$.  Enlarging the realizable residual-square set therefore
gives
\begin{equation}
  0\le\underline\nu_{\rm cap}(A)\le\nu_A.
  \label{eq:pre-response-app-capped-lower}
\end{equation}

Let $c_\star=1-\ell_-$, $K=\lceil1/c_\star\rceil$, and
\begin{equation}
  \Lambda_{A,K}=
  \max_{|J|=K}\lambda_{\max}\!\left(\sum_{i\in J}a_i a_i^\top\right).
  \label{eq:pre-response-app-topk}
\end{equation}
For a unit $v$, uniform-cap fractional knapsack bounds
$\sum_i p_i(a_i^\top v)^2$ by $c_\star$ times the sum of the $K$ largest
directional row energies.  Consequently
\begin{equation}
  t_{\rm top}(A):=
  [\lambda_{\min}(R_0)-c_\star\Lambda_{A,K}]_+
  \le\underline\nu_{\rm cap}(A).
  \label{eq:pre-response-app-topk-lower}
\end{equation}
If every row is nonzero, define the normalized-row coherence
\begin{equation}
  \rho_A^{\rm row}=\max_{i\ne j}
  \frac{|a_i^\top a_j|}{\sqrt{\ell_i\ell_j}}.
  \label{eq:pre-response-app-row-coherence}
\end{equation}
Gershgorin's theorem applied to each normalized-row $K$-Gram matrix gives
\begin{equation}
  \Lambda_{A,K}\le
  \ell_+\{1+(K-1)\rho_A^{\rm row}\}.
  \label{eq:pre-response-app-gershgorin}
\end{equation}
It follows that
\begin{equation}
  t_{\rm coh}(A):=
  [\lambda_{\min}(R_0)-
  c_\star\ell_+\{1+(K-1)\rho_A^{\rm row}\}]_+
  \le\underline\nu_{\rm cap}(A)\le\nu_A.
  \label{eq:pre-response-app-coherence-lower}
\end{equation}
At $m=2d$ with equal leverage, the sharper specialized statement
$t_\rho=(1-\rho_A)/4\le\nu_A$ is precisely the lower side of
Theorem~\ref{thm:critical-geometry}; its $\rho_A$ is the critical normalized
coherence in~\eqref{eq:critical-coherence}.
\begin{equation}
  t_\rho(A)=\frac{1-\rho_A}{4}.
  \label{eq:pre-response-coherence-certificate}
\end{equation}
By~\eqref{eq:critical-global-reduction}, this is exactly the global
complement certificate $t_\perp$ on the critical equal-leverage class.  Thus
its zero value is an exact phase test there; the generic sufficient routes
above remain one-sided at zero.

\subsection{Propagation to robust slack}

Let $t(A)$ be any proved lower certificate in
\eqref{eq:pre-response-certificate-interface}.  On the domain of
Corollary~\ref{cor:margin-slack}, the quantitative bound gives
\begin{equation}
  \alpha_s-\mathsf T_{\mathrm{spec}}^{\mathrm{RU}}(A,s)
  \ge\beta_s\gamma_s\nu_A
  \ge\beta_s\gamma_s t(A),
  \label{eq:pre-response-app-phase-propagation}
\end{equation}
which proves~\eqref{eq:pre-response-certified-slack}.  The first
inequality is the quantitative margin-to-slack bound.  Replacing $\nu_A$ by a
lower certificate is a sufficient implication and supplies no converse.

The resolvent scalarization, its hierarchy, and the general closed-form
sufficient size are proved in Section~\ref{app:scalar-certificate-hierarchy}.
For example, its formula with $m=512$, $d=64$, $t=13/60$, and tolerance
$1/3$ gives $s=338$, compared with $363$ for $t=0$. This arithmetic
is conditional on a verified lower certificate $t\le\nu_A$; it does not
assert that a particular design passes or that either size is optimal.

\subsection{Structural calibration and abstention boundaries}
\label{subsec:pre-response-structural-calibration}

The inexpensive $c_X$ route is intentionally conservative.  Since
\begin{equation}
  R_0=I_d-A^\top D_\ell A,\qquad
  \operatorname{tr}(P_\perp D_\ell P_\perp)
  =\sum_i\ell_i(1-\ell_i),
  \label{eq:pre-response-app-trace-identities}
\end{equation}
we have
\begin{equation}
  \lambda_{\min}(R_0)
  \le\frac{\sum_i\ell_i(1-\ell_i)}d,\qquad
  \tau_X\ge
  \frac{\sum_i\ell_i(1-\ell_i)}{m-d}.
  \label{eq:pre-response-app-trace-bounds}
\end{equation}
Hence $c_X=0$ whenever $d<m\le2d$.  For an equal-leverage Parseval design,
\begin{equation}
  c_X=\left[1-\frac{2d}{m}\right]_+,
  \label{eq:pre-response-app-equal-leverage}
\end{equation}
which is positive when $m>2d$.  That redundancy condition alone is not
sufficient: the full-rank Parseval core-plus-zero design
$A=[I_d;0_{(m-d)\times d}]$ has $c_X=0$ for arbitrarily large $m$.
At the critical boundary $m=2d$, $c_X$ is therefore silent even though the
coherence certificate in~\eqref{eq:pre-response-coherence-certificate} can be
positive.

These zeros demonstrate abstention, not impossibility.  In particular,
$c_X=0$, $t_{\rm top}=0$, $t_{\rm coh}=0$, or an ineligible guard does not show $\nu_A=0$ and does not invoke
the zero-margin branch of Theorem~\ref{thm:robust-pre-response-phase}.  The
positive-loss and strict-interior formulas are not evaluated at
$\lossstar=0$, $s=d$, $s=m$, or $m=d$; the endpoint statements in the
universal-envelope and residual-mechanism sections remain controlling.
\paragraph{A lower endpoint without a complement eigensolve.}
Write $L_{\max}=\max_i\ell_i$. Since $N^\top\operatorname{diag}(u_i)N\preceq(\max_i u_i)I$,
\begin{equation}
 \max\!\left\{0,1-\sum_i\ell_i u_i-\max_i u_i\right\}\le t_Q.
 \label{eq:query-cheap-endpoint}
\end{equation}
\begingroup
\makeatletter
\def\@currentlabel{\theequation}
\label{eq:query-cheap-endpoint-main}
\makeatother
\endgroup
The unit-trace PSD matrix $\Xi/W_Q$ has operator norm at most one, so
$u_i\le\ell_i$. Equation~\eqref{eq:query-cheap-endpoint} is consequently
at least $\max\{0,1-2L_{\max}\}$. It can be combined by a maximum with
other validated lower endpoints of $t_Q$, including any lower bound on
$\nu_A$. If $\Xi=I_d$ and $\ell_i=d/m$ for every row, then $u_i=1/m$
and the compression is exactly $I_n/m$, giving $t_Q=(m-d-1)/m$.

\subsection{Augmented singleton and pair moments}
The correction is computable from singleton and pair exclusions.  Put
$T=S^c$, $k=m-s$, $q=m-d-1$, $D_i=a_ia_i^\top$, and retain
$h_i=\lVert b_i\rVert_2^2$.  The singleton exclusion probability is
\begin{align}
  \pi_i&=\mathbb P_B(i\in T)=\frac{k}{q}(1-h_i),
  \label{eq:second-moment-singleton-exclusion}
\end{align}
and, for $q\ge2$, the pair exclusion probability is
\begin{align}
  \pi_{ij}&=\mathbb P_B(i,j\in T)
  =\frac{k(k-1)}{q(q-1)}
  \left[(1-h_i)(1-h_j)-(b_i^\top b_j)^2\right].
  \label{eq:second-moment-pair-exclusion}
\end{align}
When $m=d+2$, the strict domain forces $q=k=1$; in that branch set
$\pi_{ij}=0$ rather than evaluating
\eqref{eq:second-moment-pair-exclusion}.  In either branch,
\begin{equation}
  V_s=\sum_i\pi_i(1-\pi_i)D_i^2
  +\sum_{i<j}(\pi_{ij}-\pi_i\pi_j)
    (D_iD_j+D_jD_i).
  \label{eq:second-moment-pair-assembly}
\end{equation}

For the marginal formulas, fixed-cardinality Cauchy--Binet applied after
deleting one or two rows gives the support-aware specialization of the
generic dual-volume inclusion-marginal mechanism
\citep{li2017dual}.  In particular,
\begin{align}
 \mathbb P_B(i\in T)
 &=\frac{\binom{q-1}{s-d-1}}{\binom q{s-d-1}}
   \det(I_{d+1}-b_ib_i^\top),\notag\\
 \mathbb P_B(i,j\in T)
 &=\frac{\binom{q-2}{s-d-1}}{\binom q{s-d-1}}
   \det(I_{d+1}-b_ib_i^\top-b_jb_j^\top),
 \label{eq:second-moment-marginal-cauchy-binet}
\end{align}
where the second line is used only for $q\ge2$.  The rank-one and rank-two
determinant lemmas give
\eqref{eq:second-moment-singleton-exclusion} and
\eqref{eq:second-moment-pair-exclusion}.  If $q=k=1$, the excluded set has
size one, so its pair-exclusion probability is zero directly.  Finally,
$K_S=I_d-\sum_{i\in T}D_i$; expanding the square of its centered version
retains both product orders and gives the anticommutator in
\eqref{eq:second-moment-pair-assembly}.

\subsection{Shared computation of the second moment}
\label{subsec:aggregate-shared-moments}

Write $b_{ij}^\top$ for row $i$ of $B_j$, and set
$h_{ij}=\|b_{ij}\|_2^2$, $r_{ij}=1-h_{ij}$, and $D_i=a_ia_i^\top$.
The following matrices do not depend on the subset size:
\begin{align}
 R_j&=\sum_i r_{ij}D_i,&
 U&=\sum_iD_i^2,&
 W_j&=\sum_i h_{ij}D_i^2,\notag\\
 T_{pt}^{(j)}&=\sum_i b_{ij,p}b_{ij,t}D_i,&
 S_j&=\sum_{p,t=1}^{d+1}(T_{pt}^{(j)})^2.
 \label{eq:aggregate-shared-moment-statistics}
\end{align}
Every square here is an ordered matrix product, not an entrywise square.
Put $n_0=m-d-1$, $k=m-s$, $c=k/n_0=\gamma_s$, and
$c_2=k(k-1)/[n_0(n_0-1)]$ when $n_0\ge2$.
When $n_0=k=1$, set $c_2=0$ before arithmetic.
The symbol $n_0$ is unrelated to the query-row count $q$.
Then
\begin{equation}
 V_{s,j}=(c_2-c^2)R_j^2-c_2S_j
                  +(c-c_2)U+(2c_2-c)W_j.
 \label{eq:aggregate-shared-moment-contraction}
\end{equation}
To verify this identity, let $J_T=\sum_{i\in S^c}D_i$, so that
$K_S=I-J_T$ and $\mathbb E_{B_j}J_T=cR_j$.
The singleton and pair exclusion probabilities
\eqref{eq:second-moment-singleton-exclusion}--\eqref{eq:second-moment-pair-exclusion} give
\begin{equation*}
 \mathbb E_{B_j}J_T^2=c(U-W_j)
 +c_2\sum_{i\ne l}\{r_{ij}r_{lj}-(b_{ij}^\top b_{lj})^2\}D_iD_l.
\end{equation*}
Expanding the block squares yields
$S_j=\sum_{i,l}(b_{ij}^\top b_{lj})^2D_iD_l$.
The diagonal contribution to $R_j^2-S_j$ is
$\sum_i(r_{ij}^2-h_{ij}^2)D_i^2=U-2W_j$.
Subtracting $(cR_j)^2$ proves
Equation~\eqref{eq:aggregate-shared-moment-contraction}, with both product
orders retained. In the $n_0=k=1$ branch, the excluded set has one row and
$V_{s,j}=U-W_j-R_j^2$ directly. More generally, when $k=1$, $c_2=0$ and
forming $S_j$ is unnecessary.

The $A$ columns are common across outputs. Define
\begin{equation}
 \begin{gathered}
 T_{pt}^{AA}=\sum_i a_{i,p}a_{i,t}D_i,\qquad
 S_{AA}=\sum_{p,t=1}^d(T_{pt}^{AA})^2,\\
 M_{p,j}=\sum_i a_{i,p}z_{ij}D_i,\qquad
 N_j=\sum_i z_{ij}^2D_i.
 \end{gathered}
 \label{eq:aggregate-shared-mixed-blocks}
\end{equation}
Thus
\begin{equation}
 S_j=S_{AA}+2\sum_{p=1}^dM_{p,j}^2+N_j^2.
 \label{eq:aggregate-shared-block-square}
\end{equation}
The factor two counts the $(p,z)$ and $(z,p)$ blocks. A reproducible
quadratic-feature convention is the \emph{unscaled} upper triangle
$v_i=(a_{i,p}a_{i,t})_{p\le t}$, of length $d(d+1)/2$.
For $F=\sum_i v_iv_i^\top$, symmetrically unpacking column $(p,t)$ gives
$T_{pt}^{AA}$. Consequently
$S_{AA}=\sum_p(T_{pp}^{AA})^2+2\sum_{p<t}(T_{pt}^{AA})^2$.
This convention has no $\sqrt2$ scaling. The mixed feature products
$\sum_i z_{ij}a_i v_i^\top$ similarly unpack to the $M_{p,j}$ blocks.
Since $D_i^2=\ell_iD_i$, the remaining moments satisfy
\begin{equation*}
 U=\sum_i\ell_iD_i,\qquad R_j=I-U-N_j,\qquad
 W_j=\sum_i\ell_i^2D_i+\sum_i z_{ij}^2\ell_iD_i.
\end{equation*}
The first term of $W_j$ is shared as well.

One source-row stream accumulates the quadratic Gram for $S_{AA}$ and
shared $U$; per positive-loss output, it accumulates $M_{p,j}$, $N_j$,
and the remaining $W_j$. At $k=1$, $S_j$ is skipped. For each interior
size, \eqref{eq:aggregate-shared-moment-contraction} yields $V_{s,j}$;
then the ordered product $H_{s,j}^{-1}V_{s,j}H_{s,j}^{-1}$ is contracted
with the fixed query metric. No source-pair loop or $m\times m$ source
Gram is required.

For $K_+=|J_+|$ and $C$ interior cards, this direct dense construction has
arithmetic cost
\begin{equation}
 O(md^4+K_+md^3+d^5+K_+d^4+CK_+d^3).
 \label{eq:aggregate-shared-moment-cost}
\end{equation}
These terms respectively cover shared quadratic-feature products, mixed
products, shared block squares, output-specific block squares, and per-card
matrix inverses and contractions, after the common Full fit and query
geometry have been prepared. Row blocks allow accumulation without an
$m\times m$ row Gram or an $m\times d(d+1)/2$ stored feature matrix.
Apart from inputs and streamed blocks, the dense state is
$O(d^4+K_+d^3+Cd^2)$. The dependence on source-row count is linear for fixed
$d,K_+$; avoiding $m^2$ storage does not remove the explicit $d^4$ and $d^5$
limitations in \eqref{eq:aggregate-shared-moment-cost}.
\paragraph{Raw moments without a row-pair loop.}
For $q_i=ZG^{-1}x_i$, the scalar moments are
\begin{equation}
 U_0=\sum_i\frac{e_i^2}{c_i}\|q_i\|_2^2,\qquad
 U_1=\sum_{i,j}\frac{e_ie_jP_{ij}^2}{c_ic_j}q_i^\top q_j.
 \label{eq:loo-raw-query-moments}
\end{equation}
Put $T_a=\sum_i(e_i/c_i)(q_i)_a x_ix_i^\top$ for query row $a$.
Expanding $P_{ij}=\mathbf1_{i=j}-x_i^\top G^{-1}x_j$ gives the exact identity
\begin{equation}
 U_1=\sum_i\frac{(2c_i-1)e_i^2}{c_i^2}\|q_i\|_2^2
 +\sum_{a=1}^{q_g}\tr(T_aG^{-1}T_aG^{-1}).
 \label{eq:loo-raw-moment-tensor}
\end{equation}
The signed diagonal term matters. This formula supports streaming source
moments and exact rational checks; a floating-point evaluation alone is not
an outward-rounded certificate.

\section{Sufficient sizes and rigorous numerical certification}
\label{app:certification}
The inequalities below concern the exact fixed-size law. Numerical endpoints
must enclose the stated mathematical quantities before a pass or exclusion
is issued. They are not guarantees for an individual sampled subset.
\subsection{Scalar envelopes and sufficient sizes}
\label{app:scalar-certificate-hierarchy}

This section proves the scalar upper-envelope family used in
Section~\ref{sec:response-uniform-certificates}.  It starts with a proved lower
certificate $0\le t\le\nu_A$.

\paragraph{Feature-only margins and the scalar hierarchy.}
Recall the margin $\nu_A$ and the exact-sign complement certificate
$t_\perp\le\nu_A$ from Proposition~\ref{prop:global-complement-certificate}.
We now compare sufficient scalar envelopes on their common domain.

\paragraph{Maximum-leverage lower bound.}
Let $L_{\max}=\max_i\ell_i$. Then
\begin{equation}
 \nu_A\ge t_{\rm lev}(A):=\max\{0,1-2L_{\max}\}.
 \label{eq:max-leverage-certificate}
\end{equation}
For every $z\in\mathcal Z_A$, Parsevalness gives
$\sum_i\ell_i a_i a_i^\top\preceq L_{\max}I_d$.
Also $a_i a_i^\top\preceq L_{\max}I_d$ and $\sum_i z_i^2=1$, so
$\sum_i z_i^2a_i a_i^\top\preceq L_{\max}I_d$.
Thus $R_A(z)\succeq(1-2L_{\max})I_d$; the orthonormal augmented frame
$[A\ z]$ also gives $R_A(z)\succeq0$. Minimizing over $z$ proves
\eqref{eq:max-leverage-certificate}.

For any validated $0\leq t\leq\nu_A$, define
\begin{equation}
 \Buniv(s)=\alpha_s,\quad
 \Blin(s;t)=\alpha_s-\beta_s\gamma_st,\quad
 \Bres(s;t)=1-\frac{\beta_s}{1-\gamma_st}.
 \label{eq:certificate-envelope-family}
\end{equation}

\begin{lemma}[Certificate-envelope monotonicity]
\label{lem:certificate-envelope-monotonicity}
Under the assumptions of
Theorem~\ref{thm:robust-pre-response-phase}, put $n=m-d$.  If
$0\le t\le\nu_A$, then
\[
  t\le\nu_A<\frac{n-1}{n}.
\]
Consequently, for $r=s-d\in\{1,\ldots,n-1\}$, all denominators in
$\Bres(s;t)$ are positive and $\Buniv(s)$, $\Blin(s;t)$, and
$\Bres(s;t)$ are strictly decreasing in $r$.
\end{lemma}

\begin{proof}
Since $d\ge1$, the identity $I_d$ is a nonzero PSD query. The
margin-range bound in Corollary~\ref{cor:query-phase-certificate} gives
\[
 0\le t\le\nu_A\le t_{I_d}<\frac{n-1}{n}.
\]

Write
\[
 D_r=1-\frac{n-r}{n-1}t>\frac rn>0.
\]
For $1\le r\le n-2$, direct subtraction gives
\[
\begin{aligned}
 \Buniv(d+r)-\Buniv(d+r+1)
   &=\frac1n,\\
 \Blin(d+r;t)-\Blin(d+r+1;t)
   &=\frac{(n-1)+t(n-2r-1)}{n(n-1)},\\
 \Bres(d+r;t)-\Bres(d+r+1;t)
   &=\frac{1-nt/(n-1)}{nD_rD_{r+1}}.
\end{aligned}
\]
The first and third quantities are positive.  The middle numerator is at
least $(n-1)-t(n-3)\ge3(n-1)/n>0$ when $n\ge3$; for $n=2$ there is no
adjacent pair.  This proves the claim.
\end{proof}

For a fixed positive-loss response, let $z=e/\sqrt{\lossstar}$ and
$q=m-s$. For $1\le r=s-d\le n-1$, the boundary-valid resolvent inequality gives
\begin{equation}
  \frac{\overline M_s}{\lossstar}
  \preceq I_d-\frac{n-q}{n}
  \left(I_d-\frac{q}{n-1}R_A(z)\right)^{-1}.
  \label{eq:pre-response-app-resolvent-scalar}
\end{equation}
Since $R_A(z)\succeq tI_d$ and
$I_d-\gamma_sR_A(z)\succ0$ by
\eqref{eq:residual-augmented-first-moment}, inversion reverses Loewner order.
The lemma gives $1-qt/(n-1)>0$, so
\begin{align}
  \frac{\overline M_s}{\lossstar}
  &\preceq\left(1-\frac{(n-q)/n}{1-qt/(n-1)}\right)I_d \notag\\
  &=\frac{q(n-1-nt)}{n(n-1-qt)}I_d
  =\Bres(s;t)I_d.
  \label{eq:pre-response-app-scalarization-scalar}
\end{align}
This uses the one-sided resolvent at arbitrary full-rank support boundaries,
not an exact inverse-moment covariance identity there.

Set $x=\gamma_st\in[0,1)$. Direct subtraction gives
\begin{equation}
  \Blin(s;t)-\Bres(s;t)=\frac{\beta_sx^2}{1-x}\ge0,
  \qquad
  \Buniv(s)-\Blin(s;t)=\beta_sx\ge0.
  \label{eq:pre-response-app-envelope-order}
\end{equation}
Thus the scalarization gives the hierarchy:
\begin{equation}
 \frac{\overline M_s}{\lossstar}
 \preceq\Bres(s;t)I_d\preceq\Blin(s;t)I_d
 \preceq\Buniv(s)I_d.
 \label{eq:certificate-envelope-order}
\end{equation}
At Full, the covariance and all three envelopes are zero. At rank size
$s=d$, use the separate rank-size branch; the strict-interior resolvent
argument does not invoke rank-$(d+1)$ augmentation there.

\begin{corollary}[First upper-envelope passes]
\label{cor:scalar-first-passes}
Fix one admissible integer domain and tolerance $\tau$. If
$s_{\rm Res},s_{\rm Lin},s_{\rm Univ}$ are the first integers at which the
corresponding envelopes in~\eqref{eq:certificate-envelope-family} are at most
$\tau$, then
\begin{equation}
 s_{\rm Res}\le s_{\rm Lin}\le s_{\rm Univ}.
 \label{eq:certificate-first-pass-order}
\end{equation}
\end{corollary}

On a common ordered integer domain, the pointwise order in
Equation~\eqref{eq:certificate-envelope-order} nests the three pass sets.
Their first integers therefore obey
Equation~\eqref{eq:certificate-first-pass-order}; only these upper-envelope
predicates are ordered.

\subsection{Closed-form resolvent pass and leverage-scale envelope}
\label{app:resolvent-first-pass}

Here we derive the exact resolvent first pass and the
maximum-leverage consequence
\begin{equation}
 s_{\rm Res}(\tau;t_{\rm lev})
 \le\min\!\left\{m,\ d+\left\lceil
 \frac{2(1-\tau)\mu d}{\tau}\right\rceil\right\},
 \qquad \mu=mL_{\max}/d,\quad L_{\max}<1/2.
 \label{eq:leverage-coherence-envelope-consequence}
\end{equation}
Take a full-column-rank real design under
Theorem~\ref{thm:robust-pre-response-phase}, a certificate
$0\le t\le\nu_A$, and $0<\tau<1$. Search $s\in\{d+1,\ldots,m\}$,
assigning Full the exact bound zero. The inherited theorem domain has
$n=m-d\ge2$, and the monotonicity lemma gives $t<(n-1)/n$. Put
$\delta=1-nt/(n-1)\in(0,1]$ and $r=s-d$. Simplifying
Equation~\eqref{eq:pre-response-app-scalarization-scalar} gives
\begin{align}
 \Bres(d+r;t)
 &=\frac{\delta(n-r)}{n\delta+(1-\delta)r},\notag\\
 \Bres(d+r;t)\le\tau
 &\quad\Longleftrightarrow\quad
 r\ge\frac{n(1-\tau)\delta}{\tau+(1-\tau)\delta}.
 \label{eq:app-resolvent-crossing}
\end{align}
The rational denominator is positive for $1\le r\le n$; at $r=n$ the
expression is zero, agreeing with the exact Full bound. Since
$0<(1-\tau)\delta<\tau+(1-\tau)\delta$, the crossing lies in $(0,n)$.
Its ceiling is the first legal $r\in\{1,\ldots,n\}$; if that ceiling is
$n$, the pass is Full with its exact zero bound:
\begin{equation}
 s_{\rm Res}(\tau;t)=d+\left\lceil
 \frac{n(1-\tau)\delta}{\tau+(1-\tau)\delta}
 \right\rceil.
 \label{eq:resolvent-first-pass-closed-form}
\end{equation}
For a reported certificate, the stored floating-point endpoint is interpreted
as its exact rational value before evaluating the ceiling predicate, so
evaluation rounding cannot alter the cardinality for that stored certificate.

For the maximum-leverage certificate assume $L_{\max}<1/2$ and put
$t=t_{\rm lev}=1-2L_{\max}$ and $\mu=mL_{\max}/d$. Thus $\mu$ measures
maximum leverage relative to the average $d/m$, and $\mu d=mL_{\max}$
remains explicit in the sufficient-size scale. Since the average row leverage
is $d/m$, $\mu\ge1$. The definition of $\delta$ and $L_{\max}<1/2$ give
\begin{align}
 n\delta
 &=\frac{2n^2L_{\max}-n}{n-1}
 <2nL_{\max}<2mL_{\max}=2\mu d,
 \label{eq:app-leverage-delta-bound}\\
 \frac{n(1-\tau)\delta}{\tau+(1-\tau)\delta}
 &\le\frac{n(1-\tau)\delta}{\tau}
 <\frac{2(1-\tau)\mu d}{\tau}.\notag
\end{align}
The second line bounds the crossing in
Equation~\eqref{eq:app-resolvent-crossing}. Taking ceilings and assigning
the Full endpoint if needed proves
Equation~\eqref{eq:leverage-coherence-envelope-consequence}, which gives a
sufficient cardinality for the scalar-envelope predicate. It is not an
$m$-independent $O(d)$ statement unless $\mu$ is separately controlled, and
it makes no optimal sample-complexity claim.
\subsection{Known-label floors and the expected guarantee}
\label{subsec:aggregate-known-label-floors}

Put $F=ZW_F$, let $\mathcal Y$ be a specified class of allowed fixed query
label matrices, and retain the original row/output normalization:
\begin{equation}
 L_{\min}(W_F;\mathcal Y)=
 \min_{Y_Q\in\mathcal Y}\frac{\|F-Y_Q\|_F^2}{qK}.
 \label{eq:known-label-loss-floor}
\end{equation}
For rowwise allowed alphabets $\mathcal A_i$,
\begin{equation}
 L_{\min}^{\rm row}=\frac1{qK}\sum_{i=1}^q
       \min_{y\in\mathcal A_i}\|F_{i,:}-y\|_2^2.
 \label{eq:aggregate-row-floor}
\end{equation}
This covers one-hot and binary query labels. If a
fixed partition $\{G_g\}$ is known to share one one-hot class within each
group, the sharper floor is
\begin{equation}
 L_{\min}^{\rm group}=\frac1{qK}\sum_g\min_{c\in[K]}
       \sum_{i\in G_g}\|F_{i,:}-e_c\|_2^2
 \ge L_{\min}^{\rm row}.
 \label{eq:aggregate-group-floor}
\end{equation}
Every row retains its original weight. The group condition is an assumption
on the allowed label set. The allowed label set must be specified before applying the guarantee.

For every allowed labeling, $L_{Q,F}\ge L_{\min}$. The fixed-label identity
\eqref{eq:aggregate-fixed-label-identity} and $\mathcal D_s\le U_2$ give,
for $\eta\ge0$,
\begin{equation}
 U_2(s)\le\eta L_{\min}
 \quad\Longrightarrow\quad
 \E_XL_{Q,S}\le(1+\eta)L_{Q,F}
 \quad\text{for every fixed }Y_Q\in\mathcal Y.
 \label{eq:aggregate-task-tolerance}
\end{equation}
The same conclusion follows if $U_{\rm per}$ or $U_{\rm agg}$ is at most
$\eta L_{\min}$.
If the floor is zero, the multiplied inequality is
retained; a finite ratio is not manufactured from $0/0$.

One may therefore fix a finite size family $\mathcal S$ containing Full and,
before sampling, choose
\begin{equation}
 s_U^*(\eta)=\min\{s\in\mathcal S:U(s)\le
                       \eta L_{\min}\}.
 \label{eq:aggregate-card-selection}
\end{equation}
The guarantee then uses the ordinary size-$s_U^*(\eta)$ law. This label-floor criterion requires the stated restrictions on the allowed
query labels; an arbitrary scale is not a substitute for a loss floor.
\subsection{Lower-bound enclosures}
\paragraph{Safe numerical enclosure.}
Suppose exact rational endpoints enclose $U_0\in[l_0,u_0]$ and
$U_1\in[l_1,u_1]$. They must be obtained with rigorous arithmetic.
The function $x^2/(x-\lambda y)$ need not increase in $x$: its derivative
is $x(x-2\lambda y)/(x-\lambda y)^2$. Consequently, inserting $l_0,l_1$
without further work is unsafe. Intersect the box with $0\le y\le x$.
Set
\begin{equation}
 \begin{aligned}
 y_{\min}&=\max(0,l_1),& x_{\min}&=\max(0,l_0,y_{\min}),\\
 x_{\max}&=u_0,&
 x_*&=\min(x_{\max},\max(x_{\min},2\lambda y_{\min})).
 \end{aligned}
 \label{eq:loo-safe-box-argmin}
\end{equation}
Reject an inconsistent box if $y_{\min}>u_1$ or
$x_{\min}>x_{\max}$. Otherwise the exact minimum of the scalar expression
over the structural box is zero when $x_*=0$, and
$x_*^2/(x_*-\lambda y_{\min})$ when $x_*>0$. For fixed $x>0$ the
expression increases in $y$, and for $y=y_{\min}$ its derivative changes
sign at $x=2\lambda y_{\min}$. The structural inequality $y\le x$ makes
the zero endpoint continuous. Every final multiplication by $\alpha/(q_gK)$
and every comparison with the allowance is exact rational arithmetic.

An alternative is to fix an exact rational $a$ before evaluation. With
$b=2a-a^2$ and $x_a=l_0$ if $b\ge0$ (otherwise $x_a=u_0$),
\eqref{eq:loo-scalar-remainder} gives the safe lower
$\alpha\max\{0,bx_a+a^2\lambda l_1\}$. Both routes give lower endpoints
when the moment enclosures are valid. For finite binary input arrays, one
may check the Gram matrix and Full fit with exact dyadic arithmetic and
enclose the moments with outward rounding. Such a certificate concerns
the specified arrays; it does not automatically include upstream measurement
or preprocessing error.
\subsection{Task-conditioned comparisons}

\label{app:task-conditioned-risk}

This section proves the direct fixed-query consequences used in
Section~\ref{sec:task-conditioned-risk}. Throughout, all feature matrices,
responses, query features, and query targets are fixed before the subset
draw. The sampler is exactly the ordinary indexed fixed-size volume law and
the estimator is selected unweighted OLS on positive-volume subsets.

For pipeline $r$ at a valid budget $s_r$, let $S_r$ denote its subset draw and
write its same-pool selected-OLS risk as
\[
  \mathcal R_r^{\rm pool}(y;s_r)
  :=\E_r\lVert X_rw_{r,S_r}(y)-y\rVert_2^2.
\]
The following projection fact is explicitly about training-pool squared loss;
it does not itself compare external-query risks.

\begin{proposition}[Task-free strict orders fail on incomparable spaces]
\label{prop:incomparable-task-reversal}
Let $X_a\in\R^{m\times d_a}$ and $X_b\in\R^{m\times d_b}$ have full column
rank, with $\mathcal U_a=\operatorname{col}(X_a)$ and
$\mathcal U_b=\operatorname{col}(X_b)$. If neither space contains the other,
then, for every pair of valid ordinary-volume budgets, no strict
feature-only order of the two selected-OLS pipelines is sound for every
scalar response. In fact, there are responses $y^{(a)}$ and $y^{(b)}$ for
which the respective same-pool subset-refit risks satisfy
\[
  \mathcal R_a^{\rm pool}(y^{(a)};s_a)=0
  <\mathcal R_b^{\rm pool}(y^{(a)};s_b),
  \qquad
  \mathcal R_b^{\rm pool}(y^{(b)};s_b)=0
  <\mathcal R_a^{\rm pool}(y^{(b)};s_a).
\]
For full-head squared approximation loss, the universal weak order is exactly
reverse column-space inclusion:
\[
  \lVert(I-P_{\mathcal U_a})Y\rVert_F^2
  \le \lVert(I-P_{\mathcal U_b})Y\rVert_F^2\quad\text{for every }Y
  \quad\Longleftrightarrow\quad
  \mathcal U_b\subseteq\mathcal U_a.
\]
The inclusion statement does not order subset-trained pipeline risk in
general, because the subset-refit variance and budget can differ.
\end{proposition}

\begin{proof}
Choose $y^{(a)}\in\mathcal U_a\setminus\mathcal U_b$. Every subset in the
support of pipeline $a$ has full column rank, and selected OLS therefore
recovers the coefficient that realizes $y^{(a)}$ pointwise. Its risk is zero.
Pipeline $b$ has strictly positive full-head approximation loss; the
same-pool Pythagorean identity makes every selected-fit loss strictly positive.
Interchange $a$ and $b$ for the second response. For the full-head statement,
if $\mathcal U_b\subseteq\mathcal U_a$, orthogonal projection onto
$\mathcal U_a$ cannot have larger residual. Conversely, apply the displayed
inequality to every column vector in $\mathcal U_b$ to obtain containment.
\end{proof}

\paragraph{Fixed-query notation.}
For pipeline $r$, define
\[
\begin{aligned}
F_r^Q&=\lVert Z_rw_r^*-y_Q\rVert_2^2,\\
V_r^Q(s)&=\tr(Z_rM_{r,s}Z_r^\top),\\
R_r^Q(s)&=F_r^Q+V_r^Q(s),\\
C_r^Q(s)&=\alpha_r(s)L_r^*\tr(Z_rG_r^{-1}Z_r^\top),\\
U_r^Q(s)&=F_r^Q+C_r^Q(s).
\end{aligned}
\]
The universal envelope gives
$R_r^Q(s)=F_r^Q+V_r^Q(s)$, $0\le V_r^Q(s)\le C_r^Q(s)$, and
$U_r^Q(s)=F_r^Q+C_r^Q(s)$.

\begin{proposition}[One-sided fixed-query pipeline comparison]
\label{prop:query-certified-dominance}
For two pipelines $a,b$ on the same fixed raw training/query task and loss
normalization, let the preceding quantities be evaluated at valid budgets
$s,t$. Then,
\[
  U_a^Q(s)<F_b^Q
  \quad\Longrightarrow\quad
  R_a^Q(s)<R_b^Q(t).
\]
\end{proposition}

\begin{proof}
The two exact decompositions and the variance sandwich give
\[
 R_a^Q(s)=F_a^Q+V_a^Q(s)
 \le U_a^Q(s)<F_b^Q
 \le F_b^Q+V_b^Q(t)=R_b^Q(t).
\]
Each representation induces its own ordinary-volume law, so this compares
complete representation, sampler, and OLS pipelines.
\end{proof}
\subsection{Validated source-bound endpoints}
\label{app:source-bound-certificates}

\paragraph{Numerical coordinates.}
Symmetric whitening and numerical QR are congruent coordinates for the same
displacement. If $X=QR$ is a positive-diagonal reduced QR factorization,
$A=XG^{-1/2}$, and $O=G^{1/2}R^{-1}$, then $O$ is orthogonal and $Q=AO$.
With $\Delta_A=G^{1/2}(W_s-W_F)$ and
$\Delta_{\rm QR}=R(W_s-W_F)$,
\[
 \Delta_{\rm QR}=O^\top\Delta_A,
 \qquad
 \tr(\Xi\,\E\Delta_A\Delta_A^\top)
 =\tr\!\left[
 R^{-\top}Z^\top ZR^{-1}\,
 \E\bigl(\Delta_{\rm QR}\Delta_{\rm QR}^\top\bigr)\right].
\]
This bridge does not equate an empirical QR moment with an outward-certified
endpoint.

\paragraph{Feature-only margins.}
For a fixed source transform, let $\underline c_{\min}>0$ and
$\overline\rho_\perp<1$ enclose the complement norm and coherence in
Proposition~\ref{prop:global-complement-certificate}. Its analytic bound gives
\begin{equation}
 \underline t_{\rm pair}
 =\operatorname{down}\!\left[(1-\overline\rho_\perp)
               \operatorname{down}(\underline c_{\min}^{2})\right]
 \le t_\perp\le\nu_A,
 \label{eq:source-bound-cmin-certificate}
\end{equation}
where each subtraction and product is rounded downward. Similarly,
\begin{equation}
 \underline t_{\rm lev}
 =\max\{0,\operatorname{down}(2\underline c_{\min}-1)\}\le\nu_A.
 \label{eq:source-bound-leverage-endpoint}
\end{equation}
One may take the maximum of valid lower endpoints. After the common Gram
factorization, leverage evaluation costs $O(md^2)$; the pair route costs
$O(m^2d)$ and can stream its dot products without storing an $m\times m$ matrix.

\paragraph{Directed moment bounds.}
Floating-point approximations are accompanied by outward error enclosures.
The common Gram enclosure must certify positive definiteness and the inverse
perturbation domain; the exact Full solution, residuals, query Gram, and
moment contractions are enclosed using the same transform. Normalization
requires a strictly positive residual-loss lower bound, unless that output
is verified to have exactly zero loss. Positive inverse denominators and
finite error radii are required throughout. These conditions validate
\[
 \mathcal D_s\le U_2(s)\le
 U_2^+(s):=U_{\rm per}^+(s)-C_s^- ,\qquad 0\le C_s^-\le C_s.
\]
A negative lower correction can be replaced by zero because $C_s\ge0$.
An unresolved correction does not turn a first-order endpoint into a
second-moment certificate. A negative issued upper bound indicates failure.
For relative-loss decisions, the label-floor lower endpoint includes the
Full-prediction error allowance; a nonpositive lower floor is set to zero.
Full and verified zero-loss cases use their exact values before any
strict-interior normalization. Failed positivity or enclosure checks return
no certificate; nominal point estimates are not substituted.

\paragraph{First passing size.}
Candidates, tolerance, and information access are fixed before the subsequent
subset draw. For a scalar envelope searched over every admissible integer,
exact rational interpretation of a stored endpoint verifies
\begin{equation}
 B(s;\underline t)\le\tau<B(s-1;\underline t).
 \label{eq:source-bound-adjacent-crossing}
\end{equation}
The predecessor is omitted at $s=d+1$; monotonicity then proves minimality.
For a restricted grid or a moment policy, evaluate candidates in increasing
order and return the first directed upper endpoint below the allowance.
This certifies a sufficient grid size without assuming monotonicity of $U_2$
or claiming an actual-risk optimum. The moment contraction and its cost
are derived in Section~\ref{subsec:aggregate-shared-moments}.

\end{document}